\documentclass{article}

\usepackage{arxiv}

\usepackage[utf8]{inputenc} 
\usepackage[T1]{fontenc}    
\usepackage{hyperref}       
\usepackage{url}            
\usepackage{booktabs}       
\usepackage{amsfonts}       
\usepackage{nicefrac}       
\usepackage{microtype}      
\usepackage{lipsum}
\usepackage{graphicx}
\graphicspath{ {./images/} }

\author{
 Nil Ayday\thanks{Equal contribution} \\
  School of Computation, Information and Technology\\
Technical University of Munich\\
\texttt{nil.ayday@tum.de}
   \And
 Lingchu Yang$^*$  \\
 	Graduate School of Information Science and Technology\\
    The University of Tokyo\\
  \texttt{yang-lingchu@g.ecc.u-tokyo.ac.jp}
  \And
 Debarghya Ghoshdastidar \\
  School of Computation, Information and Technology\\
Technical University of Munich
}

\usepackage{comment}
\usepackage{tcolorbox}
\definecolor{hellblue}{RGB}{173, 216, 230}
\tcbset{
    colback=hellblue,
    colframe=hellblue,
    arc=2mm,
    boxrule=0pt,
}

\newenvironment{mybox}[2]{%
    \begin{tcolorbox}[colback=#1,colframe=#1]
        {\large \textbf{#2}}\\[4pt]
}{%
    \end{tcolorbox}
}

\newenvironment{mybox2}[2]{%
    \begin{tcolorbox}[colback=#1,colframe=#1]
}{%
    \end{tcolorbox}
}
\usepackage{mdframed}

\usepackage{pbox}

\usepackage{amsthm}
\usepackage{multirow}
\usepackage{booktabs}
\usepackage{amsmath}
\usepackage{graphicx}
\usepackage{caption}
\usepackage{float}
\usepackage{amssymb}   
\usepackage[american]{babel}
\usepackage{subcaption}
\usepackage{natbib} 
\usepackage{mathtools} 
\usepackage{booktabs} 
\usepackage{tikz} 
\usepackage{makecell}
\usepackage{adjustbox} 
\usepackage{placeins}

\newtheorem{theorem}{Theorem}

\newtheorem{lemma}[theorem]{Lemma}
\newtheorem{corollary}[theorem]{Corollary}

\newtheorem{remark}[theorem]{Remark}
\usepackage{wrapfig}
\title{Gaussian Process Limit Reveals Structural Benefits of  Graph Transformers}
\begin{document}
\maketitle
\begin{abstract}
Graph transformers are the state-of-the-art for learning from graph-structured data and are empirically known to avoid several pitfalls of message-passing architectures. However, there is limited theoretical analysis on why these models perform well in practice.
In this work, we prove that attention-based architectures have structural benefits over graph convolutional networks in the context of node-level prediction tasks.
Specifically, we study the neural network gaussian process limits of graph transformers (GAT, Graphormer, Specformer) with infinite width and infinite heads, and derive the node-level and edge-level kernels across the layers. 
Our results characterise how the node features and the graph structure propagate through the graph attention layers. 
As a specific example, we prove that graph transformers structurally preserve community information and maintain discriminative node representations even in deep layers, thereby preventing oversmoothing. 
We provide empirical evidence on synthetic and real-world graphs that validate our theoretical insights, such as integrating informative priors and positional encoding can improve performance of deep graph transformers.
%
\end{abstract}

\section{Introduction}
\label{sec:intro}
Graph Neural Networks (GNNs) are currently the de facto models for learning from graph-structured data. They solve complex prediction tasks on graphs by iteratively aggregating feature information across edges. Recent empirical and theoretical work has established that classical message-passing GNN architectures have several limitations, including difficulties in simultaneously handling homophilic and heterophilic graphs \citep{DBLP:journals/corr/abs-1905-09550,DBLP:journals/corr/AydaySG26}, oversmoothing by deep GNNs \citep{DBLP:conf/aaai/LiHW18,DBLP:conf/nips/Keriven22,DBLP:conf/iclr/WuCWJ23}, oversquashing \citep{DBLP:conf/iclr/0002Y21}, etc. These challenges have motivated the development of attention-based GNNs, such as Graph Attention Network (GAT) \citep{DBLP:conf/iclr/VelickovicCCRLB18}, Graphormer \citep{DBLP:conf/nips/YangLXLLASSX21}, and Specformer \citep{DBLP:conf/iclr/BoSWL23}, which are the current state of the art models in practice.

Despite the empirical success of graph transformers, there is little theoretical understanding of the benefits of the attention mechanism in graph learning. The rich literature on the expressivity of GNNs shows that transformers are more expressive than standard message-passing GNNs, but less expressive than more general message-passing architectures, like spectrally invariant GNNs \citep{DBLP:conf/icml/ZhangZM24}. However, this line of work does not explain: \emph{Why do graph transformers perform better with more layers and not suffer from oversmoothing, similar to message-passing GNNs?}
Notably, \emph{oversmoothing}---the phenomenon of node representations becoming indistinguishable in deeper layers---remains a topic of debate in recent theoretical works.
\citet{DBLP:conf/nips/WuAWJ23} use a dynamical system-based equivalence of GATs to argue that they cannot prevent oversmoothing, while a Markov chain perspective of GATs suggests their ability to mitigate oversmoothing \citet{DBLP:conf/faiml/ZhaoWHG23}.  
Unlike GAT, which is limited to local neighborhoods, global attention-based models, like Graphormer and Specformer, allow nodes to attend to all others, making them more complex to analyze, with no theoretical work currently available. 
This lack of a unified theoretical lens for graph transformers leaves a fundamental question open: \emph{Can graph transformers overcome the structural bottlenecks of message-passing, or do they merely delay the onset of oversmoothing?}


We postulate that the above questions can be answered by studying the structural characteristics of graph transformers, specifically, \emph{how node representations evolve across layers}.
To this end,
we utilize the framework of Neural Network Gaussian Processes (NNGPs), which establishes that randomly initialized neural networks converge in distribution to Gaussian processes in the infinite-width limit \citep{lee2018deep}. 
The NNGP limit allows one to simplify GNNs, but still retain the architectural properties related to information aggregation and attention mechanism.
In particular, recent works derive NNGP and related neural tangent limits of Graph Convolutional Networks \citep{DBLP:conf/iclr/NiuA023,DBLP:journals/tmlr/SabanayagamEG23}, which has also been used to characterize oversmoothing in these models. However, NNGP limits of graph transformers are not known.

\begin{wrapfigure}{r}{0.5\textwidth} 
    \centering
    \vspace{-10pt} 
    
    \begin{minipage}{0.25\textwidth}
        \centering
        \includegraphics[width=\linewidth]{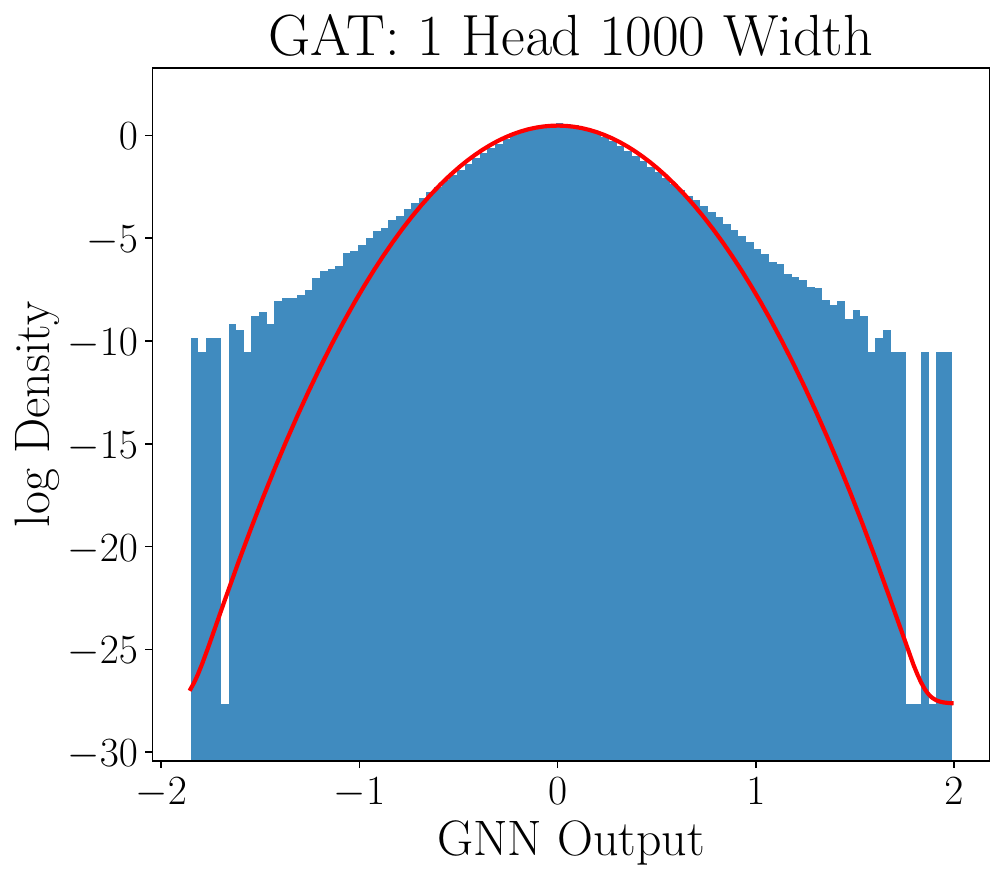}
    \end{minipage}%
    \hfill
    \begin{minipage}{0.25\textwidth}
        \centering
        \includegraphics[width=\linewidth]{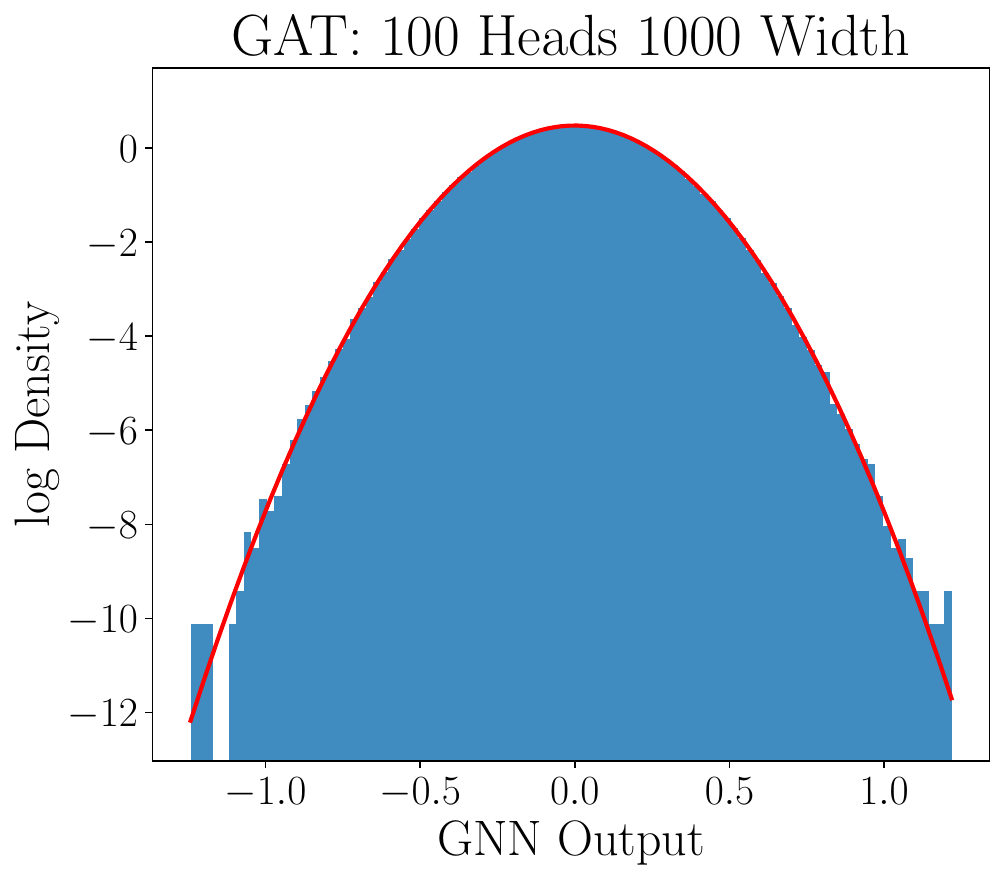}
        
    \end{minipage}

    \vspace{0.2cm}

    \begin{minipage}{0.25\textwidth}
        \centering
        \includegraphics[width=\linewidth]{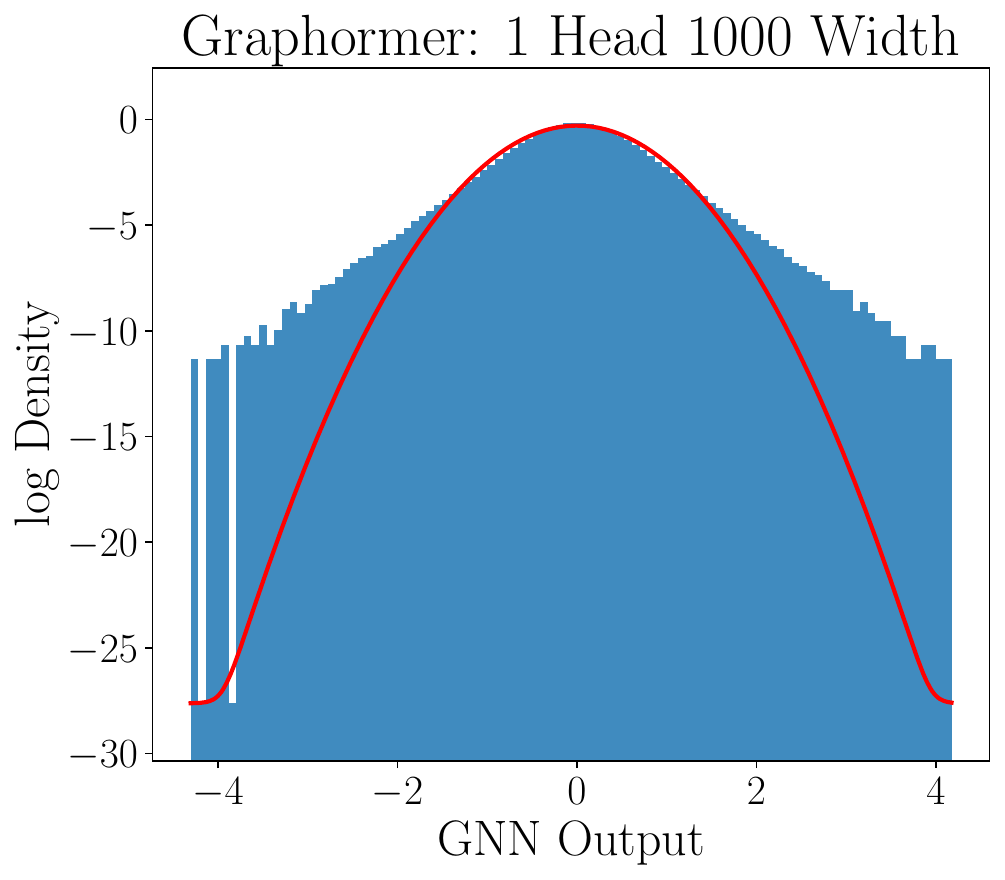}
    \end{minipage}%
    \hfill
    \begin{minipage}{0.25\textwidth}
        \centering
        \includegraphics[width=\linewidth]{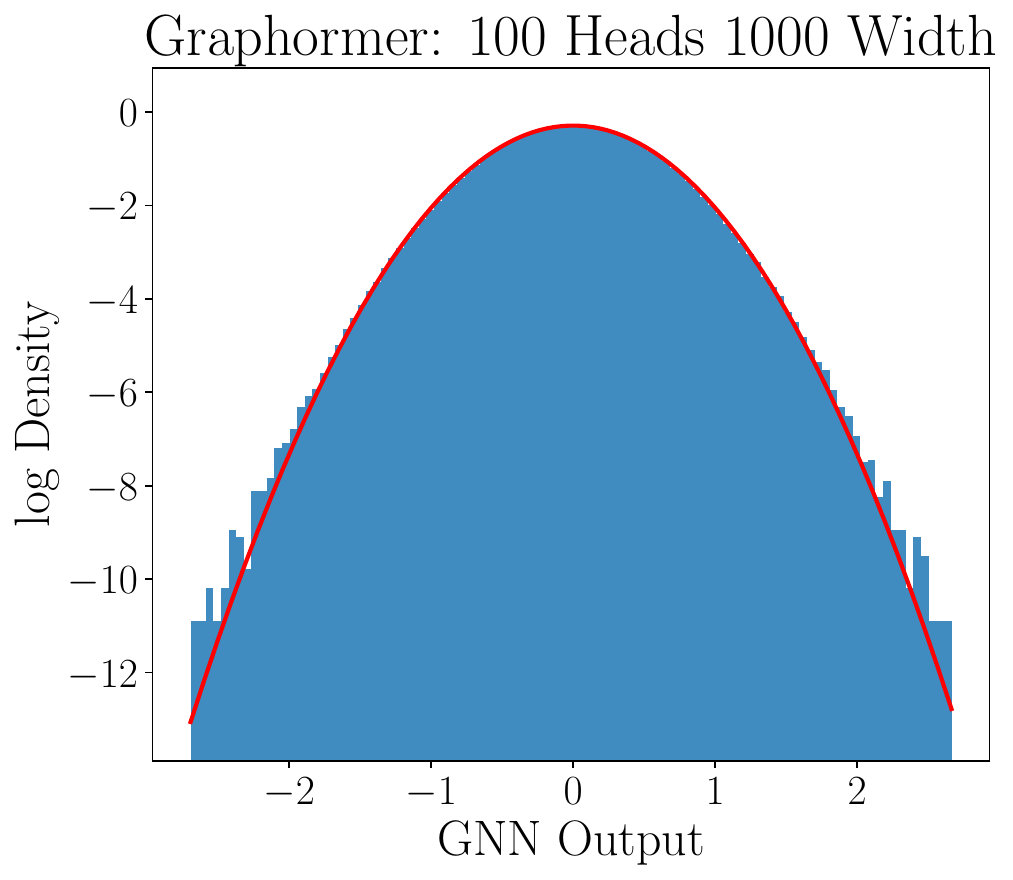}
    \end{minipage}

     \caption{Histogram of the output of GAT and Graphormer for different number of heads. The output distribution converges to a Gaussian (red line) fitted with mean and variance of the empirical distribution when both width and number of heads are large. Plots for Specformer is in Appendix~\ref{app:experimental details}.}
    \label{fig:distribution_grid}
    \vspace{-10pt} 
\end{wrapfigure}

\textbf{Contributions and significance.}
The main technical contribution of this work is a derivation of the NNGP kernels for graph transformers, which we use to analytically show when these architectures can mitigate oversmoothing. We briefly discuss the significance of the contributions.

\underline{\emph{Derivation of NNGP kernel for graph transformers.}}
We derive the NNGP limit of the three most relevant graph transformer architectures (GAT, Graphormer, and Specformer) and provide analytical forms of the kernel evolution across layers for GAT-GP (Theorem \ref{thm:GAT_GP}), Graphormer-GP (Theorem \ref{thm:Graphormer_GP}), and Specformer-GP (Theorem \ref{thm:Specformer_GP}). 
In contrast to previous work on GCN, where NNGP-equivalence holds in the infinite-width limit \citep{DBLP:conf/iclr/NiuA023},  transformer architectures behave as NNGPs only when both the width and the number of heads are large (see Figure \ref{fig:distribution_grid}). Furthermore, unlike the NNGP limit of standard transformers \citep{InfiniteAttention}, we show that graph transformers induce graph-specific kernels at multiple levels---node-level kernels reflect the node feature interactions, whereas structural kernels capture connectivity patterns, which correspond to edge-level kernels over graph neighborhoods in spatial models (GAT, Graphormer) and spectral kernels in spectral models (Specformer). These kernels interact non-trivially for various architectures (see Tables \ref{tab:kernels}--\ref{tab:linear_kernels}). 

\underline{\emph{Theoretical analysis of oversmoothing.}}
To study whether graph transformers mitigate oversmoothing, we simplify the NNGP kernels under the population version of (contextual) stochastic block model (CSBM), that is, graphs with well-defined community structure.
Our theoretical results, summarized in Table \ref{tab:sbm_summary}, confirm that while GCN is structurally destined for oversmoothing, all attention-based GNNs can retain community information across layers under specific conditions.
In particular, GAT-GP avoids rank collapse under well-separated communities, whereas Specformer-GP and Graphormer-GP counteract representational decay through spectral amplification or informative structural priors. 
The analysis provides a crucial insight: if the NNGP limit of graph transformers does not oversmooth, then their neural architecture could potentially be trained to avoid oversmoothing noted in prior work \citep{DBLP:conf/nips/WuAWJ23}.
\underline{\emph{Empirical validation of the impact of depth.}}
We validate our theoretical insights through numerical studies on synthetic and real-world benchmarks.
As a practical consequence of our theory, we identify Laplacian-based positional encodings as a structural prior that prevents oversmoothing in Graphormer-GP. Moreover, with a sufficiently informative structural prior, performance can even improve as depth increases, which is unexpected in GCNs.

\section{Preliminaries}
\label{sec:preliminaries}
\textbf{Notation.} 
$\odot$ denotes the Hadamard product, $\Vert$ concatenation, and $|\cdot|_F^2$ the squared Frobenius norm, and $\mathrm{tr}(\cdot)$ the trace of a matrix. $\phi(\cdot)$ is the row-wise softmax and  $\sigma(\cdot)$ an arbitrary activation function.

\textbf{Graph Data.} 
Let $\mathcal{G} = (\mathcal{V}, \mathcal{E})$ be an undirected graph with $n$ nodes, where $\mathcal{V}$ denotes the set of nodes and $\mathcal{E}$ the set of edges.
Let $A \in \mathbb{R}^{n \times n}$ denote the adjacency matrix of the graph. The neighborhood of a node $a \in \mathcal{V}$ is defined as
\(\mathcal{N}_a := \{\, i \in \mathcal{V} \mid (a,i) \in \mathcal{E} \,\}.\) The $ d_{\mathrm{in}}$-dimensional features for all $N$ nodes are collected in rows of the $n \times d_{\mathrm{in}}$ matrix $X$, with the $a$-th row denoted by $x_a^\top$.

\textbf{Graph Neural Networks (GNNs).} 
GNNs are a class of deep learning models designed to learn node representations from graph-structured data by aggregating information from local neighborhoods. $S_{\mathrm{GNN}}^{(\ell,h)} \in \mathbb{R}^{n \times n}$ denotes the graph convolution operator of head $h$ in layer $\ell$, whose form depends on the specific GNN architecture. For message-passing GNNs such as GCN, \( d^{\ell,H}=1\) and the graph convolution is fixed, determined solely by the graph topology, given by
\(
S_{\mathrm{GCN}} = D^{-1/2}AD^{-1/2}.
\)
In comparison, graph transformers employ a data-dependent convolution in which neighborhood weights are learned via an attention mechanism. Given input features \(f^{0}(X) := X\), pre-activation of the \(\ell\)-th layer is
\begin{equation}
g^{\ell}(X)
:= \Big(\big\|_{h=1}^{d^{\ell,H}} S_{\mathrm{GNN}}^{(\ell,h)} f^{\ell-1}(X)\, W^{\ell,h}\Big) W^{\ell,H},
\end{equation}
where $\|$ denotes concatenation of the attention heads and $W^{\ell,h} \in \mathbb{R}^{d_{\ell-1} \times d_{\ell}}$ and $W^{\ell,H} \in \mathbb{R}^{d_{\ell-1} d^{\ell,H} \times d_{\ell}}$ are  learnable weight matrices. The layer output is obtained by applying a nonlinearity. ReLU activation is a common choice and is considered in this work: \(f^{\ell}(X) = \mathrm{ReLU}\big(g^{\ell}(X)\big).\) 

\textbf{Gaussian Process Limit of GNN Node Representations.} 
We analyse the node representations learned by GNNs through the lens of \textit{Neural Network Gaussian Processes} (NNGPs). In the infinite-width and head limit, the model's pre-activations $\{ g_a^{\ell}(X) \}_{a \in V}$ converge in distribution to a centered Gaussian process
$g^{\ell}(X) \sim \mathcal{GP}(0, \Sigma^{(\ell)})$. For ReLU activations (\(f^{\ell}(X) = \mathrm{ReLU}\big(g^{\ell}(X)\big)\)), the post-activation kernel $K^{(\ell)}$ admits a closed-form recursive update \citep{DBLP:conf/nips/BiettiM19}:
\begin{align}
K^{(\ell)}_{ab}
&=
\frac{1}{2\pi}
\sqrt{
\Sigma^{(\ell)}_{aa}\,
\Sigma^{(\ell)}_{bb}
}
\Big(
\sin \theta^{(\ell)}_{ab}
+
\bigl(\pi - \theta^{(\ell)}_{ab}\bigr)
\cos \theta^{(\ell)}_{ab}
\Big),
\nonumber
\\
\theta^{(\ell)}_{ab}
&=
\arccos
\!\left(
\frac{
\Sigma^{(\ell)}_{ab}
}{
\sqrt{
\Sigma^{(\ell)}_{aa}\,
\Sigma^{(\ell)}_{bb}
}
}
\right).
\label{eq:relu-nngp-kernel}
\end{align}
Starting from the initial kernel $K^{(0)}_{ab} := \frac{ x_a \cdot x_b}{d_{\mathrm{in}}},
$
the sequence of kernels $\{\Sigma^{(\ell)}\}_{\ell \ge 0}$ and $\{K^{(\ell)}\}_{\ell \ge 0}$ are defined recursively.

Unlike standard NNGPs, in which the kernel is solely based on node features, the GNNs incorporate the graph structure and obtain predictions on unlabeled nodes by conditioning on observed node labels. In this work, we characterise the NNGP induced by graph transformers by deriving the corresponding covariance updates, thereby placing models such as GAT, Graphormer, and Specformer within a unified Gaussian process framework.

In the following, we introduce graph transformers investigated in this paper: GAT, Graphormer, and Specformer. In contrast to these models, which operate on homogeneous graphs, the Graph Transformer Network (GTN) is designed for heterogeneous graphs and is presented in Appendix~\ref{app:GTN}, including its model formulation, GP-equivalent theorem, and proof. The design choices underlying GAT, Graphormer, and Specformer are detailed in Appendix \ref{app:design}.

\textbf{Graph Attention Network (GAT).}
A multi-head GAT consists of $d^{\ell,H}$ attention heads in layer $\ell$. For each head $h$ and layer $\ell$, attention scores are computed on edges $(a,i)\in \mathcal{E}$:
\begin{equation}
E^{(\ell h)}_{ai}\!(X)
=
\bigl(\vec{v}^{\ell h}\bigr)^\top
\!\!\left[
W^{\ell,h} f_a^{\ell-1}(X)\, \Vert \,
W^{\ell,h} f_i^{\ell-1}(X)
\right]
\end{equation}
where $\vec{v}^{\ell h} \in \mathbb{R}^{2d_\ell \times 1}$ is the attention vector.

The normalized attention coefficients are obtained via
\begin{equation}
\bigl(S_{\mathrm{GAT}}^{(\ell,h)}(X)\bigr)_{ai}
=
\frac{
\exp\!\big(\sigma(E^{(\ell h)}_{ai}(X)) \big)
}{
\sum_{k \in \mathcal{N}_a}
\exp\!\big( \sigma(E^{(\ell h)}_{ak}(X)) \big)
},
\label{Eq:gat_softmax}
\end{equation}
and $\bigl(S_{\mathrm{GAT}}^{(\ell,h)}(X)\bigr)_{ai}=0$ for $(a,i) \notin \mathcal{E}$.

\textbf{Graphormer.} A Graphormer is a multi-head attention-based architecture with $d^{\ell,H}$ heads in layer $\ell$. Each layer augments node representations with graph positional encodings:
$\tilde f_i^{\ell-1}(X):=[f_i^{\ell-1}(X),\,P_i^{\ell}]$ for $i\in V$, where $P_i^{\ell}$ encodes structural information such as Laplacian eigenvectors, degree features, or random-walk-based statistics.
For each head $h$ and layer $\ell$, attention scores are computed on all
node pairs $(i,j)\in V\times V$ as
\begin{equation}
\small
E^{(\ell h)}_{ij}(X)
=
\frac{
\big(\tilde f^{\ell-1}_i(X) W^{Q,\ell h}\big)\,
\big(\tilde f^{\ell-1}_j(X) W^{K,\ell h}\big)^{\top}
}{\sqrt{d_\ell}}
\;+\;
b_{\rho(i, j)},
\label{eq:attention_score_graphormer}
\end{equation}
where $W^{Q,\ell h},W^{K,\ell h}\in\mathbb{R}^{d_{\ell-1}\times d_{\ell}}$
are learnable projection matrices, and the scaled dot-product term measures content-based
compatibility between node $i$ and node $j$.
The additive bias $b_{\rho(i,j)}$ is a learnable scalar that depends
only on the structural relation $\rho(\cdot,\cdot):V\times V\to\mathcal{R}$ (e.g., shortest-path
distance bucket, edge type, or other graph-structural categories).

The graph convolution operator is obtained by applying the row-wise softmax, to the logits, i.e.,
$S_{\mathrm{Graphormer}}^{(\ell,h)}(X)=\phi\!\big(E^{(\ell h)}(X)\big)$ for head $h$.

\textbf{Specformer.}
Specformer operates on the spectrum of the normalized graph Laplacian: \(L:=I_n-D^{-1/2}AD^{-1/2}=U\Lambda U^\top\) with
\(\Lambda=\mathrm{diag}(\lambda_1,\dots,\lambda_n)\).
Each eigenvalue $\lambda_i$ is encoded by a sinusoidal map $\rho(\lambda_i)\in\mathbb{R}^{d}$ with
\(\rho(\lambda,2t)
=\sin\!\Big(\epsilon\lambda/10000^{2t/d}\Big), 
\rho(\lambda,2t+1)
=\cos\!\Big(\epsilon\lambda/10000^{2t/d}\Big)
\)

and the initial spectral tokens are formed as \(H^{0}
=
\Big[(\lambda_1\Vert\rho(\lambda_1))^\top,\dots,(\lambda_n\Vert\rho(\lambda_n))^\top\Big]^\top
\in\mathbb{R}^{n\times (d+1)}.\)

Consider a multi-head Specformer consisting of an \emph{eigenvalue encoding} module with $d^{t,H}$ heads and a \emph{graph
convolution} module with $d^{\ell,H}$ heads. For each graph convolution head $h\in\{1,\dots,d^{\ell,H}\}$, the eigenvalue
encoding module applies a $d^{t,H}$-head attention at each eigenvalue encoding layer $t$, producing head-specific attention scores on token
pairs $(i,j)\in V\times V$:
\begin{equation}
E^{(t,k,h)}_{ij}
=
\frac{\big(H^{t-1,k,h}_{i}W^{Q,t k h}\big)\big(H^{t-1,k,h}_{j}W^{K,t k h}\big)^{\top}}{\sqrt{d_t}}.
\end{equation}
Here, $k\in{1,\dots,d^{t,H}}$ indexes the attention head within token layer $t$, and $W^{Q,t k h},W^{K,t k h}\in
\mathbb{R}^{d_{t-1}\times d_t}$ are the corresponding projection matrices. The normalized attention coefficients are given by
$S_{\mathrm{Spec}}^{(t,k,h)}=\phi\!\big(E^{(t,k,h)}\big)$. The token-layer update for the $h$-th graph convolution head is
\begin{equation}
H^{t,h}
=
\Big(\big\|_{k=1}^{d^{t,H}}
S_{\mathrm{Spec}}^{(t,k,h)}\,H^{t-1,k,h} W^{V,t kh}\Big)W^{th,H},
\end{equation}
where $W^{V,t k h}\in\mathbb{R}^{d_{t-1}\times d_t}$ is the value projection for the $k$-th eigenvalue encoding head and
$W^{t,H}\in\mathbb{R}^{d^{t,H}d_t\times d_{t}}$ is the token-layer output projection. After $T$ eigenvalue layers, the aggregated
eigenvalue encoding output associated with the $h$-th graph convolution head is denoted by $\bar H^{(h)}=H^{T,h}$.

The graph convolution module then constructs a multi-head spectral filtering operator. For each graph convolution head
$h\in\{1,\dots,d^{\ell,H}\}$, a spectral filter is generated by \(\bar\lambda^{(h)}
=
\sigma\!\big(\bar H^{(h)} W_{\lambda}^{(h)}\big)
\in\mathbb{R}^{n},\) where $W_{\lambda}^{(h)} \in \mathbb{R}^{d_{T}\times 1}$ is a learnable parameter vector associated with the $h$-th graph
convolution head. This induces the graph convolution operator
\(
\label{eq:S_specformer}
S_{\mathrm{Specformer}}^{(\ell,h)}
=
U\,\mathrm{diag}\!\big(\bar\lambda^{(h)}\big)\,U^\top
\in\mathbb{R}^{n\times n}.\)

\section{Gaussian process Limits of Graph Transformers}
In this section, we analyse the GP limits of Graph Transformers architectures introduced in Section~\ref{sec:preliminaries}. We begin by introducing the kernels that govern the Gaussian process limit and explain their roles, then present their explicit formulas in Table~\ref{tab:kernels}, and finally interpret their implications for the corresponding architectures in Section~\ref{Ch:GP_Remark}. 

Our analysis is carried out within the Neural Network Gaussian Process (NNGP) framework. Specifically, we characterise the behavior of GAT, Graphormer, and Specformer in the \textbf{infinite-width and infinite-head limit}, denoted as $d_\ell, d_{\ell,H} \to \infty$, and under the independent Gaussian priors listed in Table~\ref{tab:weight_priors}. In this regime, the models admit kernel recursions that govern their layerwise propagation. In particular, we focus on the pre-activations $g^{\ell}$ with kernel $\Sigma^{(\ell)}$, and the post-activation kernel $K^{(\ell)}$ is obtained via the standard ReLU NNGP transformation (Equation~\ref{eq:relu-nngp-kernel}). The complete theorems are provided in Appendix~\ref{App:GP_Theorems}.

 \paragraph{Spatial Models: GAT-GP and Graphormer-GP} To provide a unified analysis of GAT and Graphormer, we define four kernels that govern the GP limit. These kernels model the propagation of information: we define the structural input (post-activation), model the attention mechanism (edge-level), aggregate these into a graph operator (convolution), and finally produce the node representations (node-level). Together, they form a recursion that characterizes the evolution of graph signals through deep layers.
\begin{itemize}
     \item \textbf{Post-Activation Kernel of the Previous Layer:} The propagation is initialized using the kernel from the previous iteration. For GAT, this is the post-activation kernel $K^{(\ell-1)}$. For Graphormer, this is the \textbf{augmented kernel} $\tilde{K}^{(\ell)} = \alpha K^{(\ell-1)} + (1-\alpha)R^{(\ell)}$ for $\alpha \in [0,1]$, which incorporates structural positional encoding $R^{(\ell)}$.
    \item \textbf{Edge-Level Logit Kernel ($K^{(\ell), E}$):} The attention logits $E^{(\ell)}$ converge to a centered Gaussian process $E^{(\ell)} \sim GP(0, K^{(\ell), E})$. This kernel describes the covariance between attention scores for different edges.
    \item \textbf{Convolution Kernel ($C^{(\ell)}$):} This kernel captures the covariance of the graph convolution operator across independent heads and is defined as \(
    C^{(\ell)}_{ai,bj} :=\mathbb{E}\Big[
  \bigl(S^{(\ell,1)}(X)\bigr)_{ai}\,
  \bigl(S^{(\ell,1)}(X)\bigr)_{bj}
\Big] \). Its explicit form is determined by the Edge-Level Logit Kernel $K^{(\ell),E}$.
    \item \textbf{Node-Level Kernel ($\Sigma^{(\ell)}$):} The final covariance between node representations $a$ and $b$ after neighborhood aggregation, where $\Sigma^{(\ell)}_{ab} := \mathbb{E}[g^\ell_{a1}(X) g^\ell_{b1}(X)]$.
\end{itemize}

\paragraph{Spectral Model: Specformer-GP}
Unlike spatial models, Specformer operates on the spectrum of $L$. Alongside the layerwise limits, when $d_t, d_{t,H} \to \infty$, the model converges to a Gaussian process characterized by a set of kernels:

\begin{itemize}
    \item \textbf{Spectral Token Kernel:} The representations of the $t$-th token layer, $H^{(t)}$, converge in distribution to a Gaussian process $H^{(t)} \sim GP(0, K_H^{(t)})$, which captures the evolution of spectral information through the eigenvalue encoding.
    \item \textbf{Learned Spectral Kernel ($K_\lambda$):} After $T$ token layers,  the learned spectral filter coefficients  $\bar{\lambda}$, induces a kernel defined as $K_{\lambda, ab} := \mathbb{E}[\bar{\lambda}_a \bar{\lambda}_b]$.
    \item \textbf{Node-Level Kernel ($\Sigma^{(\ell)}$)} This kernel propagates feature covariance by mapping previous-layer representations into the Laplacian eigenbasis for elementwise modulation by the induced spectral kernel ($K_\lambda$) before projecting them back to the node domain.
\end{itemize}

The specific kernels for each architecture are summarized in Table~\ref{tab:kernels}. The independent Gaussian weight priors considered are summarized in Table~\ref{tab:weight_priors} at the beginning of Appendix~\ref{App:GP_Theorems}. Formal theorems establishing the existence of the Gaussian process limits are provided in Appendix~\ref{App:GP_Theorems}, and detailed proofs of these theorems can be found in Appendix~\ref{app:proof of GP}.

\paragraph{Explicit Kernel Expressions.} Table~\ref{tab:kernels} summarizes a general kernel recursion, which can be specialized by considering specific choices of the attention and activation nonlinearities to obtain a closed-form expressions. Table~\ref{tab:linear_kernels} presents the closed-form kernels obtained under linear attention (formally in Corollaries~\ref{cor:GAT_kernel_linearized}, and \ref{cor:graphormer-linear}). For completeness, and to facilitate the comparison with attention-based architectures in the Section~\ref{Ch:Oversmoothing}, we also include the corresponding GCN-GP kernel from \citep{DBLP:conf/iclr/NiuA023}, which does not incorporate attention mechanisms. These kernels provide a tractable way to study the GP limit under simplified attention and activation choices, while still capturing the interaction between node features and graph structure. More general kernels can be obtained by considering RELU as output activation and applying Equation \ref{eq:relu-nngp-kernel}. In Appendix~\ref{App:GP_Theorems} Corollary~\ref{cor:GAT_kernel_closed_form}, we also present GAT with RELU attention.

\begin{table*}[t] 
\centering
\renewcommand{\arraystretch}{2} 
\caption{Summary of the GP kernels for graph transformers. The weight priors are in Table~\ref{tab:weight_priors} in Appendix~\ref{App:GP_Theorems}. 
The table presents both the edge/structural kernels ($K^{(\ell),E}_{ab}$ for spatial models or $K_{\lambda, ab}$ for Specformer) 
and the node-level kernels ($\Sigma^{(\ell)}_{ab}$) that describe the propagation of covariance across layers. 
For GAT-GP and Graphormer-GP, the edge-level kernels capture correlations between attention scores, while the node-level kernels aggregate these correlations.
The learned spectral kernel $K_\lambda$ of Specformer-GP modulates the mapping of node representations in the Laplacian eigenbasis.} 

\label{tab:kernels}
\begin{tabular}{@{}lll@{}}
\toprule
\textbf{Model} & \textbf{Edge/Structural Kernel ($K^{(\ell),E}_{ab}$ or ${K_\lambda}_{, ab}$)} & \textbf{Node Kernel $\Sigma^{(\ell)}_{ab}$} \\ \midrule
GAT-GP (Theorem \ref{thm:GAT_GP})& \(\sigma_w^2\sigma_v^2
   \bigl(K^{(\ell-1)}_{ab}+K^{(\ell-1)}_{ij}\bigr)\) &\(\sigma_H^2\sigma_w^2
\sum_{i\in N_a}\sum_{j\in N_b}
K^{(\ell-1)}_{ij}\,
C^{(\ell)}_{ai,bj}\) \\
Graphormer-GP (Theorem \ref{thm:Graphormer_GP})& \(\Big(
\sigma_Q^2\sigma_K^2\,
\tilde K^{(\ell-1)}_{ab}\,
\tilde K^{(\ell-1)}_{ij}
+
\sigma_b^2\,
\mathbf 1_{\phi(a,i)=\phi(b,j)}
\Big) \)& $\sigma_H^2 \sigma_w^2 \sum_{i,j \in V} \tilde{K}^{(\ell-1)}_{ij} C^{(\ell)}_{ai,bj}$ \\
Specformer-GP (Theorem \ref{thm:Specformer_GP})& $K_{\lambda, ab} = \mathbb{E}[\bar{\lambda}_a \bar{\lambda}_b]$ (Spectral) & $\sigma_H^2 \sigma_w^2 U (K_{\lambda} \odot (U^\top K^{(\ell-1)} U)) U^\top$  \\
\bottomrule
\end{tabular}
\end{table*}

\subsection{Remarks on GP kernels}
\label{Ch:GP_Remark}
The kernels in Table~\ref{tab:kernels} reveal how node features and graph structure interact in the GP limit. In the following, we discuss the interpretation of these kernels in the graph domain and their implications for node representation propagation.
\paragraph{\textbf{GAT-GP.}} 
The edge-level kernel $K^{(\ell),E}$ of GAT-GP characterizes the covariance between attention logits on two edges $(a,i)$ and $(b,j)$, determined by the similarity of their source nodes $a$ and $b$ and their target nodes $i$ and $j$. Edges whose endpoints are simultaneously similar in representation space exhibit stronger correlation. The node-level kernel $\Sigma^{(\ell)}$ then aggregates these edge contributions across all neighbor pairs $(i,j)$ with $i\in\mathcal{N}_a$ and $j\in\mathcal{N}_b$, combining the previous-layer feature similarity with the alignment induced by attention. This yields a kernel recursion describing how both features and graph structure govern node representation propagation across layers. 
\paragraph{\textbf{Graphormer-GP}}
Graphormer incorporates graph information
through positional embeddings.
The matrix $R$ can be interpreted as the covariance matrix
of the corresponding positional embeddings which can be computed as an inner product, i.e.,
\(
R^{(\ell)}_{ab} = \langle P_a^{(\ell)}, P_b^{(\ell)} \rangle
\).
Alternatively, instead of introducing explicit positional embeddings,
the graph structure itself can be directly leveraged to construct a
node-level graph-structure covariance matrix. The node-level kernel $\Sigma^{(\ell)}$ of Graphormer-GP indicates that attention correlations are strengthened for node pairs sharing the same structural relation, ensuring that graph information is explicitly leveraged in guiding attention. The edge-level kernel $K^{(\ell),E}$ then propagates the attention  correlations to the next layer through an aggregation over all node pairs.

\begin{mybox2}{hellblue}{}
 Unlike GAT which is restricted to local neighborhoods, Graphormer induces similarity connections between every pair of nodes in the graph.
\end{mybox2}

\paragraph{\textbf{Specformer.}} The convolution in Specformer is composed of \emph{spectral tokens} and \emph{node features}. 
The spectral tokens are first processed through $T$ layers of encoding, and passed through a decoder to construct a data-dependent spectral filtering that governs the propagation of node features. As detailed in Theorem \ref{thm:Specformer_GP} in Appendix~\ref{App:GP_Theorems}, (I) the spectral-token representations $H^t$, which are updated through attention, and (II) the node-feature pre-activations $g(X)$ converge in distribution to a Gaussian process.
 The learned spectral kernel $K_{\lambda}$ induces a kernel over the
learned spectral filter coefficients. The node level kernel $\Sigma^{l}$ then defines the recursion of the node-feature kernel: the previous-layer node covariance is
mapped into the Laplacian eigenbasis, modulated elementwise according to the induced
spectral kernel, and projected back to the node domain.

\begin{mybox2}{hellblue}{}
Spatial models aggregate information over local node neighborhoods, whereas Specformer operates in the graph spectral domain.
\end{mybox2}

\begin{table*}[t]
\renewcommand{\arraystretch}{2}
\centering
\caption{Closed-form expressions for the node-level Gaussian process (GP) kernels under linear attention and linear activation. 
These kernels provide an explicit characterization of kernel propagation for each architecture. 
For comparison, the GCN-GP kernel from \citep{DBLP:conf/iclr/NiuA023} is included, which does not involve attention.}

\label{tab:linear_kernels}
\begin{tabular}{@{}ll@{}}
\toprule
\renewcommand{\arraystretch}{2}
\textbf{Model} &\textbf{Node Kernel ($K^{(\ell)}$)} \\ \midrule
GCN-GP \citep{DBLP:conf/iclr/NiuA023}&  $K^{(\ell)} = \sigma_w^2AK^{\ell-1}A^T$ \\
GAT-GP (Corollary \ref{cor:GAT_kernel_linearized}) &  $K^{(\ell)} = \sigma_H^2\sigma_w^4\sigma_v^2 \Big(K^{(\ell-1)} \odot (A K^{(\ell-1)} A^\top) + A(K^{(\ell-1)} \odot K^{(\ell-1)})A^\top \Big)$ \\
Graphormer-GP (Corollary \ref{cor:graphormer-linear})& $K^{(\ell)}_{ab}
=
\sigma_H^2\sigma_w^2
\Bigg[
\sigma_Q^2\sigma_K^2\,
\tilde K^{(\ell)}_{ab}
\sum_{i,j\in V}
\tilde K^{(\ell-1)}_{ij}\,
\tilde K^{(\ell-1)}_{ij}+
\sigma_b^2
\sum_{i,j\in V}
\tilde K^{(\ell-1)}_{ij}\,
\mathbf 1_{\rho(a,i)=\rho(b,j)}
\Bigg]$
\\
Specformer-GP (Theorem \ref{thm:Specformer_GP}) & $K^{(\ell)} = \sigma_H^2\sigma_w^2 U (K_\lambda \odot (U^\top K^{(\ell-1)} U)) U^\top$ \\
\bottomrule
\end{tabular}
\end{table*}

\section{Analysis of GNN-GPs under Contextual Stochastic Block Model (CSBM)}
\label{Ch:Oversmoothing}
The GP formulation of graph transformers allows us to compare structural benefits and limitations of different architectures. Having introduced kernel recursion of graph transformers, we now focus on a key phenomenon in deep graph models: \textbf{Oversmoothing}.
In the kernel perspective, this corresponds to convergence of the node-level covariance toward a constant matrix. To understand whether graph transformers can structurally maintain discriminative signals across layers, it is useful to study their GP kernels on a controlled setting with known community structure.

In this section, we analyse how the resulting kernels behave when the underlying graph is governed by a \textbf{CSBM}. We consider population version of 2-CSBM with $n$ nodes partitioned into two equally sized communities, $\mathcal{C}_1$ and $\mathcal{C}_2$. In this setting, the probability of an edge existing between two nodes is $p$ if they belong to the same community and $q$ if they belong to different communities. The expected adjacency matrix takes the following block structure:
\(
A = 
\begin{pmatrix} 
p \mathbf{1}\mathbf{1}^\top & q \mathbf{1}\mathbf{1}^\top \\ 
q \mathbf{1}\mathbf{1}^\top & p \mathbf{1}\mathbf{1}^\top 
\end{pmatrix},
\) where $\mathbf{1}\mathbf{1}^\top$ denotes the all-ones matrix of size $\frac{n}{2} \times \frac{n}{2}$ \footnote{Throughout this section, SBM refers to the two-community stochastic block model introduced above.
}. A kernel that captures the community structure should mirror this block structure, i.e. \(K^{(\ell)} = 
\begin{pmatrix} 
x_\ell \mathbf{1}\mathbf{1}^\top & y_\ell \mathbf{1}\mathbf{1}^\top \\ 
y_\ell \mathbf{1}\mathbf{1}^\top & x_\ell \mathbf{1}\mathbf{1}^\top 
\end{pmatrix}\), where $x_\ell$ represents the kernel value for intra-community node pairs and $y_\ell$ for inter-community pairs at layer $\ell$, initialized by the feature covariance values $x_0$ and $y_0$ at the input layer ($\ell=0$). While we analyse the population version, our experiments use random CSBM realizations, which exhibit similar behavior. We summarise the evolution of the kernels (formulas for $x_\ell$ and $y_\ell$) and their asymptotic behaviour in Table~\ref{tab:sbm_summary}, highlighting how each model propagates structural information under the CSBM. The full corollaries and proofs for the closed-form expressions of the kernels are provided in Appendix~\ref{App:SBM_Theorems}.

\begin{table*}[h]
\centering
\renewcommand{\arraystretch}{2.5} 
\caption{\textbf{Evolution of the GP kernel $K^{(\ell)} = \left( \begin{smallmatrix} x_\ell \mathbf{1}\mathbf{1}^\top & y_\ell \mathbf{1}\mathbf{1}^\top \\ y_\ell \mathbf{1}\mathbf{1}^\top & x_\ell \mathbf{1}\mathbf{1}^\top \end{smallmatrix} \right)$ under CSBM and conditions for oversmoothing.} 
To preserve community structure, the kernel must maintain a gap between the intra-community ($x_\ell$) and inter-community ($y_\ell$) values; this discriminative signal is captured by the second term in each expression.
The second column indicates whether oversmoothing occurs as $\ell \to \infty$, i.e., whether the kernel loses its ability to distinguish the two communities. 
For graph transformers oversmoothing can be avoided, maintaining the discriminative signal across layers.}
\label{tab:sbm_summary}
\begin{tabular}{@{}lll@{}}
\toprule
\textbf{Model} & \textbf{Intra/Inter-Community Formulas ($x_\ell, y_\ell$)} & \textbf{Oversmoothing ($\ell \to \infty$)} \\ \midrule

GCN-GP & 
\makecell[l]{$\frac{1}{2}(\frac{n}{2})^{2\ell} \left[ (x_0+y_0)(p+q)^{2\ell} \pm (x_0-y_0)(p-q)^{2\ell} \right]$ \\ \footnotesize (Corollary~\ref{cor:GCN_kernel_SBM})} & 
\makecell[l]{Yes \\ \footnotesize (Corollary~\ref{cor:GCN_oversmoothing})} \\

GAT-GP & 
\makecell[l]{$\frac{1}{2} \frac{G^{2\ell}}{ (\frac{n}{2})^2 (p+q)^2} \left[ 1 \pm \left(\frac{x_0-y_0}{x_0+y_0}\right)F^\ell \right]$ \\ \footnotesize (Corollary~\ref{cor:GAT_kernel_SBM})} & 
\makecell[l]{No, if communities are separated ($F \ge 1$) \\ \footnotesize (Corollary ~\ref{cor:GAT_no_oversmoothing})} \\

Graphormer-GP &
\makecell[l]{$\tilde{x}_\ell = \alpha^\ell x_0 + (1-\alpha^\ell)p, \quad \tilde{y}_\ell = \alpha^\ell y_0 + (1-\alpha^\ell)q$ \\ \footnotesize (Corollary~\ref{cor:Graphormer_kernel_SBM})} & 
\makecell[l]{No, if the prior ($R$) captures communities \\ \footnotesize (Remark~\ref{rem:Graphormer_convergence})} \\

Specformer-GP & 
\makecell[l]{$\frac{1}{2} \left[ (x_0+y_0)\bar{\lambda}_1^{2\ell} \pm (x_0-y_0)\bar{\lambda}_2^{2\ell} \right]$ \\ \footnotesize (Corollary~\ref{cor:SpecFormer_kernel_SBM})} & 
\makecell[l]{No, if $|\bar{\lambda}_2| \ge |\bar{\lambda}_1|$ \\ \footnotesize (Remark~\ref{cor:Specformer_advantage})} \\

\bottomrule
\end{tabular}
\end{table*}

Oversmoothing can be formalized by examining the normalized kernel \(\frac{K^{(\ell)}} { {\frac{1}{n}tr(K^{(\ell)})}}\). Since the normalisation term corresponds to $x_\ell$, the discriminability of the model is captured by the ratio $y_\ell / x_\ell$: if $\lim_{\ell \to \infty} y_\ell / x_\ell = 1$, the kernel converges to a rank-one matrix of ones, meaning the representations of the two communities become indistinguishable.
\paragraph{GCN.} The collapse into a rank-one kernel is clearly visible in the Gaussian Process limit of standard GCN.
\begin{corollary}[Rank collapse in GCN-GP indicates oversmoothing]
\label{cor:GCN_oversmoothing}
For the GCN kernel presented in Corollary \ref{cor:GCN_kernel_SBM} (and in Table \ref{tab:sbm_summary}). The normalized kernel converges to the all-ones matrix:
\[
\lim_{\ell \to \infty} \frac{K^{(\ell)}}{\frac{1}{n}tr(K^{(\ell)})} =  \begin{pmatrix} 
\mathbf{1}\mathbf{1}^\top & \mathbf{1}\mathbf{1}^\top \\ 
\mathbf{1}\mathbf{1}^\top & \mathbf{1}\mathbf{1}^\top 
\end{pmatrix}.
\]
Consequently, $\text{rank}\left(\lim_{\ell \to \infty} \frac{K^{(\ell)}}{\frac{1}{n}tr(K^{(\ell)})}\right) = 1$, indicating that GCNs suffer from complete oversmoothing.
\end{corollary}
The detailed derivations for the closed-form expressions of the GCN kernel and the spectral analysis leading to rank collapse are provided in Appendix \ref{Ap:Cor_SBM_GCN}. 
\paragraph{GAT.} Unlike the GCN kernel, which collapses to a rank-one matrix, we next show that GAT-GP avoids oversmoothing. Notably, while the GCN kernel follows a simple linear recurrence, the GAT kernel involves a coupled, non-linear two-variable system. The derivation of the GAT kernel under SBM (Corollary \ref{cor:GAT_kernel_SBM}) and its limit (Corollary \ref{cor:GAT_no_oversmoothing}) are provided in Appendix \ref{Ap:GAT_SBM}. The GAT kernel is governed by the \textbf{global growth factor} $G = (x_0 + y_0)(p + q)^2 (\frac{n}{2})^2$ and the \textbf{structural preservation factor} $F = 2\frac{p^2 - pq + q^2}{(p + q)^2}$. Here, $G$ controls the overall magnitude of the kernel, while $F$ determines whether the community structure is preserved across layers.
In particular, when the communities are sufficiently well-separated, $F$ ensures that the kernel maintains its ability to distinguish between them even in the infinite-depth limit.  
This property is formalised in the following corollary:

\begin{corollary}[GAT-GP avoids rank collapse indicating the preservation of discriminative community structure]
\label{cor:GAT_no_oversmoothing}
For the GAT kernel defined in Corollary \ref{cor:GAT_kernel_SBM}, the normalized kernel converges to: 
\(
\lim_{\ell \to \infty} \frac{K^{(\ell)}}{\frac{1}{n}\text{tr}(K^{(\ell)})} = \begin{pmatrix} 
\mathbf{1}\mathbf{1}^\top & \gamma \mathbf{1}\mathbf{1}^\top \\ 
\gamma \mathbf{1}\mathbf{1}^\top & \mathbf{1}\mathbf{1}^\top 
\end{pmatrix}, \quad \text{where } \gamma = \lim_{\ell \to \infty} \frac{1 - \left( \frac{x_0 - y_0}{x_0 + y_0} \right) F^\ell}{1 + \left( \frac{x_0 - y_0}{x_0 + y_0} \right) F^\ell}.
\)
Consequently, if $F \geq 1$, which occurs when \(\frac{p}{q}\) is bounded away from 1, the kernel maintains a rank of $2$ and preserves community separation as $\ell \to \infty$, avoiding oversmoothing.
\end{corollary}

\paragraph{Graphormer.} Consider the Graphormer kernel from Corollary \ref{cor:graphormer-linear} in Appendix \ref{App:GP_Theorems} without the bias term, and assume that the positional encodings capture all structural information of the graph. In the SBM setting, this implies that the spatial relation matrix coincides with the adjacency matrix, i.e., $R = A$. Under SBM, the block entries $\tilde{x}_\ell$ and $ \tilde{y}_\ell$ of $\tilde{K}^{(\ell)}$ are determined by the recurrence given in Table~\ref{tab:sbm_summary}. The normalized kernel from Table~\ref{tab:linear_kernels} without the bias term can be written as $K^{(\ell)} = \tilde{K}^{(\ell)} \cdot Z_{\ell-1}$, with $Z_{\ell-1} = tr(A^\top (\tilde{K}^{(\ell-1)} \odot \tilde{K}^{(\ell-1)}))$. Since our analysis focuses on whether the kernel preserves the block structure induced by the SBM, the trace term does not play a role. Consequently, it suffices to study the evolution of $\tilde{K}^{(\ell)}$.

\begin{remark}[Graphormer Convergence]
\label{rem:Graphormer_convergence}
As the number of layers $\ell \to \infty$, the unnormalized Graphormer kernel $\tilde{K}^{(\ell)}$ converges to the spatial relation matrix $R$. In the SBM setting where $R = A$, the kernel effectively recovers the ground truth, improving its ability to distinguish communities with increasing depth. This benefit, however, relies entirely on the quality of the positional encodings; in a real-world dataset if $R$ is uninformative, the kernel may collapse toward a non-discriminative prior, resulting in oversmoothing.
\end{remark}

\paragraph{Specformer.} The evolution of the Specformer kernel under SBM is presented in Table~\ref{tab:sbm_summary}. Unlike the GCN and Graphormer kernels, the Specformer uses a learnable spectral filter, which is characterised by the parameters $\bar{\lambda}_1$ and $\bar{\lambda}_2$ that determine the cross-covariance $K_{\lambda}$ (Equation \ref{eq:specformer_lambda_kernel_T}). These eigenvalues control how different spectral components are amplified across layers, and therefore determine whether the kernel preserves the block structure induced by the SBM. Depending on the choice of $\bar{\lambda}_1$ and $\bar{\lambda}_2$, Specformer can either collapse to the GCN behaviour or maintain community separation, as detailed in the following remarks.

\begin{remark}[Specformer Collapses to GCN]
\label{rem:SpecFormer_GCN}
The Specformer-GP kernel collapses to the GCN-GP kernel if the learned spectral weights correspond to the eigenvalues of the adjacency matrix, i.e., $\bar{\lambda}_1 = \frac{n}{2}(p+q)$ and $\bar{\lambda}_2 = \frac{n}{2}(p-q)$. In this case, the Specformer kernel exhibits the same oversmoothing behaviour  (Corollary~\ref{cor:GCN_oversmoothing}).
\end{remark}

\begin{remark}
\label{cor:Specformer_advantage}
Specformer can preserve community structure, even in the limit $\ell \to \infty$, when the learned spectral filter satisfies
$|\bar{\lambda}_2| \geq |\bar{\lambda}_1|$, as this amplifies the spectral direction that carries the community information.
\end{remark}

\begin{mybox}{hellblue}{\textcolor{black}{Summary of Oversmoothing Behavior}}
While GCN-GP is structurally destined for rank collapse and oversmoothing, attention-based models provide mechanisms to preserve structural information. GAT-GP avoids oversmoothing under a data-dependent condition: communities need to be sufficiently separated, which is a standard expectation for graph learning tasks. Graphormer-GP relies on a prior that encodes meaningful structural information to prevent oversmoothing. In contrast, Specformer-GP offers a learnable solution, where avoiding oversmoothing is determined by the spectral weights optimized during training rather than by the data or prior. Together, these findings highlight how architectural design influence oversmoothing.
\end{mybox}

\begin{wrapfigure}{!}{0.48\columnwidth}
    \centering
    \includegraphics[width=\linewidth]{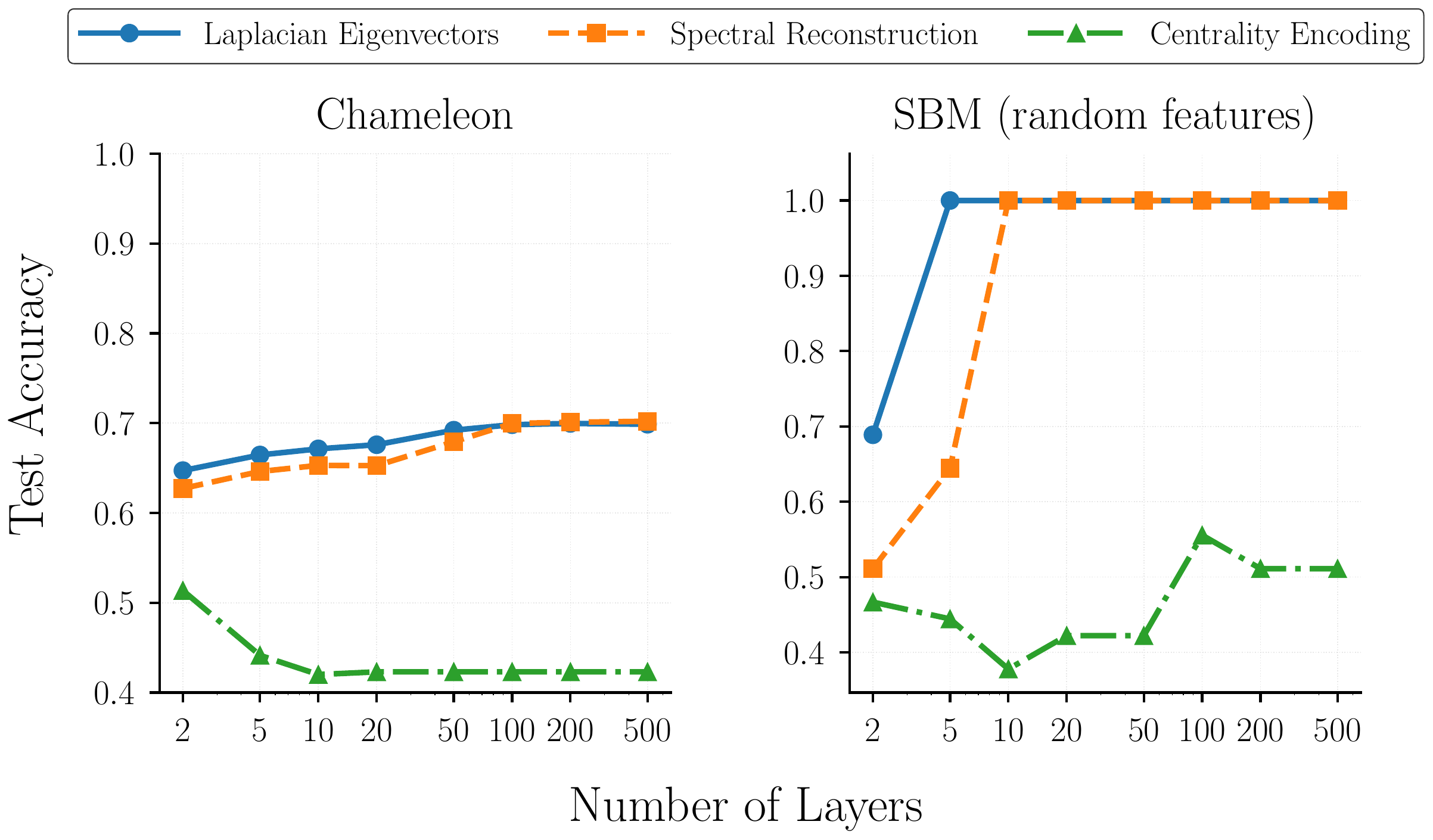}
    \caption{\textbf{Oversmoothing behaviour and the impact of positional encodings.} Test accuracy of \textbf{Graphormer-GP} on Chameleon (left) and SBM with random features (right) as a function of the number of layers. While accuracy in graph models typically suffers from performance degradation with an increasing number of layers, \textbf{Graphormer-GP exhibits increasing accuracy with depth} when utilizing informative positional encodings, such as Laplacian Eigenvectors (blue) and Spectral Reconstruction (orange). This empirical trend aligns with the theoretical results in Corollary~\ref{rem:Graphormer_convergence}.}
    \label{fig:pe_oversmoothing_graphormer}
\end{wrapfigure}

\section{Discussions and Experiments}
While graph transformers are notoriously difficult to analyse due to complex attention mechanisms and training dynamics, our Gaussian Process (GP) framework provides a mathematically rigorous closed-form equivalent. We provide the first kernel equivalent for graph transformers, thereby offering a tool for studying their structural properties. A structural phenomenon we investigate is \textbf{oversmoothing}. 
In Section~\ref{Ch:Oversmoothing}, we formalise this behaviour by analysing GP kernels under CSBM, revealing a fundamental distinction: while \textbf{GCN-GP is destined to oversmooth} (Corollary~\ref{cor:GCN_oversmoothing}), attention-based architectures can preserve discriminative community representations as depth increases. Consistent with this theory, evaluating GP kernels on synthetic and benchmark datasets (homophilic Pubmed \citep{sen2008collective} and heterophilic Chameleon \citep{rozemberczki2019gemsec}) shows that \textbf{GCN-GP performance deteriorates with depth}, whereas \textbf{GAT-GP remains resilient to oversmoothing} (Figure~\ref{fig:oversmoothing_comparison}, left; Corollary~\ref{cor:GAT_no_oversmoothing}). Details of the experiments are provided in Appendix~\ref{app:experimental details}, and the code to reproduce the experiments is available at \url{https://figshare.com/s/4ad5245d4d6c405f46f0}.

\begin{wrapfigure}[18]{r}{0.48\columnwidth}
    \centering
    \vspace{-280pt} 
    \includegraphics[width=\linewidth]{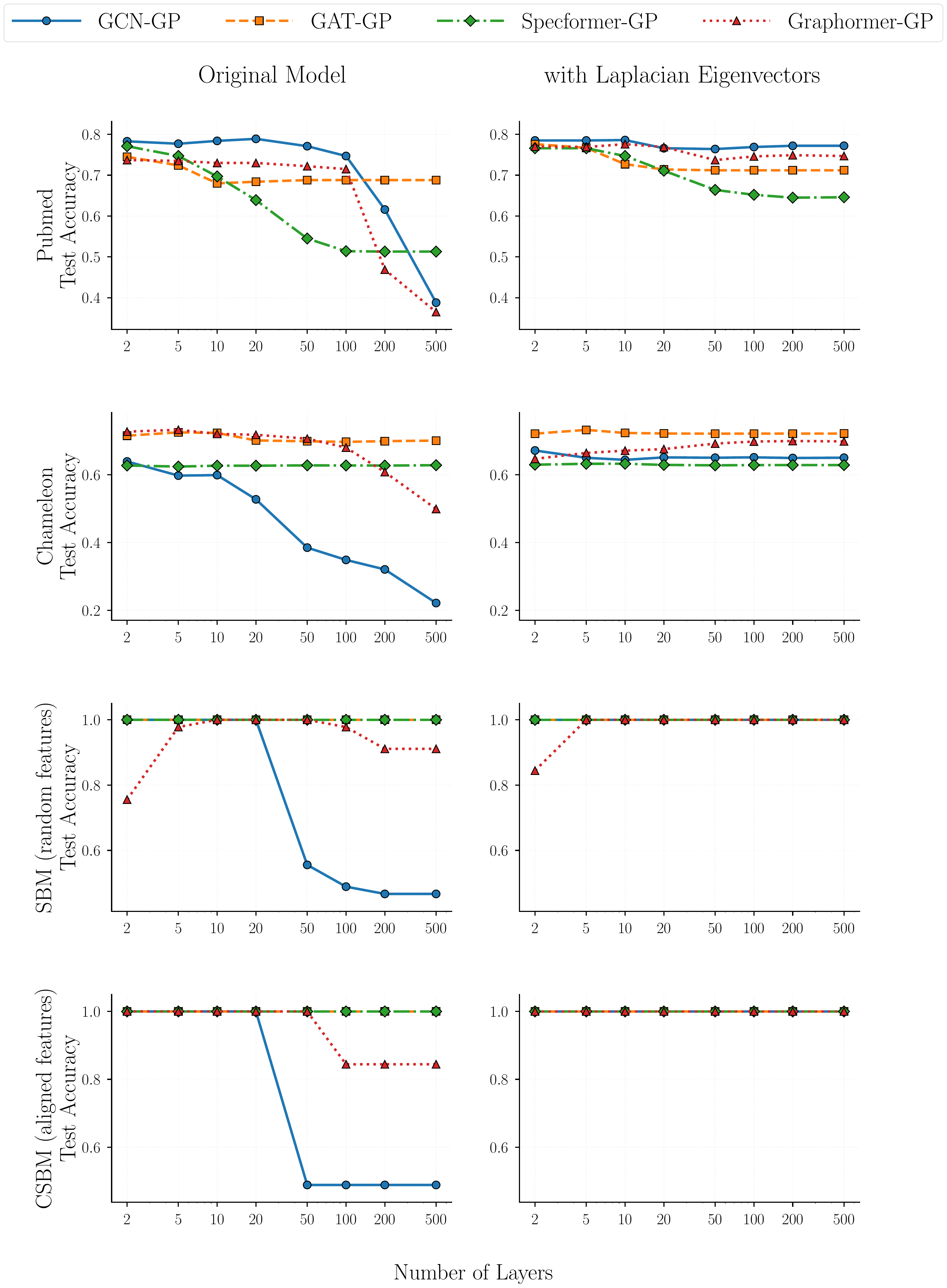}
    \caption{\textbf{Oversmoothing behaviour across various GNN-GP architectures.} Test accuracy is shown as a function of depth for the original models (left) and models augmented with Laplacian Eigenvectors (right). In the original configurations, \textbf{GCN-GP (blue) consistently suffers from a performance drop-off as the number of layers increases}. In contrast, \textbf{GAT-GP (orange) demonstrates resilience to depth}, maintaining stable performance as argued in Corollary~\ref{cor:GAT_no_oversmoothing}. The behavior of Specformer-GP and Graphormer-GP depends on the dataset. Notably, upon incorporating Laplacian Eigenvectors (right), the tendency to oversmooth is significantly mitigated across all architectures.}
    \label{fig:oversmoothing_comparison}
\end{wrapfigure}

\textbf{Graphormer-GP is sensitive to structural prior.}
Remark \ref{rem:Graphormer_convergence}, argues that when the structural prior is informative, Graphormer benefits from depth, as the kernel converges toward the spatial relation matrix (e.g., the ground-truth adjacency in SBM). However, when an informative structural prior is not explicitly known for a real-world dataset, Graphormer may still suffer from oversmoothing as the kernel collapses toward a \textbf{non-discriminative prior}. Since Graphormer can incorporate different structural priors, we ran additional experiments on SBM with random features and Chameleon. We evaluate three positional encodings (without bias) at large depth (Figure~\ref{fig:pe_oversmoothing_graphormer}): \textbf{Laplacian Eigenvectors}, \textbf{Spectral Reconstruction} (low-rank Laplacian approximation), and \textbf{Centrality Encoding} (from \cite{DBLP:conf/nips/YangLXLLASSX21}). Across both datasets, Laplacian-based encodings outperform centrality encoding and resist oversmoothing. Motivated by these results, we incorporate Laplacian Eigenvectors into the other GP architectures (Figure~\ref{fig:oversmoothing_comparison}, right) and find that they can mitigate the onset of oversmoothing.


\paragraph{Broader Perspective}
The GP equivalence for graph transformers established in this work provides a foundation for a broader theoretical analysis. Analysing the model in the infinite-width regime, independent of training dynamics, yields corresponding kernel equivalents that enable the study of structural properties. Among the central questions in such analyses is expressivity, as the model’s representational power can be characterised by the expressivity of the corresponding kernel. While our current analysis demonstrates community separation at large depth in the CSBM setting, it can be extended to study other critical structural phenomena, including \emph{oversquashing} (the bottleneck of information propagation) and the model's \emph{sensitivity to heterophily} under more complex random graph models. GP equivalence has already successfully enabled the study of generalisation \citep{Seeger2003PACBayesianGE} and inductive biases \citep{li2021metalearning} in deep learning. This work provides the first unified theoretical lens based on GP for graph transformers, laying the foundation for a deeper understanding of graph attention mechanisms.


\newpage
\bibliography{references}

@inproceedings{DBLP:conf/iclr/NiuA023,
  author       = {Zehao Niu and
                  Mihai Anitescu and
                  Jie Chen},
  title        = {Graph Neural Network-Inspired Kernels for Gaussian Processes in Semi-Supervised
                  Learning},
  booktitle    = {The Eleventh International Conference on Learning Representations,
                  {ICLR}},
  year         = {2023}
}

@inproceedings{DBLP:conf/icml/ZhangZM24,
  author       = {Bohang Zhang and
                  Lingxiao Zhao and
                  Haggai Maron},
  editor       = {Ruslan Salakhutdinov and
                  Zico Kolter and
                  Katherine A. Heller and
                  Adrian Weller and
                  Nuria Oliver and
                  Jonathan Scarlett and
                  Felix Berkenkamp},
  title        = {On the Expressive Power of Spectral Invariant Graph Neural Networks},
  booktitle    = {Forty-first International Conference on Machine Learning, {ICML} 2024,
                  Vienna, Austria, July 21-27, 2024},
  series       = {Proceedings of Machine Learning Research},
  volume       = {235},
  pages        = {60496--60526},
  publisher    = {{PMLR} / OpenReview.net},
  year         = {2024},
}

@article{DBLP:journals/tmlr/SabanayagamEG23,
  author       = {Mahalakshmi Sabanayagam and
                  Pascal Mattia Esser and
                  Debarghya Ghoshdastidar},
  title        = {Analysis of Convolutions, Non-linearity and Depth in Graph Neural
                  Networks using Neural Tangent Kernel},
  journal      = {Trans. Mach. Learn. Res.},
  volume       = {2023},
  year         = {2023},
 }

@inproceedings{DBLP:conf/nips/Keriven22,
  author       = {Nicolas Keriven},
  editor       = {Sanmi Koyejo and
                  S. Mohamed and
                  A. Agarwal and
                  Danielle Belgrave and
                  K. Cho and
                  A. Oh},
  title        = {Not too little, not too much: a theoretical analysis of graph (over)smoothing},
  booktitle    = {Advances in Neural Information Processing Systems 35: Annual Conference
                  on Neural Information Processing Systems 2022, NeurIPS 2022, New Orleans,
                  LA, USA, November 28 - December 9, 2022},
  year         = {2022},
}

@article{DBLP:journals/corr/AydaySG26,
  author       = {Nil Ayday and
                  Mahalakshmi Sabanayagam and
                  Debarghya Ghoshdastidar},
  title        = {Why does your graph neural network fail on some graphs? Insights from
                  exact generalisation error},
  journal      = {CoRR},
  volume       = {abs/2509.10337},
  year         = {2025},
}

@inproceedings{DBLP:conf/nips/BiettiM19,
  author       = {Alberto Bietti and
                  Julien Mairal},
  title        = {On the Inductive Bias of Neural Tangent Kernels},
  booktitle    = {Advances in Neural Information Processing Systems},
  year         = {2019},
}

@inproceedings{Matthews2018,
  author    = {Matthews, Alexander G. de G. and Hron, Jiri and Rowland, Mark and Turner, Richard E. and Ghahramani, Zoubin},
  title     = {Gaussian Process Behaviour in Wide Deep Neural Networks},
  booktitle = {International Conference on Learning Representations (ICLR)},
  year      = {2018}
}

@article{Blum1958,
  author  = {Blum, J. R. and Kiefer, J. and Rosenblatt, M.},
  title   = {Distribution Free Tests of Independence Based on the Sample Distribution Function},
  journal = {Annals of Mathematical Statistics},
  volume  = {29},
  pages   = {485--498},
  year    = {1958}
}

@book{Billingsley1986,
  author    = {Billingsley, Patrick},
  title     = {Probability and Measure},
  edition   = {2},
  publisher = {Wiley},
  year      = {1986}
}

@inproceedings{InfiniteAttention,
  author    = {Hron, Jiri and Bahri, Yasaman and Sohl-Dickstein, Jascha and Novak, Roman},
  title     = {Infinite Attention: {NNGP} and {NTK} for Deep Attention Networks},
  booktitle = {Proceedings of the 37th International Conference on Machine Learning (ICML)},
  series    = {Proceedings of Machine Learning Research (PMLR)},
  year      = {2020}
}

@inproceedings{Lee2018,
  author    = {Lee, Jaehoon and Bahri, Yasaman and Novak, Roman and Schoenholz, Samuel S. and Sohl-Dickstein, Jascha and Pennington, Jeffrey},
  title     = {Deep Neural Networks as Gaussian Processes},
  booktitle = {International Conference on Learning Representations (ICLR)},
  year      = {2018}
}

@inproceedings{Novak2019,
  author    = {Novak, Roman and Xiao, Lechao and Bahri, Yasaman and Lee, Jaehoon and Yang, Greg and Pennington, Jeffrey and Sohl-Dickstein, Jascha},
  title     = {Bayesian Deep Convolutional Networks with Many Channels are Gaussian Processes},
  booktitle = {International Conference on Learning Representations (ICLR)},
  year      = {2019}
}

@inproceedings{GarrigaAlonso2019,
  author    = {Garriga-Alonso, Adri{\`a} and Rasmussen, Carl Edward and Aitchison, Laurence},
  title     = {Deep Convolutional Networks as Shallow Gaussian Processes},
  booktitle = {International Conference on Learning Representations (ICLR)},
  year      = {2019}
}

@misc{Yang2019a,
  author       = {Yang, Greg},
  title        = {Scaling Limits of Wide Neural Networks with Weight Sharing: Gaussian Process Behavior, Gradient Independence, and Neural Tangent Kernel},
  howpublished = {arXiv preprint},
  year         = {2019},
  eprint       = {1902.04760},
  archivePrefix= {arXiv}
}

@article{sen2008collective,
  author  = {Sen, Prithviraj and Namata, Galileo and Bilgic, Mustafa and Getoor, Lise and Galligher, Brian and Eliassi-Rad, Tina},
  title   = {Collective Classification in Network Data},
  journal = {AI Magazine},
  volume  = {29},
  number  = {3},
  pages   = {93--106},
  year    = {2008}
}

@inproceedings{rozemberczki2019gemsec,
  author    = {Rozemberczki, Benedek and Davies, Ryan and Sarkar, Rik and Sutton, Charles},
  title     = {{GEMSEC}: Graph Embedding with Self Clustering},
  booktitle = {Proceedings of the IEEE/ACM International Conference on Advances in Social Networks Analysis and Mining (ASONAM)},
  pages     = {65--72},
  year      = {2019}
}

@inproceedings{DBLP:conf/nips/YangLXLLASSX21,
  author       = {Junhan Yang and
                  Zheng Liu and
                  Shitao Xiao and
                  Chaozhuo Li and
                  Defu Lian and
                  Sanjay Agrawal and
                  Amit Singh and
                  Guangzhong Sun and
                  Xing Xie},
  title        = {GraphFormers: GNN-nested Transformers for Representation Learning
                  on Textual Graph},
  booktitle    = {Advances in Neural Information Processing Systems 34: Annual Conference
                  on Neural Information Processing Systems },
  year         = {2021}
}

@inproceedings{DBLP:conf/faiml/ZhaoWHG23,
  author       = {Weichen Zhao and
                  Chenguang Wang and
                  Congying Han and
                  Tiande Guo},
  title        = {Exploring Over-smoothing in Graph Attention Networks from the Markov
                  Chain Perspective},
  booktitle    = {Proceedings of the International Conference on Frontiers of Artificial
                  Intelligence and Machine Learning, {FAIML}},
  year         = {2023}
}

@inproceedings{DBLP:conf/nips/WuAWJ23,
  author       = {Xinyi Wu and
                  Amir Ajorlou and
                  Zihui Wu and
                  Ali Jadbabaie},
  title        = {Demystifying Oversmoothing in Attention-Based Graph Neural Networks},
  booktitle    = {Advances in Neural Information Processing Systems 36: Annual Conference
                  on Neural Information Processing Systems 2023},
  year         = {2023}
}

@inproceedings{DBLP:conf/iclr/VelickovicCCRLB18,
  author       = {Petar Velickovic and
                  Guillem Cucurull and
                  Arantxa Casanova and
                  Adriana Romero and
                  Pietro Li{\`{o}} and
                  Yoshua Bengio},
  title        = {Graph Attention Networks},
  booktitle    = {6th International Conference on Learning Representations},
  year         = {2018}
}

@inproceedings{DBLP:conf/iclr/BoSWL23,
  author       = {Deyu Bo and
                  Chuan Shi and
                  Lele Wang and
                  Renjie Liao},
  title        = {Specformer: Spectral Graph Neural Networks Meet Transformers},
  booktitle    = {The Eleventh International Conference on Learning Representations,
                  {ICLR}},
  year         = {2023}
}

@inproceedings{DBLP:conf/aaai/LiHW18,
  author       = {Qimai Li and
                  Zhichao Han and
                  Xiao{-}Ming Wu},
  title        = {Deeper Insights Into Graph Convolutional Networks for Semi-Supervised
                  Learning},
  booktitle    = {Proceedings of the Thirty-Second {AAAI} Conference on Artificial Intelligence,
                  (AAAI-18)},
  year         = {2018}
}

@inproceedings{DBLP:conf/iclr/WuCWJ23,
  author       = {Xinyi Wu and
                  Zhengdao Chen and
                  William Wei Wang and
                  Ali Jadbabaie},
  title        = {A Non-Asymptotic Analysis of Oversmoothing in Graph Neural Networks},
  booktitle    = {The Eleventh International Conference on Learning Representations},
  year         = {2023}
}

@inproceedings{
lee2018deep,
title={Deep Neural Networks as Gaussian Processes},
author={Jaehoon Lee and Jascha Sohl-dickstein and Jeffrey Pennington and Roman Novak and Sam Schoenholz and Yasaman Bahri},
booktitle={International Conference on Learning Representations},
year={2018},
}

@inproceedings{DBLP:conf/iclr/0002Y21,
  author       = {Uri Alon and
                  Eran Yahav},
  title        = {On the Bottleneck of Graph Neural Networks and its Practical Implications},
  booktitle    = {9th International Conference on Learning Representations, 2021},
  year         = {2021}, 
}

@article{DBLP:journals/corr/abs-1905-09550,
  author       = {Hoang NT and
                  Takanori Maehara},
  title        = {Revisiting Graph Neural Networks: All We Have is Low-Pass Filters},
  journal      = {CoRR},
  year         = {2019}
}

@article{wang2019han,
  title   = {Heterogeneous Graph Attention Network},
  author  = {Wang, Xiao and Ji, Houye and Shi, Chuan and Wang, Bai and Cui, Peng and Yu, Philip S. and Ye, Yanfang},
  journal = {arXiv preprint arXiv:1903.07293},
  year    = {2019}
}

@inproceedings{
li2021metalearning,
title={Meta-learning inductive biases of learning systems with Gaussian processes},
author={Michael Y. Li and Erin Grant and Thomas L. Griffiths},
booktitle={Fifth Workshop on Meta-Learning at the Conference on Neural Information Processing Systems},
year={2021},
}

@article{Seeger2003PACBayesianGE,
  title={PAC-Bayesian Generalisation Error Bounds for Gaussian Process Classification},
  author={Matthias W. Seeger},
  journal={J. Mach. Learn. Res.},
  year={2003},
  pages={233-269},
}


\newpage
\onecolumn
\title{Gaussian Process Limit Reveals Structural Benefits of  Graph Transformers\\(Supplementary Material)}
\maketitle
\appendix
\section{Design Choices}

\label{app:design}

\subsection{GNN with Multi-Head  Aggregation}
\label{app:gnn-design}

A common finite-head formulation aggregates $H=d^{\ell,H}$ head outputs by concatenation:
\begin{equation}
\label{eq:finite_head_gnn_concat}
f^{\ell}(X)
=
\Big\|_{h=1}^{d^{\ell,H}}
\sigma\!\left(
\mathbf{S}^{(\ell,h)}_{\mathrm{GNN}}
f^{\ell-1}(\mathbf{X})\,
\mathbf{W}^{\ell,h}
\right),
\end{equation}
where $\Vert$ denotes concatenation. This formulation is not suitable for taking the
limit $d^{\ell,H}\to\infty$, since concatenation leads to an ever-growing output dimension
and prevents a CLT-based infinite-head analysis.

To obtain a well-defined infinite-head limit in a fixed feature dimension, we modify the
aggregation mechanism by introducing a shared projection that aggregates all head-wise
outputs into the next-layer representation:
\begin{equation}
\label{eq:infinite_head_gnn}
f^{\ell}(\mathbf{X})
=
\sigma\!\left(
\Big(
\big\|_{h=1}^{d^{\ell,H}}
\mathbf{S}^{(\ell,h)}_{\mathrm{GNN}}
f^{\ell-1}(\mathbf{X})\,
\mathbf{W}^{\ell,h}
\Big)\,
\mathbf{W}^{\ell,H}
\right),
\end{equation}
where $\mathbf{W}^{\ell,h}$ are head-specific weight matrices and $\mathbf{W}^{\ell,H}$ is a shared
projection mapping the concatenated head outputs to the fixed-dimensional feature space of
layer $\ell$.

This modification plays a crucial role in the infinite-head analysis: it replaces
dimension growth by a fixed-dimensional linear aggregation of head-wise random features,
which enables a CLT-type argument despite the $n\times n$ stochasticity of
$\mathbf{S}^{(\ell,h)}_{\mathrm{GNN}}(\mathbf{X})$ and the induced dependence across coordinates.
As a result, the infinite-head GNN in~\eqref{eq:infinite_head_gnn} admits a principled
Gaussian-process limit while retaining the expressive stochastic message-passing structure
encoded by $\mathbf{S}^{(\ell,h)}_{\mathrm{GNN}}$.

\subsection{Graphormer}
\label{app:graphormer-design}
In the original Graphormer architecture, structural information is injected through a \emph{centrality encoding}, where each node representation is augmented by learnable embeddings indexed by its in-degree and out-degree.
Using our notation, the input node representation can be written as
\begin{equation}
\label{eq:graphormer_centrality_f}
f_i^{(0)}(X)
=
x_i
+
z_{\deg^-}(v_i)
+
z_{\deg^+}(v_i),
\end{equation}
where $x_i$ denotes the raw node feature and $z_{\deg^-}, z_{\deg^+} \in \mathbb{R}^d$ are learnable degree-based embeddings.
For undirected graphs, the two degree terms are unified into a single degree encoding.

\vspace{0.8em}
While such a design effectively injects node importance signals into the attention mechanism and has demonstrated strong empirical performance on graph-level tasks, it represents a specific instance of structural augmentation based solely on degree statistics.
This degree-centric formulation becomes restrictive when considering node-level prediction tasks or when aiming to incorporate richer and more general notions of graph structure.

To generalize this mechanism, we reinterpret centrality encoding as a special case of feature-level graph encoding insertion.
Instead of restricting the structural signal to degree-based embeddings, we introduce a general graph encoding term $P_i^{\ell}$ at layer $\ell$, and define the augmented node representation as
\begin{equation}
\label{eq:general_graph_encoding_concat}
\tilde f_i^{\ell}(X)
:=
\bigl[
f_i^{\ell}(X),\;
P_i^{\ell}
\bigr],
\qquad i \in V,
\end{equation}
where $f_i^{\ell}(X)$ denotes the node representation produced by the backbone network at layer $\ell$, and $P_i^{\ell} \in \mathbb{R}^{d_p}$ encodes structural or positional information associated with node $v_i$.
More importantly,~\eqref{eq:general_graph_encoding_concat} allows $P_i^{\ell}$ to represent arbitrary graph encodings, including but not limited to Laplacian positional encodings, random walk-based encodings, or other task-specific structural descriptors.

By explicitly separating semantic features $f_i^{\ell}(X)$ from structural encodings $P_i^{\ell}$, the proposed formulation provides a flexible and general mechanism for injecting graph structure into the model.
This is particularly advantageous for node-level tasks, where the prediction target depends on fine-grained local and mesoscopic structural patterns rather than global graph summaries.
Furthermore, the concatenation-based design preserves the modularity of the architecture and facilitates theoretical analysis, as different choices of $P_i^{\ell}$ can be studied independently of the backbone representation.

\vspace{0.5em}
In summary, we treat centrality encoding as a special case of a more general graph encoding framework.
The unified representation~\eqref{eq:general_graph_encoding_concat} enables the incorporation of diverse structural signals in a principled manner, while maintaining compatibility with Transformer-based architectures and node-level prediction settings.

\subsection{Specformer}
\label{app:specformer-design}
\paragraph{Spectral Graph Construction in Specformer.}
In the original Specformer model, the graph structure is learned in the
spectral domain through a set of spectral tokens.
Let $H^t \in \mathbb{R}^{n \times d^{\ell,H}}$ denote the spectral
tokens at layer t in eigenvalue encoding, where each column corresponds to the output of
one attention head.
At layer t, each attention head
$h = 1, \dots, d^{\ell,H}$ produces a separate spectral representation
via
\begin{equation}
Z_h
=
\mathrm{Attention}
\big(
H^t W^{Q,t h},
\;
H^t W^{K,t h},
\;
H^t W^{V,t h}
\big),
\end{equation}
where $W^{Q,t h}$, $W^{K,t h}$, and $W^{V,t h}$ are learnable
projection matrices.
The head-specific representation $Z_h$ is then mapped to a vector of
filtered eigenvalues,
\begin{equation}
\lambda_h
=
\sigma\!\left( Z_h W_\lambda \right)
\in \mathbb{R}^{N},
\end{equation}
where $\sigma(\cdot)$ denotes an activation function.
Each eigenvalue vector $\lambda_h$ defines a spectral graph operator
\begin{equation}
S_h
=
U \, \mathrm{diag}(\lambda_h) \, U^\top ,
\end{equation}
with $U$ denoting the eigenvector matrix of the graph Laplacian.

As a result, the $d^{\ell,H}$ attention heads produce
$d^{\ell,H}$ distinct spectral graph operators
$\{ S_h \}_{h=1}^{d^{t,H}}$.
These operators are subsequently used as intermediate bases to
construct feature-wise graph structures.
Specifically, the operators are concatenated along the channel
dimension and mapped to the node feature dimension $d$ through a
feed-forward network (FFN),
\begin{equation}
\hat S
=
\mathrm{FFN}
\big(
[\, I_N \,\|\, S_1 \,\|\, \cdots \,\|\, S_{d^{t,H}} \,]
\big)
\;\in\; \mathbb{R}^{N \times N \times d}.
\end{equation}
The output $\hat S$ thus consists of $d$ graph operators, where each
slice $\hat S_{:,:,i}$ corresponds to a distinct adjacency (or Laplacian)
matrix associated with the $i$-th node feature dimension.

Based on these feature-wise graph operators, graph convolution is
performed independently for each feature channel.
Let $f(X)^{(\ell-1)} \in \mathbb{R}^{N \times d}$ denote the node
representations at layer $\ell-1$.
For the $i$-th feature dimension, the propagation is given by
\begin{equation}
\hat f(X)^{(\ell-1)}_{:,i}
=
\hat S_{:,:,i} \, f(X)^{(\ell-1)}_{:,i},
\end{equation}
and the node representations are updated as
\begin{equation}
f(X)^{(\ell)}
=
\sigma\!\left(
\hat f(X)^{(\ell-1)} W^{(\ell-1)}
\right) .
\end{equation}
While this construction is expressive, it introduces a fundamental
statistical issue.
Since the $d$ graph operators
$\{ \hat S_{:,:,i} \}_{i=1}^d$
are jointly generated through a shared FFN, they are statistically
dependent.
Consequently, the feature dimensions of $\hat f^{(\ell)}(X)$ are no longer
independent, violating the independence assumptions required by the
Central Limit Theorem (CLT).
As a result, the original Specformer formulation does not satisfy the
conditions necessary for CLT-based infinite-width analysis.

\paragraph{Multi-head Aggregated Spectral Construction.}
To address this issue, for each head of graph convolution layer heads $d^{\ell,H}$, we modify the spectral construction by employing
multi-head attention to generate a single set of spectral coefficients,
rather than one per feature dimension.
Specifically, at each layer $t$, the outputs of the
$d^{t,H}$ spectral heads are aggregated prior to spectral filtering.
The spectral tokens are updated as
\begin{equation}
H^{t+1}
=
\Big(
\big\|_{h=1}^{d^{t,H}}
S_{\mathrm{Spec}}^{(t,h)}\!\left(H^t\right)
\, H^t W^{V,t h}
\Big)
\, W^{H,t},
\end{equation}
where $W^{H,t}$ is a learnable projection matrix.

After $T$ eigenvalue encoding layers, we obtain the final spectral tokens $\bar H = H^{T}$.
Based on $\bar H$, a single set of spectral coefficients is generated as
\begin{equation}
\bar\lambda
=
\sigma\!\left( \bar H W_{\lambda} \right)
\in \mathbb{R}^{n},
\end{equation}
which defines the final graph convolution operator
\begin{equation}
S_{\mathrm{Specformer}}
=
U \, \mathrm{diag}(\bar\lambda) \, U^\top
\in \mathbb{R}^{n \times n}.
\end{equation}
By constructing a shared spectral operator across feature dimensions,
this formulation preserves feature-wise independence and restores the
conditions required for CLT-based infinite-width analysis.

\section{Formal Theorems on Gaussian Process Limits of Graph Transformers}
\label{App:GP_Theorems}
In this Appendix, we present the formal theorems and kernel recursions establishing the Gaussian process limits for graph transformers, derived under the independent Gaussian weight priors summarized in Table~\ref{tab:weight_priors}.
\begin{table*}
\centering
\caption{Summary of the Gaussian weight priors used for each graph transformer in the GP analysis. 
These priors define the variance of the initial weights for all linear transformations and attention parameters.}
\label{tab:weight_priors}
\begin{small}
\begin{tabular}{@{}ll@{}}
\toprule
\textbf{Model} & \textbf{Weight Priors (Variance)} \\ \midrule
GAT & \(
W^{\ell,H}_{ij}
\sim \mathcal{N}\!\left(0,\frac{\sigma_H^2}{d^{\ell,H}d_{\ell-1}}\right),
W^{\ell,h}_{ij}
\sim \mathcal{N}\!\left(0,\frac{\sigma_w^2}{d_{\ell-1}}\right),
\vec{v}^{\ell h}_k
\sim \mathcal{N}\!\left(0,\frac{\sigma_v^2}{d_{\ell-1}}\right)
\) \\
Graphormer & \(W^{Q,\ell h}_{ij}
\sim \mathcal{N}\!\left(0,\frac{\sigma_Q^2}{d_{\ell-1}}\right),
W^{K,\ell h}_{ij}
\sim \mathcal{N}\!\left(0,\frac{\sigma_K^2}{d_{\ell-1}}\right),
W^{\ell,H}_{ij}
\sim \mathcal{N}\!\left(0,\frac{\sigma_H^2}{d^{\ell,H}d_{\ell-1}}\right),
W^{V,\ell h}_{ij}
\sim \mathcal{N}\!\left(0,\frac{\sigma_w^2}{d_{\ell-1}}\right),
b_{\rho}(i, j)
\sim \mathcal{N}(0,\sigma_b^2)\) \\
Specformer & \(
W^{Q,t k h}_{ij}\sim \mathcal{N}(0,\sigma_Q^2/d_{t-1}),\quad
W^{K,t k h}_{ij}\sim \mathcal{N}(0,\sigma_K^2/d_{t-1}),\quad
W^{V,t k h}_{ij}\sim \mathcal{N}(0,\sigma_V^2/d_{t-1}),
W^{th,H}_{ij}\sim \mathcal{N}\!\Big(0,\frac{\sigma_O^2}{d^{t,H}d_{t-1}}\Big),
\) \\
\bottomrule
\end{tabular}
\end{small}
\end{table*}

\textbf{Graph Attention Network (GAT)}
In GAT, attention is computed locally over graph neighborhoods, and in the infinite-width and infinite-head limit, both the attention logits and node pre-activations converge to Gaussian processes. The resulting kernel recursion reveals how features and neighborhood structure jointly shape representation propagation.
\begin{theorem}
\label{thm:GAT_GP}
(Infinite-width and infinite-head limit of a GAT layer) Assume that the parameters of the $\ell$-th GAT layer satisfy
\(
W^{\ell,H}_{ij}
\sim \mathcal{N}\!\left(0,\frac{\sigma_H^2}{d^{\ell,H}d_{\ell-1}}\right),
W^{\ell,h}_{ij}
\sim \mathcal{N}\!\left(0,\frac{\sigma_w^2}{d_{\ell-1}}\right),
\vec{v}^{\ell h}_k
\sim \mathcal{N}\!\left(0,\frac{\sigma_v^2}{d_{\ell-1}}\right),
\)
independently for all $i,j,k$ and heads $h$,
then, as $d_{\ell},d^{\ell,H}\to\infty$, the following holds:

\textup{(I)} $E^{\ell }(X)=\big\{E^{(\ell h)}(X) :  h\in\mathbb{N}\big\}$
converges in distribution to $E^{\ell }(X) \sim \mathcal{GP}(0, K^{(\ell),E})$ with
\begin{equation}
\small
\begin{aligned}
K^{(\ell h),E}_{ai,bj}
&= \mathbb{E}\big[E^{(\ell h)}_{ai}(X)\,E^{(\ell h')}_{bj}(X)\big]\\
&= \delta_{h=h'}\, \sigma_w^2\sigma_v^2
   \bigl(K^{(\ell-1)}_{ab}+K^{(\ell-1)}_{ij}\bigr),
\end{aligned}
\label{eq:head-cov}
\end{equation}

\textup{(II)} $g^{\ell }(X)$
converges in distribution to $g^{\ell }(X) \sim \mathcal{GP}(0, \Sigma^{(\ell)})$ with
{\small
\begin{equation}
\begin{aligned}
C^{(\ell)}_{ai,bj}
&:=\mathbb{E}\Big[
  \bigl(S_{\mathrm{GAT}}^{(\ell,1)}(X)\bigr)_{ai}\,
  \bigl(S_{\mathrm{GAT}}^{(\ell,1)}(X)\bigr)_{bj}
\Big]\\
&= \mathbb{E}\Big[
  \phi\bigl(\sigma(E^{(\ell 1)}_{ai}(X))\bigr)\,
  \phi\bigl(\sigma(E^{(\ell 1)}_{bj}(X))\bigr)
\Big],
\end{aligned}
\label{eq:attention-kernel}
\end{equation}
}
{\small
\begin{equation}
\begin{aligned}
\Sigma_{ab}^{(\ell)}
&:= \mathbb{E}\big[g^{\ell }_{a1}(X)\,g^{\ell }_{b1}(X)\big]\\
&= \sigma_H^2\sigma_w^2
\sum_{i\in N_a}\sum_{j\in N_b}
K^{(\ell-1)}_{ij}\,
C^{(\ell)}_{ai,bj},
\end{aligned}
\label{eq:one-layer-node-pre}
\end{equation}
}
where $a,b\in V$ and $(a,i),(b,j)\in E$.
\end{theorem}
Theorem \ref{thm:GAT_GP} shows that, for each attention head, the collection of attention logits over edges converges in distribution to a centered Gaussian process, with different heads becoming independent in the limit. Equation~\eqref{eq:head-cov} characterizes the covariance between attention logits on two edges $(a,i)$ and $(b,j)$, which is driven by the similarity between the two source nodes $a,b$ together with the similarity between the two target nodes $i,j$, so edges whose endpoints are simultaneously close in representation space exhibit stronger correlation. Using Equation~\eqref{eq:head-cov}, the theorem then derives the Gaussian process limit of the pre-activations. Equation~\eqref{eq:one-layer-node-pre} describes how node-level dependence propagates to the next layer, where the covariance between the pre-activations of nodes $a$ and $b$ at layer $\ell$ is obtained by aggregating contributions from all neighbor pairs $(i,j)$ with $i\in\mathcal N_a$ and $j\in\mathcal N_b$, combining similarity among neighboring node representations from the previous layer with the alignment induced by attention on the associated edges. Together, these results yield a kernel recursion that describes how feature and graph govern representation propagation across layers. Theorem \ref{thm:GAT_GP} establishes a general kernel recursion, and to obtain a closed-form recursion, we next consider specific choices of the attention and activation nonlinearities.

\begin{corollary}
\label{cor:GAT_kernel_closed_form}
(Closed-form GAT kernel recursion) Under the assumptions of Theorem~\ref{thm:GAT_GP}, for $\phi(x)=x$ and ReLU activation $\sigma(\cdot)$, the node-level covariance kernel at layer $\ell$
admits the closed-form expression
\begin{equation}
\small
\begin{aligned}
\Sigma^{(\ell)}_{ab}
=&\;
\frac{\sigma_H^2\sigma_w^4\sigma_v^2}{2\pi}
\sum_{i\in N_a}\sum_{j\in N_b}
K^{(\ell-1)}_{ij} \\[2pt]
&\;\times
\sqrt{
\big(K^{(\ell-1)}_{aa}+K^{(\ell-1)}_{ii}\big)
\big(K^{(\ell-1)}_{bb}+K^{(\ell-1)}_{jj}\big)
} \\[2pt]
&\;\times
\Big(
\sin\theta^{(\ell-1)}_{ai,bj}
+
\big(\pi-\theta^{(\ell-1)}_{ai,bj}\big)
\cos\theta^{(\ell-1)}_{ai,bj}
\Big).
\end{aligned}
\end{equation}

\begin{equation}
\small
\theta^{(\ell-1)}_{ai,bj}
=
\arccos
\!\left(
\frac{
K^{(\ell-1)}_{ab}+K^{(\ell-1)}_{ij}
}{
\sqrt{
\big(K^{(\ell-1)}_{aa}+K^{(\ell-1)}_{ii}\big)
\big(K^{(\ell-1)}_{bb}+K^{(\ell-1)}_{jj}\big)
}
}
\right).
\end{equation}
\end{corollary}

Kernel in Corollary~\ref{cor:GAT_kernel_closed_form}
induces a quadratic computational cost over neighborhood pairs. To further
simplify the recursion, we consider a linear activation.
\begin{corollary}[Linear GAT kernel]
\label{cor:GAT_kernel_linearized}
Under the assumptions of Corollary~\ref{cor:GAT_kernel_closed_form}, and for
$\sigma(x)=x$, the node-level kernel admits the following simplified
closed-form recursion:
\begin{align*}
\label{eq:GAT_kernel_linearized}
\Sigma^{(\ell)} 
&= \sigma_H^2\,\sigma_w^4\,\sigma_v^2 \Big(
K^{(\ell-1)} \odot \big(A K^{(\ell-1)} A^\top\big) \notag\\
&\quad + A\big(K^{(\ell-1)} \odot K^{(\ell-1)}\big) A^\top
\Big).
\end{align*}
\end{corollary}

\textbf{Graphormer} As introduced in Section \ref{sec:preliminaries}, Graphormer incorporates graph information
through a positional encoding mechanism.
In the infinite-width limit, the output  converges in
distribution to a Gaussian process whose kernel satisfies
\begin{equation}
\tilde K^{(\ell-1)}\longmapsto\;
\alpha\, K^{(\ell-1)} + (1 - \alpha)\, R^{(\ell)} \quad,
\alpha \in (0,1).
\label{eq:graphormer-tildeK}
\end{equation}
Where the matrix $R$ can be interpreted as the covariance matrix
of the corresponding positional encoding which can be computed as an inner product of positional
embeddings, i.e.,
\(
R^{(\ell)}_{ab} = \langle P_a^{(\ell)}, P_b^{(\ell)} \rangle
\).
Alternatively, instead of introducing explicit positional embeddings,
the graph structure itself can be directly leveraged to construct a
node-level graph-structure covariance matrix. 
Theorem~\ref{thm:Graphormer_GP} characterizes the Gaussian process limit of a Graphormer layer by establishing joint limits for both the attention logits and the node pre-activations. Although the proof strategy is analogous to GAT, the incorporation of positional encodings and how the attention scores are computed modifies the kernel.

Centrality Encoding (CE) differs from other positional encoding that it constructs a learnable embedding based on node degrees. 
Specifically, a CE parameter matrix has size $\mathbf{P}_{\mathrm{CE}} \in \mathbb{R}^{(\texttt{max\_degree}+1)\times d}$, 
and each entry is initialized as
$
(\mathbf{P}_{\mathrm{CE}})_{ij} \sim \mathcal{N}\!\left(0,\frac{\sigma^2_{\mathrm{CE}}}{d}\right).
$
Since we only consider undirected graphs in this paper, nodes with the same degree share the same embedding vector, i.e.,
$
P_v = \mathbf{P}_{\mathrm{CE}}[\deg(v)] \in \mathbb{R}^d .
$
Under this construction, the embedding correlation satisfies
\[
\mathbb{E}\!\left[P_a^\top P_b\right]
=
\begin{cases}
\sigma^2_{\mathrm{CE}}, & \deg(a)=\deg(b),\\
0, & \deg(a)\neq \deg(b).
\end{cases}
\]
Therefore, we obtain
\[
R_{ab}
=
\begin{cases}
\sigma^2_{\mathrm{CE}}, & \deg(a)=\deg(b),\\
0, & \deg(a)\neq \deg(b).
\end{cases}
\]

\begin{theorem}(Infinite-width and infinite-head limit of a Graphormer layer)
\label{thm:Graphormer_GP}
Assume the parameters of the $\ell$-th Graphormer layer satisfy
$
W^{Q,\ell h}_{ij}
\sim \mathcal{N}\!\left(0,\frac{\sigma_Q^2}{d_{\ell-1}}\right),
W^{K,\ell h}_{ij}
\sim \mathcal{N}\!\left(0,\frac{\sigma_K^2}{d_{\ell-1}}\right),
W^{\ell,H}_{ij}
\sim \mathcal{N}\!\left(0,\frac{\sigma_H^2}{d^{\ell,H}d_{\ell-1}}\right),
W^{V,\ell h}_{ij}
\sim \mathcal{N}\!\left(0,\frac{\sigma_w^2}{d_{\ell-1}}\right),
b_{\phi}(i, j)
\sim \mathcal{N}(0,\sigma_b^2),
$
independently for all $i,j$, heads $h$, and structural indices $\phi(i,j)$,
then, as $d_{\ell},d^{\ell,H}\to\infty$, the following holds:

\textup{(I)} $E^{\ell }(X)=\big\{E^{(\ell h)}(X) :  h\in\mathbb{N}\big\}$
converges in distribution to $E^{\ell-1 }(X) \sim \mathcal{GP}(0, K^{(\ell),E})$ with
\begin{equation}
\small
\begin{aligned}
K^{(\ell h),E}_{ai,bj}
&= \mathbb{E}\big[E^{(\ell h)}_{ai}(X)\,E^{(\ell h')}_{bj}(X)\big] \\
&= \delta_{h=h'}\Big(
\sigma_Q^2\sigma_K^2\,
\tilde K^{(\ell-1)}_{ab}\,
\tilde K^{(\ell-1)}_{ij}
+
\sigma_b^2\,
\mathbf 1_{\phi(a,i)=\phi(b,j)}
\Big).
\end{aligned}
\label{eq:graphormer-head-cov}
\end{equation}

\textup{(II)} $g^{\ell }(X)$
converges in distribution to $g^{\ell }(X) \sim \mathcal{GP}(0, \Sigma^{(\ell)})$ with
{\small
\begin{equation}
\begin{aligned}
C^{(\ell)}_{ai,bj}
&:=\mathbb{E}\Big[
  \bigl(S_{\mathrm{Graphormer}}^{(\ell,1)}(X)\bigr)_{ai}\,
  \bigl(S_{\mathrm{Graphormer}}^{(\ell,1)}(X)\bigr)_{bj}
\Big]\\
&= \mathbb{E}\Big[
  \phi\bigl(E^{(\ell 1)}_{ai}(X)\bigr)\,
  \phi\bigl(E^{(\ell 1)}_{bj}(X)\bigr)
\Big],
\end{aligned}
\label{eq:graphormer-attention-kernel}
\end{equation}
}
{\small
\begin{equation}
\begin{aligned}
\Sigma^{(\ell)}_{ab}
&:=\mathbb{E}\big[g^{\ell }_{a1}(X)\,g^{\ell }_{b1}(X)\big]
= \sigma_H^2\sigma_w^2
\sum_{i\in V}\sum_{j\in V}
\tilde K^{(\ell-1)}_{ij}\,
C^{(\ell)}_{ai,bj},
\end{aligned}
\label{eq:graphormer-one-layer-node-pre}
\end{equation}
}
where $a,b,i,j\in V$.
\end{theorem}
Equation~\eqref{eq:graphormer-tildeK} shows that the effective similarity between nodes depends on both the similarity of node features and the similarity of graph structural features encoded by the positional information. Equation~\eqref{eq:graphormer-head-cov} further indicates that attention correlations are strengthened for node pairs sharing the same structural relation, ensuring that graph information is explicitly leveraged in guiding attention. Equation~\eqref{eq:graphormer-one-layer-node-pre} then propagates the attention  correlations to the next layer through an aggregation over all node pairs, reflecting that, unlike GAT which is restricted to local neighborhoods, Graphormer induces similarity connections between every pair of nodes in the graph. The following corollary specializes Theorem~\ref{thm:Graphormer_GP} to linear attention.
\begin{corollary}[Closed-form Graphormer kernel under linear attention]
\label{cor:graphormer-linear}
 Under the assumptions of Theorem~\ref{thm:Graphormer_GP}, 
and for $\phi(x)=x$,
Consequently, the node-level covariance kernel at layer $\ell$
admits the closed-form expression
\begin{equation*}
\begin{aligned}
&\Sigma^{(\ell)}_{ab}
=
\sigma_H^2\sigma_w^2
\Bigg[
\sigma_Q^2\sigma_K^2\,
\tilde K^{(\ell-1)}_{ab}
\sum_{i,j\in V}
\tilde K^{(\ell-1)}_{ij}\,
\tilde K^{(\ell-1)}_{ij}+
\\
&\quad
\sigma_b^2
\sum_{i,j\in V}
\tilde K^{(\ell-1)}_{ij}\,
\mathbf 1_{\phi(a,i)=\phi(b,j)}
\Bigg], \text{where } a,b,i,j\in V.
\end{aligned}
\end{equation*}
\end{corollary}

\textbf{Specformer.}
As detailed in Section \ref{sec:preliminaries}, the convolution in Specformer is composed of \emph{spectral tokens} and \emph{node features}. 
The spectral tokens are first processed through $T$ layers of encoding, and are subsequently passed through a decoder to construct a data-dependent spectral filtering that governs the propagation of node features. The following theorem characterizes the behavior of Specformer in the infinite-width and infinite-head limit. In this regime, (I) the spectral-token representations $H^t$, which are updated through attention, and (II) the node-feature pre-activations $g(X)$ converge in distribution to a Gaussian process.


\begin{theorem}
\label{thm:Specformer_GP}
(Infinite-width and infinite-head limit of Specformer) If the parameters of the $t$-st token layer satisfy $
W^{Q,t k h}_{ij}\sim \mathcal{N}(0,\sigma_Q^2/d_{t-1}),\quad
W^{K,t k h}_{ij}\sim \mathcal{N}(0,\sigma_K^2/d_{t-1}),\quad
W^{V,t k h}_{ij}\sim \mathcal{N}(0,\sigma_V^2/d_{t-1})
$
and $
W^{th,H}_{ij}\sim \mathcal{N}\!\Big(0,\frac{\sigma_O^2}{d^{t,H}d_{t-1}}\Big),
$
independently for all indices, token heads $k$, and spectral heads $h$,
then, as $d_t,d^{t,H}\to\infty$, the following holds.

(I)
$H^{t}
$
converges in distribution to
$
H^{t}\sim\mathcal{GP}(0,K^{(t)}_H)
$
with
\begin{equation}
\small
\begin{aligned}
K^{(t,k,h),E}_{ik,jl}
&=
\mathbb{E}\!\left[
E^{(t,k,h)}_{ij}(H)\,
E^{(t,k',h')}_{kl}(H')
\right]\\
&=
\delta_{k=k'}\,\delta_{h=h'}\,
\sigma_Q^2\sigma_K^2\,
K^{(t-1)}_{H,ik}\,
K^{(t-1)}_{H,jl},
\end{aligned}
\label{eq:specformer_head_cov_H_t}
\end{equation}
{\small
\begin{equation}
\begin{aligned}
C^{(t,h)}_{ik,jl}(H,H')
&:=
\delta_{h=h'}\,\mathbb{E}\Big[
\bigl(S_{\mathrm{Spec}}^{(t,1,h)}(H)\bigr)_{ik}\,
\bigl(S_{\mathrm{Spec}}^{(t,1,h')}(H')\bigr)_{jl}
\Big]\\
&=
\delta_{h=h'}\,\mathbb{E}\Big[
\phi\!\big(E^{(t,1,h)}_{ik}(H)\big)\,
\phi\!\big(E^{(t,1,hi)}_{jl}(H')\big)
\Big],
\end{aligned}
\label{eq:specformer_attention_kernel_H_t}
\end{equation}
}
{\small
\begin{equation}
\begin{aligned}
K^{(t,h)}_{H,ij}
&:=
\mathbb{E}\big[(H^{t,h})_i(H^{t,h'})_j\big]\\
&=
\delta_{h=h'}\,\big(\sigma_O^2\sigma_V^2
\sum_{k,l\in V}
K^{(t-1)}_{H,kl}\,
C^{(t,h)}_{ik,jl}(H,H')\big),
\end{aligned}
\label{eq:specformer_one_layer_H_pre_t}
\end{equation}
}
where $i,k,j,l\in V$ and  $
K_{H,ab}^{(0,h)} := \frac{H^{0}_a \cdot H^0_b}{d}\,
$ with $d$ as the spectral-token embedding dimension.

(II)
If the parameters of the $\ell$-st graph convolution layer satisfy
\(
W_{\ell,ij}\sim \mathcal{N}(0,\sigma_w^2/d_{\ell-1}),
W^{\ell,H}_{ij}\sim \mathcal{N}\!\Big(0,\frac{\sigma_H^2}{d^{\ell,H}d_{\ell-1}}\Big)
\)
and the eigenvalue decoder layers satisfy
\(
W^{(\ell,h)}_{\lambda,ij}\sim \mathcal{N}(0,\sigma_\lambda^2/d_T)
\)
independently for all indices and spectral heads $h$, then, as
$d_\ell,d_T,d^{\ell,H}\to\infty$,
$
g^{\ell}(X)
$
converges in distribution to
$
\mathcal{GP}(0,\Sigma^{(\ell)})
$
with
{\small
\begin{equation}
\begin{aligned}
K_{\lambda,ab}
&:=
\delta_{h=h'}\,\mathbb{E}\big[
\bar\lambda^{(h)}_a(H)\,
\bar\lambda^{(h')}_b(H')
\big]\\
&=
\delta_{h=h'}\,\mathbb{E}\Big[
\sigma\!\big((\bar H^{(h)} W^{(\ell,h)}_\lambda)_a\big)\,
\sigma\!\big((\bar H^{(h')} W^{(\ell,h)}_\lambda)_b\big)
\Big],
\end{aligned}
\label{eq:specformer_lambda_kernel_T}
\end{equation}
}
{\small
\begin{equation}
\Sigma^{(\ell)}_{ab}
=
\sigma^2_H\sigma^2_w
U\Big(
K_{\lambda}
\odot
\big(U^\top K^{(\ell-1)}U\big)
\Big)U^\top,
\label{eq:specformer_X_propagated_T}
\end{equation}
}
where $a,b\in V$.
\end{theorem}

Equations~\eqref{eq:specformer_head_cov_H_t}--\eqref{eq:specformer_one_layer_H_pre_t}
describe how the covariance induced by the previous token layer is recursively
transformed through a standard transformer. At the terminal token layer, Equation~\eqref{eq:specformer_lambda_kernel_T} transfers the
resulting token-level covariance to the spectral domain by inducing a kernel over the
learned spectral filter coefficients. Equation~\eqref{eq:specformer_X_propagated_T} then defines the recursion of the node-feature kernel: the previous-layer node covariance is
mapped into the Laplacian eigenbasis, modulated elementwise according to the induced
spectral kernel, and projected back to the node domain. Together, these equations specify
a coupled kernel recursion in which attention governs spectral reweighting, and graph
convolution propagates the resulting covariance across layers.

\section{Proofs of Theorems~\ref{thm:GAT_GP}, ~\ref{thm:Graphormer_GP}, ~\ref{thm:Specformer_GP}, ~\ref{thm:GTN_GP}}
\paragraph{Proof technique.}
\label{app:proof of GP}
Our analysis follows the standard induction-on-layers framework for
establishing Gaussian process (GP) limits of randomly initialized deep
neural architectures
(\citet{Matthews2018,Lee2018,Novak2019,GarrigaAlonso2019,Yang2019a,InfiniteAttention}).
When the layer-$(\ell-1)$ representations converge in distribution
to a Gaussian process, we show that the layer-$\ell$ representations
also converge to a Gaussian process under suitable initialization and
moment conditions.

A key methodological step of this work is to develop a unified convergence
framework for graph transformers in the infinite-head regime.
At a high level, our strategy is to establish a single infinite-head
convergence theorem that yields a Gaussian-process limit for the model
outputs, while making explicit the analytic conditions required for
interchanging limits and expectations.

Concretely, Appendix~\ref{app:infinite_head_proof} proves a general
infinite-head convergence result for multi-head attention aggregation
layers, covering all four architectures studied in this paper within one
theorem. Importantly, this output-level convergence in the infinite-head
limit is not obtained in isolation: the proof requires
distributional convergence of the corresponding graph convolution operators,
together with uniform integrability, which ensures convergence of expectations and
justifies interchanging limits and $E[\cdot]$. Therefore,
we separately establish the convergence of the graph convolution operators
for each model Appendix~\ref{app:gat-proof}, \ref{app:graphormer-proof},
\ref{app:specformer_proof}, and \ref{app:gtn_proof}. These operator-level results, together
with uniform integrability (Lemma~\ref{lem:uniform-integrability}), yield
convergence of expectations via Theorem~\ref{thm:expectation-convergence}.
Combining these ingredients completes a single coherent route to the final,
uniform convergence statements for all models under infinitely many heads.

Technically, the core convergence arguments proceed by reducing
finite-dimensional distributional claims to one-dimensional linear
projections via the Cram\'er--Wold device (Lemma~\ref{lem:cramer-wold}),
followed by a central limit theorem for exchangeable triangular arrays
(Lemma~\ref{lem:blum}) applied at the multi-head aggregation level.
Slutsky's lemma (Lemma~\ref{lem:slutsky}) and the continuous mapping theorem
are then used to combine the convergent random components. Finally, passage
of limits through expectations is justified by uniform integrability
(Lemma~\ref{lem:uniform-integrability}) together with the
convergence-of-expectations theorem (Theorem~\ref{thm:expectation-convergence}),
which is enabled precisely by the integrability properties established
through the  graph convolution operator convergence in Appendix~\ref{app:gat-proof}--\ref{app:gtn_proof}.

\subsection{Convergence of $g^\ell(X)$}
\label{app:infinite_head_proof}
Since all this four models require the number of attention heads to grow to infinity in order to admit a well-defined Gaussian Process (GP) limit, we adopt a unified and general analytical framework to establish their GP convergence. 

We analyze the multi-head aggregation at a fixed layer index $\ell$,
mapping the GP input $f^{\ell-1}$ to the GP output $f^{\ell}$. For each head $h\in\{1,\dots,d^{\ell,H}\}$, node $a\in V$ and input $X\in\mathcal X_0$, define the single-head pre-activation output at layer $\ell$ by
\begin{equation}
g^{\ell h}(X)
:=  
S_{\mathrm{GNN}}^{(\ell,h)} f^{\ell-1}(X)\, W_{\ell,h},
\label{eq:single-head-output-new}
\end{equation}

The multi-head output at layer $\ell$ is given by
\begin{equation}
g^{\ell}(X)
=
\big[g^{\ell1}(X),\dots,g^{\ell d^{\ell,H}}(X)\big]\,
W^{\ell,H},
\label{eq:gat_concat-again}
\end{equation}
where $W^{\ell,H}$ is the output projection matrix, whose columns we write as
\[
W^{\ell,H} = \bigl[\,w^{\ell,H}_1,\dots,w^{\ell,H}_{d^{\ell,H}}\,\bigr],
\]
with $w^{\ell,H}_h$ corresponding to head $h$.
The variance of $W^{\ell,H}$ is chosen such that the overall scaling corresponds to a properly rescaled weighted sum over the $d^{\ell,H}$ attention heads.

\begin{lemma}[Adaptation of Theorem 2 from \citealp{Blum1958}]
\label{lem:blum}
For each $m\in\mathbb N$, let $\{X_{m,i} : i=1,2,\ldots\}$ be an infinitely exchangeable sequence with $\mathbb E[X_{n,1}]=0$ and $\mathbb E[X_{m,1}^2]=\sigma_n^2$, such that $\lim_{m\to\infty}\sigma_m^2=\sigma^2$ for some $\sigma^2\ge 0$.
Let
\begin{equation}
S_n
= \frac{1}{\sqrt{d_m}} \sum_{i=1}^{d_m} X_{m,i},
\label{eq:Sn}
\end{equation}
for some sequence $(d_m)_{m\ge 1}\subset\mathbb N$ such that $\lim_{m\to\infty}d_m=\infty$.
Assume:
\[
\begin{aligned}
(a)\;& \mathbb E[X_{m,1}X_{m,2}] = 0, \\
(b)\;& \lim_{n\to\infty} \mathbb E[X_{m,1}^2 X_{m,2}^2] = \sigma^4, \\
(c)\;& \mathbb E|X_{m,1}|^3 = o(\sqrt{d_m}).
\end{aligned}
\]
Then $S_n \xrightarrow{d} Z$, where $Z=0$ if $\sigma^2=0$, and $Z\sim\mathcal N(0,\sigma^2)$ otherwise.
\end{lemma}

\begin{lemma}[Node-level GP limit]
\label{lem:Node-gp-limit}
Let the assumptions of Theorem~\ref{thm:GAT_GP} , Theorem~\ref{thm:Graphormer_GP}, Theorem~\ref{thm:Specformer_GP} and Theorem~\ref{thm:GTN_GP} hold.
 The collection of
layer-$\ell$ pre-activation outputs
\[
g^{\ell}(X)
:= \bigl\{ g^{\ell}_a(X) : X\in\mathcal X_0,\ a\in V \bigr\}
\]
converges in distribution to a centred Gaussian process whose
finite-dimensional marginals are Gaussian with covariance given by the
node-level kernel $K^{(\ell)}$.
\end{lemma}

\emph{Proof.}
Following the Cramér--Wold device \citep[p.~383]{Billingsley1986}, it suffices
to establish convergence of one-dimensional projections of
finite-dimensional marginals of $g^{\ell}(X)$.

Fix a finite index set $\mathcal L\subset\mathcal X_0\times V$ and test vectors $
\{\alpha^{X,a}\in\mathbb R^{N} : (X,a)\in\mathcal L\}.$
Consider the linear statistic
\begin{equation}
T^{(\ell)}
:= \sum_{(X,a)\in\mathcal L} (\alpha^{X,a})^\top g_a^{\ell}(X).
\label{eq:T-layer-ell}
\end{equation}
Using the concatenation form \eqref{eq:gat_concat-again} and writing
$g_a^{(\ell+1)h}(X)$ for the $a$-th row of the $h$-th head, we obtain
\begin{align*}
T^{(\ell)}
&= \sum_{(X,a)\in\mathcal L} (\alpha^{X,a})^\top
\Bigl(\sum_{h=1}^{d^{\ell,H}} g_a^{\ell h}(X)\, w^{\ell,H}_h\Bigr) \\
&= \sum_{h=1}^{d^{\ell,H}}
\underbrace{\sum_{(X,a)\in\mathcal L}
(\alpha^{X,a})^\top g_a^{\ell h}(X)\, w^{\ell,H}_h}_{=: \;\zeta_h}.
\end{align*}
Define the rescaled head-level variables
\[
\psi_h := \sqrt{d^{\ell,H}}\;\zeta_h,
\qquad h=1,\dots,d^{\ell,H},
\]
so that
\begin{equation}
T^{(\ell)}
= \frac{1}{\sqrt{d^{\ell,H}}}
\sum_{h=1}^{d^{\ell,H}} \psi_h.
\label{eq:T-psi-sum}
\end{equation}

For each head width $d^{\ell,H}$, this defines a triangular array
\[
\{\psi_{d^{\ell,H},h}\}_{h=1}^{d^{\ell,H}},
\]
where, suppressing the explicit dependence on $d^{\ell,H}$ for notational
simplicity, each $\psi_h$ admits the representation
\begin{equation}
\psi_h
= \sqrt{d^{\ell,H}}
\sum_{(X,a)\in\mathcal L}
(\alpha^{X,a})^\top g^{\ell h}_a(X)\, w^{\ell,H}_h.
\label{eq:psi-head}
\end{equation}

Thus $T^{(\ell)}$ is of the form~\eqref{eq:Sn} in Lemma~\ref{lem:blum}, with
$X_{n,i}$ identified with $\psi_h$ and $d_n$ with $d^{\ell,H}$.
To invoke Lemma~\ref{lem:blum}, it suffices to verify conditions (a)--(c) for the
array $\{\psi_h\}_{h=1}^{d^{\ell,H}}$, which is achieved by the following
lemmas:
\begin{itemize}
    \item Exchangeability of $\{\psi_h\}$ in the head index $h$ is established in
    Lemma~\ref{lem:head-exchangeable-new}.
    \item Zero mean and vanishing cross-covariance,
    $\mathbb E[\psi_1]=0$ and $\mathbb E[\psi_1\psi_2]=0$,
    follow from Lemma~\ref{lem:head-zero-mean-new}.
    \item Convergence of the variance,
    $\mathbb E[\psi_1^2]\to\tau^2$,
    is proved in Lemma~\ref{lem:head-variance-new}.
    \item Convergence of the mixed fourth moment
    $\mathbb E[\psi_1^2\psi_2^2]\to\tau^4$ and the growth condition
    $\mathbb E|\psi_1|^3=o(\sqrt{d^{\ell,H}})$
    are established in Lemma~\ref{lem:head-wise mixed fourth moment}.
\end{itemize}

Therefore, by Lemma~\ref{lem:blum},
$T^{(\ell)}$ converges in distribution to a centred Gaussian random variable
with variance $\tau^2$.
Since $\mathcal L$ and the test vectors $\{\alpha^{X,a}\}$ are arbitrary, the
finite-dimensional distributions of $g^{\ell}(X)$ converge to those of a
Gaussian process with covariance given by the node-level kernel
$K^{(\ell)}$.
\hfill$\square$
\medskip

\begin{lemma}[Head-wise exchangeability]
\label{lem:head-exchangeable-new}
Under the assumptions of Theorem~\ref{thm:GAT_GP}, the random variables $\{\psi_{h}\}_{h=1}^{d^{\ell,H}}$ are exchangeable over the head index $h$.
\end{lemma}

\begin{proof}
In all four of our models, for distinct heads $h$, the parameters
\[
\{W_{\ell,h},\, w^{\ell,H}_h\}
\]
(including those appearing in $S_{\mathrm{GNN}}^{(\ell,h)}$) are i.i.d.\ across $h$ and independent of all previous-layer variables. 
Conditioning on the previous-layer representation $f^{\ell-1}(\cdot)$ (i.e., fixing $\{f^{\ell-1}(x):x\in\mathcal X\}$), \eqref{eq:psi-head} implies that $(\psi_1,\ldots,\psi_{d^{\ell,H}})$ are i.i.d.\ in $h$, and hence their conditional joint distribution is invariant under any permutation of the head indices. Therefore, by de Finetti's theorem, $\{\psi_h\}_{h\in\mathbb N}$ is infinitely exchangeable.

\end{proof}

\begin{lemma}[Head-wise zero mean and vanishing cross-covariance]
\label{lem:head-zero-mean-new}
$\mathbb E[\psi_{1}] = 
\mathbb E[\psi_{1}\psi_{2}] = 0.
$
\end{lemma}

\begin{proof}
From \eqref{eq:psi-head},
\[
\psi_1
= \sum_{(X,a)\in\mathcal L}\sqrt{d^{\ell,H}}
(\alpha^{X,a})^\top
\Bigl(S_{\mathrm{GNN}}^{(\ell,1)}\,f^{\ell}(X)\Bigr)_a\,
w^{\ell,H}_1.
\]
Condition on the $\sigma$-algebra generated by $\{f^{\ell}(X):X\in\mathcal X_0\}$ and $\{S_{\mathrm{GNN}}^{(\ell,h)}:h\in\mathbb N\}$.
Under this conditioning, $w^{\ell,H}_1$ is independent of all conditioned variables and is centred Gaussian. Therefore,
\[
\mathbb E\!\left[\psi_1 \,\middle|\, \{f^{\ell}(X)\}_{X\in\mathcal X_0}, \{S_{\mathrm{GNN}}^{(\ell,h)}\}_{h\in\mathbb N}\right]
=
\sum_{(X,a)\in\mathcal L}\sqrt{d^{\ell,H}}
(\alpha^{X,a})^\top
\Bigl(S_{\mathrm{GNN}}^{(\ell,1)}\,f^{\ell}(X) \Bigr)_a\,
\mathbb E[w^{\ell,H}_1]
=0,
\]
and hence $\mathbb E[\psi_1]=0$.

For $\mathbb E[\psi_1\psi_2]$, expand the product using \eqref{eq:psi-head}:
\[
\psi_1\psi_2
=
d^{\ell,H}\!\!\sum_{(X,a)\in\mathcal L}\sum_{(X',b)\in\mathcal L}
(\alpha^{X,a})^\top
\Bigl(S_{\mathrm{GNN}}^{(\ell,1)}\,f^{\ell}(X)\Bigr)_a\,
(\alpha^{X',b})^\top
\Bigl(S_{\mathrm{GNN}}^{(\ell,2)}\,f^{\ell}(X')\Bigr)_b\,
w^{\ell,H}_1\,w^{\ell,H}_2 .
\]
Conditioning on $\{f^{\ell}(X)\}_{X\in\mathcal X_0}$ and $\{S_{\mathrm{GNN}}^{(\ell,h)}\}_{h\in\mathbb N}$, the random variables $w^{\ell,H}_1$ and $w^{\ell,H}_2$ are independent, centred Gaussian and independent of the remaining factors in each summand. Thus
\[
\mathbb E[w^{\ell,H}_1 w^{\ell,H}_2]=
\mathbb E[w^{\ell,H}_1]\mathbb E[w^{\ell,H}_2]=0,
\]
which implies
\[
\mathbb E\!\left[\psi_1\psi_2 \,\middle|\, \{f^{\ell}(X)\}_{X\in\mathcal X_0}, \{S_{\mathrm{GNN}}^{(\ell,h)}\}_{h\in\mathbb N}\right]=0,
\qquad\text{and hence}\qquad
\mathbb E[\psi_1\psi_2]=0.
\]
\end{proof}

\begin{lemma}[Head-wise variance convergence]
\label{lem:head-variance-new}
There exists $\tau^2\ge 0$ such that
$
\lim_{d^{\ell,H}\to\infty}\mathbb E[\psi_{1}^2] = \tau^2 .
$
\end{lemma}

\begin{proof}
Observe that $\mathbb E[\psi_1^2]$ can be written as
\begin{align*}
\mathbb E[\psi_1^2]
&= d^{\ell,H}
\sum_{(X,a),(X',b)\in\mathcal L}
(\alpha^{X,a})^\top
\mathbb E\!\Big[
g^{\ell 1}(X)\,w^{\ell,H}_1 w^{\ell,H}_1{}^\top\,
g^{\ell 1}(X')^\top
\Big]
\alpha^{X',b}  \\
&\frac{\sigma_H^2}{d_\ell}
\sum_{(X,a),(X',b)\in\mathcal L}
(\alpha^{X,a})^\top
\mathbb E\!\Big[
g^{\ell 1}(X)\,g^{\ell 1}(X')^\top
\Big]
\alpha^{X',b},
\end{align*}
where we used $\mathbb E[w^{\ell,H}_1 w^{\ell,H}_1{}^\top]=\frac{\sigma_H^2}{d_\ell d^{\ell,H}}$.

\begin{align*}
\frac{\sigma_H^2}{d_\ell}\mathbb E\!\Big[
g^{\ell 1}(X)\,g^{\ell 1}(X)^\top
\Big]
&=
\frac{\sigma_H^2}{d_\ell}\mathbb E\!\Big[
\bigl(S_{\mathrm{GNN}}^{(\ell,1)} f^{\ell-1}(X) W_{\ell,1}\bigr)\,
\bigl(S_{\mathrm{GNN}}^{(\ell,1)} f^{\ell-1}(X) W_{\ell,1}\bigr)^\top
\Big]\\
&=\sigma_H^2\mathbb E\!\Big[ S_{\mathrm{GNN}}^{(\ell,1)}\frac{f^{\ell-1}(X)f^{\ell-1}(X)^\top}{d_{\ell-1}} (S_{\mathrm{GNN}}^{(\ell,1)})^\top  \Big]
\end{align*}
We may thus apply Lemma~\ref{thm:expectation-convergence}, which requires convergence in distribution of the integrands to the
relevant limit and uniform integrability of the integrand family.
Convergence in distribution follows by combining the model-specific convergence argument in
Lemma~\ref{lem:En-GP}, Lemma~\ref{lem:En-GP-graphormer}, Lemma~\ref{lem:S-specformer}, Lemma~\ref{lem:S-GTN}
with Lemma~\ref{lem:slutsky} and \citet[Lemma~33]{InfiniteAttention} combined continuous mapping theorem.
Uniform integrability is obtained by H\"older's inequality together with \citet[Lemma~32]{InfiniteAttention} and the corresponding
moment propagation result for $S_{\mathrm{GNN}}^{(\ell,h)}$, namely
Lemma~\ref{lem:mp-gat-polyphi}, Lemma~\ref{lem:mp-graphormer-polyphi},Lemma~\ref{lem:mp-specformer}, Lemma~\ref{lem:mp-gtn-finiteV} . Hence the integrand family is uniformly
integrable by Lemma~\ref{lem:uniform-integrability}, concluding the proof.
\end{proof}

\begin{lemma}
\label{lem:head-wise mixed fourth moment}
For any $h,h'\in\mathbb N$, $\mathbb E[\psi_h^2\psi_{h'}^2]$ converges to the mean of the weak limit of
$\{\psi_h^2\psi_{h'}^2\}_{d^{\ell,H}\to\infty}$.
\end{lemma}

\begin{proof}
It is enough to prove convergence of expectations of the form
\[
\mathbb E\!\left[
g^{\ell h}_a(X)\,
g^{\ell h}_b(X')^\top\,
g^{\ell h'}_c(Y)\,
g^{\ell h'}_d(Y')^\top
\right],
\]
where $a,b,c,d$ range over feature indices and $(X,X',Y,Y')\in\mathcal L$.

Using the definition of $g^{\ell h}$, we may rewrite the above as
\begin{align*}
&\mathbb E\!\left[
g^{\ell h}_a(X)\,
g^{\ell h}_b(X')^\top\,
g^{\ell h'}_c(Y)\,
g^{\ell h'}_d(Y')^\top
\right] \\
&\qquad=
\sigma_w^4\,
\mathbb E\Big[
\bigl(S_{\mathrm{GNN}}^{(\ell,h)}\bigr)_a\,
\frac{f^{\ell-1}(X)\,f^{\ell-1}(X')^\top}{d^{\ell-1}}\,
\bigl(S_{\mathrm{GNN}}^{(\ell,h)}\bigr)_b \\
&\qquad\qquad\quad
\bigl(S_{\mathrm{GNN}}^{(\ell,h')}\bigr)_c\,
\frac{f^{\ell-1}(Y)\,f^{\ell-1}(Y')^\top}{d^{\ell-1}}\,
\bigl(S_{\mathrm{GNN}}^{(\ell,h')}\bigr)_d
\Big].
\end{align*}

By \citet[Lemma~33]{InfiniteAttention},
\[
\left(
\frac{f^{\ell-1}(X)\,f^{\ell-1}(X')^\top}{d^{\ell-1}},
\frac{f^{\ell-1}(Y)\,f^{\ell-1}(Y')^\top}{d^{\ell-1}}
\right)
\xrightarrow{P}
\bigl(K^{(\ell-1)}(X,X'),\,K^{(\ell-1)}(Y,Y')\bigr).
\]
Since $S_{\mathrm{GNN}}^{(\ell,h)}$ and $S_{\mathrm{GNN}}^{(\ell,h')}$ converge in distribution (and are independent
of the weight matrices producing $f^{\ell-1}$). Then the entire integrand converges in distribution by Lemma~\ref{lem:slutsky}.
Finally, \citet[Theorem~3.5]{Billingsley1986} together with H\"older's inequality and \citet[Lemma~32]{InfiniteAttention} implies
uniform integrability of the integrand family and hence convergence of the expectation, concluding the proof.
\end{proof}

\subsection{Convergence of $S_{\mathrm{GAT}}^{(\ell,h)}$}
\label{app:gat-proof}
In Lemma~\ref{lem:Node-gp-limit}, we showed that, under the GP induction
hypothesis on the node features, it is sufficient to assume that
$S_{\mathrm{GNN}}^{(\ell,h)}$ converges in distribution in order for the
layer-$(\ell)$ pre-activation outputs $g^{\ell}(X)$ to converge in
distribution to a Gaussian process. In the standard GAT parameterisation,
$S_{\mathrm{GNN}}^{(\ell,h)}$ is a measurable function of the attention
logits $E^{(\ell h)}$, and hence convergence of $E^{(\ell h)}$ to a centred
Gaussian process implies convergence of $S_{\mathrm{GNN}}^{(\ell,h)}$ by the
continuous mapping theorem. Consequently, to establish Gaussian process
convergence of the GAT outputs, it suffices to prove that the attention
logits converge in distribution to a centred Gaussian process.

We now recall the version of the Blum--Kiefer--Rosenblatt central limit theorem
that we will use, following \citet[lemma~10]{Matthews2018}, and apply it to
establish convergence of the attention logits.
Fix a layer index $\ell$.
For each head $h$ and graph edge $(a,i)\in\mathcal E$, recall the attention logit
\[
E^{(\ell h)}_{ai}(X)
:= \vec v_{\ell h}^{\top}\big[W^{\ell,h} f_a^{\ell}(X),\; W^{\ell,h} f_i^{\ell}(X)\big],
\]
with the scaling and independence assumptions on $W_{\ell,h}$ and
$\vec a_{\ell h}$ given in Theorem~\ref{thm:GAT_GP}.

Because the graph structure in GAT only affects the model through subsequent
linear aggregation of the attention outputs, and linear transformations
preserve Gaussianity, it suffices to establish convergence of the attention
logits themselves.
More precisely, for the fixed layer $\ell$, we show that the collection of
logits
\[
E^{(\ell)}
:= \bigl\{E^{(\ell h)}_{ai}(X) : X\in\mathcal X_0,\ (a,i)\in\mathcal E,\ h\in\mathbb N\bigr\}
\]
has finite-dimensional marginals converging, as $d_\ell\to\infty$, to those of
a centred process with covariance given by the edge-level kernel
$K^{(\ell),E}$ defined in Equation~\eqref{eq:head-cov} of
Theorem~\ref{thm:GAT_GP}.
\begin{lemma}[Logit-level GP limit of GAT]
\label{lem:En-GP}
Let the assumptions of Theorem~\ref{thm:GAT_GP} hold.
Then, for the fixed layer $\ell$, the process
\[
E^{(\ell)} := \bigl\{E^{(\ell h)}_{ai}(X) : X\in\mathcal X_0,\ (a,i)\in\mathcal E,\ h\in\mathbb N\bigr\}
\]
converges in distribution (as $d_\ell\to\infty$) to a centred process whose finite-dimensional marginals are Gaussian with covariance given by the edge kernel $K^{(\ell),E}$ of \eqref{eq:head-cov}  in Theorem~~\ref{thm:GAT_GP}.
\end{lemma}

\emph{Proof.}
Using \citet[Lemma~27]{InfiniteAttention} and the Cramér--Wold device \citep{Billingsley1986}, we may restrict attention to one-dimensional projections of finite-dimensional marginals of $E^{(\ell)}$.
Let $\mathcal C\subset\mathcal X_0\times\mathcal E\times\mathbb N$ be finite, and choose test matrices $\{\beta^{X,ai,h}\}$ of the same shape as $E^{(\ell h)}_{ai}(X)$.
We consider linear statistics of the form
\begin{align*}
T^{E^\ell}
&:= \sum_{(X,(a,i),h)\in\mathcal C}
\bigl\langle \beta^{X,ai,h}, E^{(\ell h)}_{ai}(X) \bigr\rangle_F \\
&= \sum_{(X,(a,i),h)\in\mathcal C}
\left\langle
\beta^{X,ai,h},
\frac{1}{\sqrt{d_\ell}}
\sum_{j=1}^{d_\ell}
\Bigl[\vec v^{(L)}_{\ell h,j} W_{\ell,h,\cdot j} f_a^{\ell-1}(X) \cdot \mathbf 1^\top
 + \mathbf 1 \cdot \vec v^{(R)}_{\ell h,j} W_{\ell,h,\cdot j} f_i^{\ell-1}(X)\Bigr]
\right\rangle_F \\
&= \frac{1}{\sqrt{d_\ell}}
\sum_{j=1}^{d_\ell}
\underbrace{\sum_{(X,(a,i),h)\in\mathcal C}
\bigl\langle
\beta^{X,ai,h},
\vec v^{(L)}_{\ell h,j} W_{\ell,h,\cdot j} f_a^{\ell-1}(X) \cdot \mathbf 1^\top
+ \mathbf 1 \cdot \vec v^{(R)}_{\ell h,j} W_{\ell,h,\cdot j} f_i^{\ell-1}(X)
\bigr\rangle_F}_{=:\,\varphi_{j}}.
\end{align*}
Here $\langle\cdot,\cdot\rangle_F$ denotes the Frobenius inner product (equivalently, the Euclidean inner product on vectorised matrices). We decompose the vector $v$ into two subvectors $v^{(L)}$ and $v^{(R)}$, 
which have equal length and identical distribution.
Thus $T^{E^\ell}$ is of the form~\eqref{eq:Sn} with $X_{n,i}$ identified with $\varphi_j$ and $d_n$ with $d_\ell$.
To invoke Lemma~\ref{lem:blum} it suffices to verify conditions (a)--(c) for $\{\varphi_{j}\}_{j=1}^{d_\ell}$, which we do in Lemmas~\ref{lem:exchangeable}--\ref{lem:third-moment} below:
\begin{itemize}
    \item Exchangeability in $j$ is established in Lemma~\ref{lem:exchangeable}.
    \item Zero mean and vanishing cross-covariance follow from Lemma~\ref{lem:zero-mean}.
    \item Convergence of the variance is established in Lemma~\ref{lem:variance}.
    \item Convergence of $\mathbb E[\varphi_{1}^2\varphi_{2}^2]$ is proved in Lemma~\ref{lem:fourth-moment}.
    \item The $o(\sqrt{d_\ell})$ growth of the third absolute moment is implied by Lemma~\ref{lem:third-moment}.
\end{itemize}
Hence $T^{E^\ell}$ converges in distribution to a centred Gaussian random variable with variance determined by the limit of $\mathbb E[\varphi_{1}^2]$, which is exactly the covariance specified by the edge-level kernel $K^{(\ell h),E}$.
Since $\mathcal C$ and the test coefficients are arbitrary, the claim follows by Cramér--Wold.
\hfill$\square$

\begin{lemma}[Exchangeability over feature index]
\label{lem:exchangeable}
Under the assumptions of Theorem~\ref{thm:GAT_GP}, the random variables $\{\varphi_{j}\}_{j=1}^{d_\ell}$ are exchangeable over the index $j$.
\end{lemma}

\begin{lemma}[Zero mean and vanishing cross-covariance]
\label{lem:zero-mean}
Under the assumptions of Theorem~\ref{thm:GAT_GP}, we have
\[
\mathbb{E}[\varphi_{1}] = 0,
\qquad
\mathbb{E}[\varphi_{1}\varphi_{2}] = 0 .
\]
\end{lemma}

\begin{proof}
Recall that, for a finite index set 
$\mathcal C\subset\mathcal X_0\times\mathcal E\times\mathbb N$,
\[
\varphi_j
=
\sum_{(X,(a,i),h)\in\mathcal C}
\bigl\langle
\beta^{X,ai,h},
\vec v^{(L)}_{\ell h,1} W_{\ell,h,\cdot 1} f_a^{\ell-1}(X) \cdot \mathbf 1^\top
+ \mathbf 1 \cdot \vec v^{(R)}_{\ell h,1} W_{\ell,h,\cdot 1} f_i^{\ell-1}(X)
\bigr\rangle_F,
\qquad j=1,2,
\]
where $\langle\cdot,\cdot\rangle_F$ is the Frobenius inner product.

\medskip
\noindent
\emph{Mean.}
For any fixed $(X,(a,i),h)\in\mathcal C$,
\[
\mathbb{E}\!\left[
\bigl\langle
\beta^{X,ai,h},
\vec v^{(L)}_{\ell h,1} W_{\ell,h,\cdot 1} f_a^{\ell}(X) \cdot \mathbf 1^\top
+ \mathbf 1 \cdot \vec v^{(R)}_{\ell h,1} W_{\ell,h,\cdot 1} f_i^{\ell}(X)
\bigr\rangle_F
\right]
=
\bigl\langle \beta^{X,ai,h}, 0\bigr\rangle_F
=0,
\]
since $W_{\ell,h,\cdot 1}$, $\vec v^{(L)}_{\ell h,1}$ and $\vec v^{(R)}_{\ell h,1}$ are independent, centred, and have finite moments, and the entries of $f^{\ell}(X)$ are almost surely finite.
Summing over $(X,(a,i),h)\in\mathcal C$ gives $\mathbb{E}[\varphi_1]=0$.

\medskip
\noindent
\emph{Cross-covariance.}
Define, for $j=1,2$,
\[
R_{j}^{h}(X;a,i)
:=
\vec v^{(L)}_{\ell h,j} W_{\ell,h,\cdot j} f_a^{\ell-1}(X) \cdot \mathbf 1^\top
+ \mathbf 1 \cdot \vec v^{(R)}_{\ell h,j} W_{\ell,h,\cdot j} f_i^{\ell-1}(X) .
\]
Then
\[
\varphi_j
=
\sum_{(X,(a,i),h)\in\mathcal C}
\langle \beta^{X,ai,h}, R_{j}^{h}(X;a,i)\rangle_F,
\qquad j=1,2,
\]
and hence
\begin{align*}
\mathbb{E}[\varphi_1\varphi_2]
&=
\sum_{(X,(a,i),h)}\sum_{(X',(a',i'),h')}
(\beta^{X,ai,h})^\top
\mathbb{E}\!\Big[
R_{1}^{h}(X;a,i)\, R_{2}^{h'}(X';a',i')^\top
\Big]
\beta^{X',a'i',h'} .
\end{align*}
We now expand the inner expectation.
By definition,
\begin{align*}
R_{1}^{h}(X;a,i)
&=
\vec v^{(L)}_{\ell h,1} W_{\ell,h,\cdot 1} f_a^{\ell-1}(X) \cdot \mathbf 1^\top
+ \mathbf 1 \cdot \vec v^{(R)}_{\ell h,1} W_{\ell,h,\cdot 1} f_i^{\ell-1}(X),\\
R_{2}^{h'}(X';a',i')
&=
\vec v^{(L)}_{\ell h',2} W_{\ell,h',\cdot 2} f_{a'}^{\ell-1}(X') \cdot \mathbf 1^\top
+ \mathbf 1 \cdot \vec v^{(R)}_{\ell h',2} W_{\ell,h',\cdot 2} f_{i'}^{\ell-1}(X').
\end{align*}
Thus
\begin{align*}
&\mathbb{E}\!\Big[
R_{1}^{h}(X;a,i)\, R_{2}^{h'}(X';a',i')^\top
\Big] \\
&=
\mathbb{E}\Big[
\vec v^{(L)}_{\ell h,1} W_{\ell,h,\cdot 1} f_a^{\ell-1}(X) f_{a'}^{\ell-1}(X')^\top
 (W_{\ell,h',\cdot 2})^\top \vec (v^{(R)}_{\ell h',2})^\top \\
&\qquad
+ \vec v^{(L)}_{\ell h,1} W_{\ell,h,\cdot 1} f_a^{\ell-1}(X) f_{i'}^{\ell-1}(X')^\top
 (W_{\ell,h',\cdot 2})^\top (\vec v^{(R)}_{\ell h',2})^\top \\
&\qquad
+ \vec v^{(R)}_{\ell h,1} W_{\ell,h,\cdot 1} f_i^{\ell-1}(X) f_{a'}^{\ell-1}(X')^\top
 (W_{\ell,h',\cdot 2})^\top (\vec v^{(L)}_{\ell h',2})^\top \\
&\qquad
+ \vec v^{(R)}_{\ell h,1} W_{\ell,h,\cdot 1} f_i^{\ell-1}(X) f_{i'}^{\ell-1}(X')^\top
 (W_{\ell,h',\cdot 2})^\top (\vec v^{(R)}_{\ell h',2})^\top
\Big].
\end{align*}
In each of the four terms above, the random matrices
$W_{\ell,h,\cdot 1}$ and $W_{\ell,h',\cdot 2}$ appear with total degree one.
Since the columns $\{W_{\ell,h,\cdot j}\}_{j}$ are independent and centred,
and are also independent of $f^{\ell}(X)$ and $f^{\ell}(X')$, we obtain
\[
\mathbb{E}\!\Big[
R_{1}^{h}(X;a,i)\, R_{2}^{h'}(X';a',i')^\top
\Big]=0
\]
entrywise, for all $(X,(a,i),h)$ and $(X',(a',i'),h')$.
Therefore every term in the double sum for $\mathbb{E}[\varphi_1\varphi_2]$
vanishes, and hence
\[
\mathbb{E}[\varphi_1\varphi_2]=0 .
\]
This completes the proof.
\end{proof}

\begin{lemma}[Convergence of the variance]
\label{lem:variance}
Under the assumptions of Theorem~\ref{thm:GAT_GP}, there exists $\sigma_\ast^2\ge 0$ such that
\[
\lim_{d_\ell\to\infty}\mathbb E[\varphi_{1}^2] = \sigma_\ast^2.
\]
\end{lemma}

\begin{proof}
From the definition of $\varphi_1$ in Lemma~\ref{lem:zero-mean}, we can write
\[
\mathbb E[\varphi_1^2]
= \sum_{(X,(a,i),h)}\sum_{(X',(a',i'),h')}
(\beta^{X,ai,h})^\top
\mathbb E\!\Big[
R_{1}^{h}(X;a,i)\, R_{1}^{h'}(X';a',i')^\top
\Big]
\beta^{X',a'i',h'},
\]
where
\[
R_{1}^{h}(X;a,i)
:= \vec v^{(L)}_{1} W_{\ell,h,\cdot 1} f_a^{\ell-1}(X) \cdot \mathbf 1^\top
 + \mathbf 1 \cdot \vec v^{(R)}_{1} W_{\ell,h,\cdot 1} f_i^{\ell-1}(X).
\]

The inner expectation can be expanded as
\begin{align*}
&\mathbb E\!\Big[
R_{1}^{h}(X;a,i)\, R_{1}^{h'}(X';a',i')^\top
\Big] \\
&\quad = \mathbb E\Big[
\bigl(\vec v^{(L)}_{1} W_{\ell,h,\cdot 1} f_a^{\ell}(X) \cdot \mathbf 1^\top
 + \mathbf 1 \cdot \vec v^{(R)}_{1} W_{\ell,h,\cdot 1} f_i^{\ell}(X)\bigr)
\bigl(\vec v^{(L)}_{1} W_{\ell,h',\cdot 1} f_{a'}^{\ell}(X') \cdot \mathbf 1^\top
 + \mathbf 1 \cdot \vec v^{(R)}_{1} W_{\ell,h',\cdot 1} f_{i'}^{\ell}(X')\bigr)^\top
\Big] \\
&\quad = \mathbb E\Big[
\vec v^{(L)}_{1} W_{\ell,h,\cdot 1} f_a^{\ell-1}(X) f_{a'}^{\ell-1}(X')^\top
 (W_{\ell,h',\cdot 1})^\top (\vec v^{(L)}_{1})^\top \\
&\qquad\qquad
+ \vec v^{(L)}_{1} W_{\ell,h,\cdot 1} f_a^{\ell-1}(X) f_{i'}^{\ell-1}(X')^\top
 (W_{\ell,h',\cdot 1})^\top (\vec v^{(R)}_{1})^\top \\
&\qquad\qquad
+ \vec v^{(R)}_{1} W_{\ell,h,\cdot 1} f_i^{\ell-1}(X) f_{a'}^{\ell-1}(X')^\top
 (W_{\ell,h',\cdot 1})^\top (\vec v^{(L)}_{1})^\top \\
&\qquad\qquad
+ \vec v^{(R)}_{1} W_{\ell,h,\cdot 1} f_i^{\ell-1}(X) f_{i'}^{\ell-1}(X')^\top
 (W_{\ell,h',\cdot 1})^\top (\vec v^{(R)}_{1})^\top
\Big].
\end{align*}

Hence
\begin{align*}
\mathbb E[\varphi_1^2]
&= \sum_{(X,(a,i),h)}\sum_{(X',(a',i'),h')}
(\beta^{X,ai,h})^\top \Big\{
\mathbb E\Big[
v^{(L)}_{1} W_{\ell,h,\cdot 1} f_a^{\ell-1}(X) f_{a'}^{\ell-1}(X')^\top
 (W_{\ell,h',\cdot 1})^\top (v^{(L)}_{1})^\top
\Big] \\
&\qquad
+ \mathbb E\Big[
v^{(L)}_{1} W_{\ell,h,\cdot 1} f_a^{\ell-1}(X) f_{i'}^{\ell-1}(X')^\top
 (W_{\ell,h',\cdot 1})^\top (v^{(R)}_{1})^\top
\Big] \\
&\qquad
+ \mathbb E\Big[
v^{(R)}_{1} W_{\ell,h,\cdot 1} f_i^{\ell-1}(X) f_{a'}^{\ell-1}(X')^\top
 (W_{\ell,h',\cdot 1})^\top (v^{(L)}_{1})^\top
\Big] \\
&\qquad
+ \mathbb E\Big[
v^{(R)}_{1} W_{\ell,h,\cdot 1} f_i^{\ell-1}(X) f_{i'}^{\ell-1}(X')^\top
 (W_{\ell,h',\cdot 1})^\top (v^{(R)}_{1})^\top
\Big]
\Big\}\,\beta^{X',a'i',h'}.
\end{align*}

Using independence, zero mean and the variance scaling of $W_{\ell,h,\cdot 1}$ and $W_{\ell,h',\cdot 1}$ from Equation~(5) of Theorem~1, each of the four expectations above can be rewritten as a constant multiple of
\[
\mathbb E\!\left[
\frac{f_u^{\ell-1}(X) f_v^{\ell-1}(X'){}^\top}{d_{\ell-1}}
\right]
\]
with $(u,v)\in\{a,i\}\times\{a',i'\}$.
By the inductive GP hypothesis for layer $\ell$, these expectations converge entrywise to the corresponding entries of $K^{(\ell-1)}(X,X')$\citep[Lemma~33]{InfiniteAttention}.
Therefore $\mathbb E[\varphi_1^2]$ converges to a finite limit, which we denote by $\sigma_\ast^2$.
\end{proof}

\begin{lemma}[Convergence of the mixed fourth moment]
\label{lem:fourth-moment}
Under the assumptions of Theorem~\ref{thm:GAT_GP},
\[
\lim_{d_\ell\to\infty}\mathbb E[\varphi_{1}^2\varphi_{2}^2] = \sigma_\ast^4.
\]
\end{lemma}

\begin{proof}
Define
\[
R_{j}^{h}(X;a,i)
:= \vec v^{(L)}_{j} W_{\ell,h,\cdot j} f_a^{\ell-1}(X) \cdot \mathbf 1^\top
 + \mathbf 1 \cdot \vec v^{(R)}_{j} W_{\ell,h,\cdot j} f_i^{\ell-1}(X),
\qquad j=1,2.
\]
Then
\[
\varphi_j
= \sum_{(X,(a,i),h)\in\mathcal C}
\langle \beta^{X,ai,h}, R_{j}^{h}(X;a,i)\rangle_F,
\quad j=1,2.
\]
Hence
\begin{align*}
\mathbb E[\varphi_1^2\varphi_2^2]
&= \sum_{(X,(a,i),h)}\sum_{(X',(a',i'),h')}
(\beta^{X,ai,h})^\top
\mathbb E\!\Big[
R_{1}^{h}(X;a,i) R_{1}^{h}(X;a,i)^\top
\Big]
\beta^{X,ai,h} \\
&\qquad\qquad\cdot
(\beta^{X',a'i',h'})^\top
\mathbb E\!\Big[
R_{2}^{h'}(X';a',i') R_{2}^{h'}(X';a',i')^\top
\Big]
\beta^{X',a'i',h'}.
\end{align*}
We have, for example,
\begin{align*}
&\mathbb E\!\Big[
R_{1}^{h}(X;a,i) R_{1}^{h}(X;a,i)^\top
\Big] \\
&\quad = \mathbb E\Big\{
\bigl(\vec v^{(L)}_{1} W_{\ell,h,\cdot 1} f_a^{\ell-1}(X) \cdot \mathbf 1^\top
 + \mathbf 1 \cdot \vec v^{(R)}_{1} W_{\ell,h,\cdot 1} f_i^{\ell-1}(X)\bigr) \\
&\qquad\qquad\qquad\qquad\times
\bigl(\vec v^{(L)}_{1} W_{\ell,h,\cdot 1} f_a^{\ell-1}(X) \cdot \mathbf 1^\top
 + \mathbf 1 \cdot \vec v^{(R)}_{1} W_{\ell,h,\cdot 1} f_i^{\ell-1}(X)\bigr)^\top
\Big\},
\end{align*}
and similarly for $R_{2}^{h'}(X';a',i')$ with parameters $(\vec a_2,\vec b_2)$.
Thus a generic term in the expansion of $\mathbb E[\varphi_1^2\varphi_2^2]$ takes the form
\begin{align*}
\mathbb E\Big\{
&\bigl(
    v^{(L)}_{1} W_{\ell,h,\cdot 1} f_a^{\ell-1}(X)\, \mathbf 1^\top
  + \mathbf 1\, v^{(R)}_{1} W_{\ell,h,\cdot 1} f_i^{\ell-1}(X)
\bigr)
\bigl(
    v^{(L)}_{1} W_{\ell,h,\cdot 1} f_a^{\ell-1}(X)\, \mathbf 1^\top
  + \mathbf 1\, v^{(R)}_{1} W_{\ell,h,\cdot 1} f_i^{\ell-1}(X)
\bigr)^\top \\
&\qquad\times
\bigl(
    v^{(L)}_{2} W_{\ell,h',\cdot 2} f_{a'}^{\ell-1}(X')\, \mathbf 1^\top
  + \mathbf 1\, v^{(R)}_{2} W_{\ell,h',\cdot 2} f_{i'}^{\ell-1}(X')
\bigr)
\bigl(
    v^{(L)}_{2} W_{\ell,h',\cdot 2} f_{a'}^{\ell-1}(X')\, \mathbf 1^\top
  + \mathbf 1\, v^{(R)}_{2} W_{\ell,h',\cdot 2} f_{i'}^{\ell-1}(X')
\bigr)^\top
\Big\}.
\end{align*}

Expanding this product yields a finite sum of terms of the form
\[
\mathbb E\Big[
\vec u_1^\top W_{\ell,h,\cdot 1} f_u^{\ell-1}(X) f_z^{\ell-1}(X)^\top
 (W_{\ell,h,\cdot 1})^\top \vec u_2
\cdot
\vec z_1^\top W_{\ell,h',\cdot 2} f_{u'}^{\ell-1}(X') f_{z'}^{\ell-1}(X')^\top
 (W_{\ell,h',\cdot 2})^\top \vec z_2
\Big],
\]
where $\vec u_1,\vec u_2,\vec z_1,\vec z_2$ are fixed coefficient vectors formed from $\vec v^{(L)}_j,\vec v^{(R)}_j$, and $(u,z)\in\{a,i\}^2$, $(u',z')\in\{a',i'\}^2$.
Using independence and the variance scaling of $W_{\ell,h,\cdot 1},W_{\ell,h',\cdot 2}$, each such term is proportional to
\[
\mathbb E\Big[
\frac{f_u^{\ell-1}(X) f_z^{\ell-1}(X)^\top}{d_{\ell-1}}
\frac{f_{u'}^{\ell-1}(X') f_{z'}^{\ell-1}(X')^\top}{d_{\ell-1}}
\Big].
\]

Thus $\mathbb E[\varphi_1^2\varphi_2^2]$ can be written as a weighted sum of expectations of the form
\[
\mathbb E\Big[
\frac{f_u^{\ell-1}(X) f_v^{\ell-1}(X)^\top}{d_{\ell-1}}
\frac{f_{u'}^{\ell-1}(X') f_{v'}^{\ell-1}(X')^\top}{d_{\ell-1}}
\Big],
\]
with $(u,v,u',v')$ ranging over a finite index set.
By the inductive GP hypothesis for layer $\ell$ and Lemma~\ref{lem:variance}, each factor
\[
\frac{f_u^{\ell-1}(X) f_z^{\ell-1}(X)^\top}{d_{\ell-1}}
\]
converges in probability to the corresponding kernel entry $K^{(\ell-1)}_{uz}(X,X)$, and similarly for $(X',u',z')$.
By the continuous mapping theorem the products converge in probability to sums of products of kernel limits such as
\[
K^{(\ell-1)}_{uz}(X,X)\,K^{(\ell-1)}_{u'z'}(X',X').
\]
Using the eighth-moment bound on $f^{\ell}$ (as in Lemma~\ref{lem:third-moment}) and Hölder's inequality, the collection of these products is uniformly integrable.
Hence the expectations converge to the corresponding products of limits, and we obtain
\[
\lim_{d_\ell\to\infty}\mathbb E[\varphi_1^2\varphi_2^2] = \sigma_\ast^4,
\]
with $\sigma_\ast^2$ as in Lemma~\ref{lem:variance}.
\end{proof}

\begin{lemma}[Third absolute moment]
\label{lem:third-moment}
Under the assumptions of Theorem~\ref{thm:GAT_GP},
\[
\mathbb E\bigl|\varphi_{1}\bigr|^3 = o\bigl(\sqrt{d_\ell}\bigr).
\]
\end{lemma}

\emph{Proof.}
By Hölder's inequality, it suffices to show that
\[
\limsup_{d_\ell\to\infty}\mathbb E\bigl|\varphi_{1}\bigr|^4 < \infty.
\]
Using the same notation $R_{1}^h(X)$ as above, we can write
\[
\mathbb E\bigl|\varphi_{1}\bigr|^4
= \sum_{(X,(a,i),h),(X',(a',i'),h')}
(\beta^{X,ai,h})^\top
\mathbb E\!\Big[
R_{1}^{h}(X) R_{1}^{h}(X)^\top
R_{1}^{h'}(X') R_{1}^{h'}(X')^\top
\Big]
\beta^{X',a'i'}.
\]
Each expectation is a finite sum of terms bounded by
\[
\max_{c\in V}\max_{Z\in\{X,X'\}}\mathbb E\bigl|f_{c1}^{\ell}(Z)\bigr|^8,
\]
which is uniformly bounded in $d_\ell$ by the same eighth-moment assumption used in \citet[Lemma~32]{InfiniteAttention}.
Thus $\mathbb E|\varphi_{1}|^4$ is uniformly bounded, which implies $\mathbb E|\varphi_{1}|^3 = o(\sqrt{d_\ell})$ and completes the proof.
\hfill$\square$

\subsection{Convergence of $S_{\mathrm{Graphormer}}^{(\ell,h)}$}
\label{app:graphormer-proof}
Since the infinite-width limit for networks with positional encodings is treated in
\citet[Appendix~C]{InfiniteAttention}, and since Lemma~\ref{lem:Node-gp-limit} reduces the
infinite-head argument to convergence in distribution of $S_{\mathrm{GNN}}^{(\ell,h)}$, it
remains to establish this convergence for Graphormer with structural bias. In Graphormer,
$S_{\mathrm{GNN}}^{(\ell,h)}$ is obtained from the attention logits $E^{(\ell h)}$ by adding
the structural bias and then applying a softmax functin. Therefore $S_{\mathrm{GNN}}^{(\ell,h)}$
is a measurable function of the bias-augmented logits, and convergence in distribution of the
logits implies convergence of $S_{\mathrm{GNN}}^{(\ell,h)}$ by the continuous mapping theorem.
Hence it suffices to prove that the attention logits still converge when the structural bias
is included.
\begin{lemma}[Logit-level GP limit of Graphormer]
\label{lem:En-GP-graphormer}
Let the assumptions of Theorem~\ref{thm:Graphormer_GP} hold.
Then, for the fixed layer $\ell$, the process
\[
E^{(\ell)} := \bigl\{E^{(\ell h)}_{ai}(X) : X\in\mathcal X_0,\ (a,i)\in(V\times V),\ h\in\mathbb N\bigr\}
\]
converges in distribution (as $d_\ell\to\infty$) to a centred process whose finite-dimensional marginals are Gaussian with covariance given by the edge kernel $K^{(\ell),E}$ of \eqref{eq:head-cov}  in Theorem~~\ref{thm:Graphormer_GP}.
\end{lemma}
\proof
By definition of the Graphormer attention logits
(Equation~\eqref{eq:attention_score_graphormer}),
we have
\[
E^{(\ell h)}_{ij}(X)
=
\frac{(\tilde f^{\ell}_i(X) W^{Q,\ell h})
(\tilde f^{\ell}_j(X) W^{K,\ell h})^\top}{\sqrt{d_\ell}}
+
b_{\rho(i,j)}.
\]
Define $G_{ij}^{(\ell h)}(X) = \frac{\big(f^{\ell}_i(X) W^{Q,\ell h}\big) \big(f^{\ell}_j(X) W^{K,\ell h}\big)^{\top}}{\sqrt{d_\ell}}$. 
The covariance $\mathbb{E}\big[E^{(\ell h)}_{ai}(X) E^{(\ell h')}_{bj}(X)\big]$ can be rewritten as:
\begin{align*}
\mathbb{E}\big[E^{(\ell h)}_{ai}(X) E^{(\ell h')}_{bj}(X)\big] 
&= \mathbb{E}\big[ (G^{\ell h}_{ai}(X) + b_{\rho(a,i)}) (G^{\ell h'}_{bj}(X) + b_{\rho(b,j)}) \big] \\
&= \mathbb{E}\big[G^{\ell h}_{ai}(X)G^{\ell h'}_{bj}(X) + b_{\rho(a,i)}G^{\ell h'}_{bj}(X) + b_{\rho(b,j)}G^{\ell h}_{ai}(X) + b_{\rho(a,i)}b_{\rho(b,j)}\big] \\
&= \mathbb{E}\big[G^{\ell h}_{ai}(X)G^{\ell h'}_{bj}(X)\big] + \mathbb{E}\big[b_{\rho(a,i)}b_{\rho(b,j)}\big]
\end{align*}
Since the convergence of $G^{(\ell h)}(X)$ have already been proved in \citet[Lemma 12]{InfiniteAttention}, and $b_{\phi(v_i,v_j)}$ are independent for different value of $\phi(i,j)$, thus
\[
\mathbb{E}\big[E^{(\ell h)}_{ai}(X)\,E^{(\ell h')}_{bj}(X)\big] \\
= \delta_{h=h'}\Big(
\sigma_Q^2\sigma_K^2\,
\tilde K^{(\ell-1)}_{ab}\,
\tilde K^{(\ell-1)}_{ij}
+
\sigma_b^2\,
\mathbf 1_{\phi(a,i)=\phi(b,j)}
\Big).
\]

\subsection{Convergence of $S_{\mathrm{Specformer}}^{h}$}
\label{app:specformer_proof}
In Lemma~\ref{lem:Node-gp-limit}, we showed that, under the GP induction
hypothesis on the node features, it is sufficient to establish that
$S_{\mathrm{GNN}}^{(\ell,h)}$ converges in distribution in order for the
layer-$\ell$ pre-activation outputs $g^{\ell}(X)$ to converge in
distribution to a Gaussian process. In Specformer, the message-passing
operator is defined by Equation\eqref{eq:S_specformer}, so $S_{\mathrm{GNN}}^{(\ell,h)}
\equiv S^h_{\mathrm{Specformer}}$. Consequently, to establish Gaussian process
convergence of the Specformer outputs, it suffices to prove that
$S^h_{\mathrm{Specformer}}$ converges in distribution.
\begin{lemma}
\label{lem:S-specformer}
Let the assumptions of Theorem~\ref{thm:Specformer_GP} hold.
Then, for the fixed layer $\ell$, the process
\[
S_{\mathrm{Specformer}} := \bigl\{S_{\mathrm{Specformer}}^{h} :  h\in\mathbb N\bigr\}
\]
converges in distribution.
\end{lemma}

\begin{proof}
Recall from \eqref{eq:S_specformer} that the Specformer convolution operator in layer $\ell$ and head $h$ is defined as
\[
S_{\mathrm{Specformer}}^{(\ell,h)}
=U\,\mathrm{diag}\!\bigl(\bar\lambda^{(h)}\bigr)\,U^\top,
\qquad 
\bar\lambda^{(h)}=\sigma\!\bigl(\tilde\lambda^{(h)}\bigr),
\qquad 
\tilde\lambda^{(h)}:=\bar H^{(h)} W_{\lambda,h},
\qquad 
\bar H^{(h)}:=H^{T,(h)}.
\]
Since the eigenvalue encoding block is a standard Transformer, \citet[Appendix~B.1]{InfiniteAttention} establishes that
$H^{t+1}$ converges in distribution to a centred Gaussian process
\[
H^{t+1} \sim \mathcal{GP}\!\bigl(0, K_H^{(t+1)}\bigr)
\]
as $d_t,d^{t,H}\to\infty$, with the kernel recursion given in Equation~\eqref{eq:specformer_one_layer_H_pre_t}. In particular,
$\bar H^{(h)}=H^{T,(h)}$ converges in distribution to a centred Gaussian process with covariance $K_H^{(T)}$. With the standard
$1/d_T$ variance scaling and independence of $W_{\lambda,h}$, each coordinate of
$\tilde\lambda^{(h)}=\bar H^{(h)}W_{\lambda,h}$ is a sum of $d_T$ independent centred terms, so by the CLT
$\tilde\lambda^{(h)}$ converges in distribution to a centred Gaussian process with covariance
\[
C_{\lambda,ab}
:=\mathbb E\!\left[(\tilde\lambda^{(h)})_a\,(\tilde\lambda^{(h)\prime})_b\right]
=\sigma_\lambda^2\,K^{(T)}_{H,ab}.
\]
Consequently, $\bar\lambda^{(h)}=\sigma(\tilde\lambda^{(h)})$ also converges in distribution. Since $U$ is fixed and the map
$\lambda\mapsto U\,\mathrm{diag}(\lambda)\,U^\top$ is continuous, the continuous mapping theorem implies that
$S_{\mathrm{Specformer}}^{(\ell,h)}$ converges in distribution as $d_T\to\infty$.

Finally, writing $K^{(\ell-1)}(X,X'):=\mathbb E[f^{\ell-1}(X)f^{\ell-1}(X')^\top]$ and letting $(\bar H,\bar H')$ denote the
pair of limiting Gaussian-process draws associated with $(X,X')$, we obtain
\begin{align*}
K^{(\ell)}(X,X')
&:=\mathbb E\!\left[g^{\ell}(X)\,g^{\ell}(X')^\top\right]
=\sigma_O^2\sigma_w^2\,\mathbb E\!\left[S^1_{\mathrm{Specformer}}(\bar H)\,K^{(\ell-1)}(X,X')\,S^1_{\mathrm{Specformer}}(\bar H')^\top\right]\\
&=\sigma_O^2\sigma_w^2\,\mathbb E\!\left[U\,\mathrm{diag}(\bar\lambda^1(\bar H))\,(U^\top K^{(\ell-1)}(X,X')U)\,
\mathrm{diag}(\bar\lambda^1(\bar H'))\,U^\top\right]\\
&=\sigma_O^2\sigma_w^2\,U\Big( K_{\lambda}\odot\big(U^\top K^{(\ell-1)}(X,X')U\big)\Big)U^\top,
\end{align*}
where $K_{\lambda,ab}:=\mathbb E[\bar\lambda_a(\bar H)\bar\lambda_b(\bar H')]$. Therefore the Specformer outputs $g^{\ell}(X)$
converge in distribution to a centred Gaussian process with covariance kernel $K^{(\ell)}$.
\end{proof}

\subsection{Auxiliary results}
\begin{lemma}[ {\citep[Theorem~29.4]{Billingsley1986}}]
\label{lem:cramer-wold}
Let $X, (X_n)_{n\ge 1}$ be random variables taking values in
$(\mathbb R^N, \mathcal B^N)$, where $\mathcal B^N$ denotes the Borel
$\sigma$-algebra.
Then $X_n \xrightarrow{d} X$ if and only if for every finite index set
$J \subset \mathbb N$ and the corresponding projection
$\Gamma^J : \mathbb R^N \to \mathbb R^J$, we have
\[
\Gamma^J(X_n) \xrightarrow{d} \Gamma^J(X) \quad \text{as } n \to \infty.
\]
\end{lemma}

\begin{lemma}[ {\citep[p.~31]{Billingsley1986}}]
\label{lem:uniform-integrability}
A sequence of real-valued random variables $(X_n)_{n\ge 1}$ is uniformly
integrable if
\[
\sup_{n \ge 1} \mathbb E\bigl[ |X_n|^{1+\varepsilon} \bigr] < \infty
\]
for some $\varepsilon > 0$.
\end{lemma}

\begin{theorem}[{\citep[Theorem~3.5]{Billingsley1986}}]
\label{thm:expectation-convergence}
If $(X_n)_{n\ge 1}$ is uniformly integrable and $X_n \xrightarrow{d} X$,
then $X$ is integrable and
\[
\mathbb E[X_n] \to \mathbb E[X].
\]
\end{theorem}

\begin{lemma}[Slutsky’s lemma]
\label{lem:slutsky}
Let $X, (X_n)_{n\ge 1}$ and $(Y_n)_{n\ge 1}$ be real-valued random variables
defined on the same probability space.
If $X_n \xrightarrow{d} X$ and $Y_n \xrightarrow{P} c$ for some constant
$c \in \mathbb R$, then
\[
X_n Y_n \xrightarrow{d} cX,
\qquad
X_n + Y_n \xrightarrow{d} X + c.
\]
\end{lemma}

\begin{lemma}[Moment bound for GAT operators]
\label{lem:mp-gat-polyphi}
Under the assumptions of Theorem~\ref{thm:GAT_GP}, suppose that $\phi:\mathbb R\to\mathbb R$
is entrywise polynomially bounded, i.e., 
Then for any $t\ge 1$,
\begin{equation}
\sup_{a\in V}\ \sup_{i\in\mathcal N_a}\ 
\sup_{d_\ell\in\mathbb N}
\mathbb E\bigl|S^{(\ell,h)}_{\mathrm{GAT},ai}(X)\bigr|^{t}
<\infty.
\label{eq:gat_S_moment_bound}
\end{equation}
\end{lemma}

\begin{proof}
For each layer width $d_\ell\in\mathbb N$, recall that the unnormalized
GAT attention score is given by
\[
E^{(\ell h, d_\ell)}_{ai}(X)
=
(\vec v^{\ell h})^\top
\Big[
W^{\ell,h} f^{\ell-1}_{a}(X)\ \Vert\ W^{\ell,h} f^{\ell-1}_{i}(X)
\Big],
\]
where $W^{\ell,h}$ and $\vec v^{\ell h}$ are independently fan-in scaled
Gaussian and independent of $f^{\ell-1}(x)$.

Conditioned on
\[
u^{(d_\ell)}_{ai}(X):=
\Big[
W^{\ell,h} f^{\ell-1}_{a}(X)\ \Vert\ W^{\ell,h} f^{\ell-1}_{i}(X)
\Big],
\]
the random variable $E^{(\ell h, d_\ell)}_{ai}(X)$ is centered Gaussian
with variance proportional to $\|u^{(d_\ell)}_{ai}(X)\|^2$. Hence for any
$t\ge 1$,
\[
\mathbb E\bigl(
|E^{(\ell h, d_\ell)}_{ai}(X)|^t
\mid
u^{(d_\ell)}_{ai}(X)
\bigr)
\ \lesssim\ 
\|u^{(d_\ell)}_{ai}(X)\|^t.
\]

Taking expectations and using the assumptions of
Theorem~\ref{thm:GAT_GP}, which ensure that all moments of fan-in scaled
linear transformations of $f^{\ell-1}(x)$ are uniformly bounded in
$d_\ell$, we obtain
\[
\sup_{a\in V}\ \sup_{i\in\mathcal N_a}\ 
\sup_{d_\ell\in\mathbb N}
\mathbb E\bigl|E^{(\ell h, d_\ell)}_{ai}(X)\bigr|^t
<\infty.
\]

Finally, since $\phi$ is entrywise polynomially bounded, this uniform
moment bound implies that
\[
\sup_{a\in V}\ \sup_{i\in\mathcal N_a}\ 
\sup_{d_\ell\in\mathbb N}
\mathbb E\bigl|S^{(\ell,h)}_{\mathrm{GAT},ai}(X)\bigr|^{t}
<\infty,
\]
which completes the proof.
\end{proof}


\begin{lemma}[Moment propagation for Graphormer operators]
\label{lem:mp-graphormer-polyphi}
Under the assumptions of Theorem~\ref{thm:Graphormer_GP}, suppose that
$\phi:\mathbb R\to\mathbb R$ is entrywise polynomially bounded.
Then for any $t\ge 1$,
\begin{equation}
\sup_{i,j\in V}\ 
\sup_{d_\ell\in\mathbb N}
\mathbb E\bigl|S^{(\ell,h)}_{\mathrm{Graphormer},ij}(X)\bigr|^{t}
<\infty.
\label{eq:graphormer_S_moment_bound}
\end{equation}
\end{lemma}

\begin{proof}
For each layer width $d_\ell\in\mathbb N$, the unnormalized Graphormer
attention score is given by
\[
E^{(\ell h, d_\ell)}_{ij}(X)
=
\frac{
\big(\tilde f^{\ell}_{i}(X) W^{Q,\ell h}\big)\,
\big(\tilde f^{\ell}_{j}(X) W^{K,\ell h}\big)^{\top}
}{\sqrt{d_\ell}}
\;+\;
b_{\rho(i,j)},
\]
where $W^{Q,\ell h}$ and $W^{K,\ell h}$ are independently fan-in scaled
Gaussian and independent of $\tilde f^\ell(X)$, and the bias term satisfies
$\sup_{i,j}|b_{\rho(i,j)}|\le B<\infty$.

Write
\[
E^{(\ell h, d_\ell)}_{ij}(X)
=
T^{(d_\ell)}_{ij}(X)+b_{\rho(i,j)},
\qquad
T^{(d_\ell)}_{ij}(X)
:=
d_\ell^{-1/2}
\Big\langle
\tilde f^{\ell}_{i}(X) W^{Q,\ell h},
\tilde f^{\ell}_{j}(X) W^{K,\ell h}
\Big\rangle.
\]
Using the inequality $|a+b|^t\le 2^{t-1}(|a|^t+|b|^t)$ and the boundedness
of $b_{\rho(i,j)}$, it suffices to control the moments of
$T^{(d_\ell)}_{ij}(X)$ uniformly in $d_\ell$.

By Cauchy--Schwarz and H\"older's inequality,
\[
\mathbb E\bigl|T^{(d_\ell)}_{ij}(X)\bigr|^t
\le
d_\ell^{-t/2}
\Big(\mathbb E\|\tilde f^{\ell}_{i}(X) W^{Q,\ell h}\|^{2t}\Big)^{1/2}
\Big(\mathbb E\|\tilde f^{\ell}_{j}(X) W^{K,\ell h}\|^{2t}\Big)^{1/2}.
\]

Under the assumptions of Theorem~\ref{thm:GAT_GP}, fan-in scaling and the
moment bounds on the upstream features imply that all moments of the
linear transforms
$\tilde f^{\ell}_{i}(X) W^{Q,\ell h}$ and
$\tilde f^{\ell}_{j}(X) W^{K,\ell h}$ are uniformly bounded in $d_\ell$.
Consequently,
\[
\sup_{i,j\in V}\ 
\sup_{d_\ell\in\mathbb N}
\mathbb E\bigl|E^{(\ell h, d_\ell)}_{ij}(X)\bigr|^t
<\infty.
\]

Finally, since $\phi$ is entrywise polynomially bounded, this uniform
moment bound implies that
\[
\sup_{i,j\in V}\ 
\sup_{d_\ell\in\mathbb N}
\mathbb E\bigl|S^{(\ell,h)}_{\mathrm{Graphormer},ij}(X)\bigr|^{t}
<\infty,
\]
which completes the proof.
\end{proof}

\begin{lemma}[Moment propagation for Specformer spectral filters]
\label{lem:mp-specformer}
Under the assumptions of Theorem~\ref{thm:Specformer_GP}.
Then for any $t\ge 1$ and any random vector $u\in\mathbb R^{|V|}$,
\begin{equation}
\sup_{d_t,d^{t,H}}
\mathbb E\bigl\|S_{\mathrm{Specformer}}u\bigr\|_2^{t}
<\infty,
\qquad
\text{whenever}
\qquad
\sup_{d_t,d^{t,H}}
\mathbb E\|u\|_2^{t}<\infty.
\label{eq:specformer_moment}
\end{equation}
In particular, for any fixed coordinate $a\in V$,
\[
\sup_{d_t,d^{t,H}}
\mathbb E\bigl|(S_{\mathrm{Specformer}}^{(h)}u)_a\bigr|^{t}
<\infty.
\]
\end{lemma}

\begin{proof}
By Lemma~32 of \citet{InfiniteAttention}, the upstream hidden representation
$\bar H$ produced by the Infinite Attention mechanism admits uniform moment
bounds of all orders; in particular, for any $t\ge 1$,
\[
\sup_{d_t,d^{t,H}}
\mathbb E\|\bar H\|_\infty^{t}
<\infty.
\]
Applying a fan-in scaled linear transformation preserves uniform moment
bounds, and 
since $\sigma$ is entrywise polynomially bounded, it follows that the
spectral coefficients 
admit uniform moment bounds, namely
\[
\sup_{d_t,d^{t,H}}
\mathbb E\|\bar\lambda^{(h)}\|_\infty^{t}
<\infty.
\]

Now write
\[
S_{\mathrm{Specformer}}^{(h)}
=
U\,\operatorname{diag}\!\big(\bar\lambda^{(h)}\big)\,U^\top.
\]
Because $U$ is orthogonal, we have
\[
\|S_{\mathrm{Specformer}}^{(h)}\|_{2\to2}
=
\|\operatorname{diag}(\bar\lambda^{(h)})\|_{2\to2}
=
\|\bar\lambda^{(h)}\|_\infty.
\]
Therefore, for any random vector $u$,
\[
\|S_{\mathrm{Specformer}}^{(h)}u\|_2
\le
\|\bar\lambda^{(h)}\|_\infty\,\|u\|_2.
\]
Raising both sides to the power $t$ and taking expectations yields
\[
\mathbb E\bigl\|S_{\mathrm{Specformer}}^{(h)}u\bigr\|_2^{t}
\le
\Big(\mathbb E\|\bar\lambda^{(h)}\|_\infty^{2t}\Big)^{1/2}
\Big(\mathbb E\|u\|_2^{2t}\Big)^{1/2},
\]
by Cauchy--Schwarz. The claimed bound follows from the established uniform
moment bounds.
\end{proof}


\begin{lemma}[Moment propagation for GTN operators]
\label{lem:mp-gtn-finiteV}
Under the assumptions of Theorem~\ref{thm:GTN_GP},
fix a head $h$ and a meta-path length $K$.
Then for any $t\ge 1$ and any random vector $u$,
\[
\sup_{n}\ \sup_{a\in V}\ 
\mathbb E\bigl|(S^{(K,h)}(X)u)_a\bigr|^{t}
\;\le\;
\sup_{n}\ \sup_{i\in V}\ 
\mathbb E|u_i|^{t}.
\]
\end{lemma}

\begin{proof}
By construction, each relation adjacency matrix is row-normalized and
entrywise nonnegative.
Convex combinations with simplex weights preserve these properties, and
products of row-stochastic, entrywise nonnegative matrices remain
row-stochastic and entrywise nonnegative.
Hence each row of $S^{(K,h)}(x)$ is a probability vector.

For any $a\in V$,
\[
(S^{(K,h)}(x)u)_a
=
\sum_{i\in V} S^{(K,h)}_{ai}(x)\,u_i.
\]
Since the function $z\mapsto |z|^t$ is convex for $t\ge 1$,
Jensen's inequality yields
\[
\bigl|(S^{(K,h)}(X)u)_a\bigr|^t
\le
\sum_{i\in V} S^{(K,h)}_{ai}(X)\,|u_i|^t.
\]
Taking expectations and using $\sum_{i} S^{(K,h)}_{ai}(X)=1$ gives
\[
\mathbb E\bigl|(S^{(K,h)}(X)u)_a\bigr|^t
\le
\sup_{i\in V}\mathbb E|u_i|^t.
\]
Taking the supremum over $a\in V$ and $n$ completes the proof.
\end{proof}

\section{Graph Transformers on Heterogeneous Graphs}
\label{app:GTN}

In this Appendix, the GP setting is extended to heterogeneous graphs with multiple edge types.

\paragraph{Heterogeneous Graphs}
Let $\mathcal{G} = (\mathcal{V}, \mathcal{E}, \mathcal{T})$ be an undirected graph with $N$ nodes, where $\mathcal{V}$ denotes the set of nodes, $\mathcal{E}$ the set of edges, and $\mathcal{T}$ the set of edge types (relations). 
Heterogeneous graphs contain multiple edge types, i.e., $|\mathcal{T}| > 1$. For each edge type $t \in \mathcal{T}$, let $A_t \in \mathbb{R}^{N \times N}$ denote the corresponding adjacency matrix.The node feature matrix $X \in \mathbb{R}^{N \times d_{\mathrm{in}}}$ remains unchanged.

\paragraph{Graph Transformers Network (GTN)}
GTN operates on heterogeneous graphs by recursively constructing
meta-path adjacency matrices through soft selections over edge types.
For each GTN head $h \in \{1,\dots,d^{\ell,H}\}$, define
$\tilde A_i := D_i^{-1} A_i$ for $i \in \mathcal{T}$,
where $D_i$ denotes the diagonal degree matrix associated with $A_i$.
The recursion is initialized as
$A^{(1,h)} = \sum_{i=1}^{\mathcal{T}} \beta_i^{(1,h)} \tilde A_i$.
For a fixed meta-path length $K$,
\begin{equation}
S_{\text{GTN}}^{(K,h)} = A^{(K,h)}
=
A^{(K-1,h)}
\left(
\sum_{i=1}^{\mathcal{T}} \beta_i^{(K,h)} A_i
\right),
\label{eq:gtn-graph}
\end{equation}
where $\beta^{(k,h)}$ for $k = 1,\dots,K$ are learnable parameters that
weight the relation-specific adjacency matrices for the $h$-th head.
The design choices underlying the reformulated GTN formulation are
detailed in the Appendix~\ref{app:gtn-design}.

\paragraph{Design Choices for GTN} 
\label{app:gtn-design}
\paragraph{Reparameterizing GTN Coefficients for Tractable Kernel Construction}In the original Graph Transformer Network (GTN) formulation,
the coefficients $\{ \beta^{(k)} \}_{k=0}^{K}$ are obtained by applying
a softmax function to a trainable parameter vector
$\{ W_{\phi}^{(k)} \}_{k=0}^{K}$, i.e.,
\[
\beta^{(k)} = \phi\!\left( W_{\phi}^{(k)} \right),
\qquad k = 0, \ldots, K,
\]
where $\phi(\cdot)$ denotes the softmax operator.
Under random initialization of $W_{\phi}$, the softmax transformation
induces complex dependencies among the coefficients $\beta^{(k)}$,
making their joint distribution difficult to characterize.
In particular, the resulting $\beta^{(k)}$ are no longer independent,
and moments such as $\mathbb{E}[\beta_i \beta_j]$ do not admit a
tractable closed-form expression.

To facilitate a principled Gaussian kernel construction, for each head of graph convolution layer heads $d^{\ell,H}$, therefore
depart from the original parameterization and directly initialize the
coefficients $\{ \beta^{(k)} \}$.
This choice enables explicit control over their statistical properties
and allows the corresponding kernel expectations to be computed in a
tractable manner.\\
\paragraph{Decoupled Degree Normalization for Well-Defined Graph Kernels}Meanwhile,in original paper
the GTN adjacency matrix at layer $K$ is given by the recursive
construction
\begin{equation}
A^{(K)}
=
\bigl(D^{(K-1)}\bigr)^{-1}
A^{(K-1)}
\left(
\sum_{i=1}^{\mathcal{T}} \beta_i^{(K)} A_i
\right),
\label{eq:AK-def}
\end{equation}
where the associated degree matrix is defined as
\begin{equation}
D^{(K)}_{uu}
=
\sum_{v=1}^n A^{(K)}_{uv},
\qquad
D^{(K)} = \sum_{u=1}^n D^{(K)}_{uu} \, e_u e_u^\top .
\label{eq:DK-def}
\end{equation}
By definition, both $A^{(K)}$ and $D^{(K)}$ are random matrices depending
on the same collection of random variables
$\{\beta_i^{(k)}\}_{k \le K,\, i \le \mathcal{T}}$.

Substituting~\eqref{eq:AK-def} into~\eqref{eq:DK-def}, the diagonal entry
$D^{(K)}_{uu}$ can be written explicitly as
\begin{equation}
D^{(K)}_{uu}
=
\sum_{v=1}^n
\left[
\bigl(D^{(K-1)}\bigr)^{-1}
A^{(K-1)}
\left(
\sum_{i=1}^{\mathcal{T}} \beta_i^{(K)} A_i
\right)
\right]_{uv}.
\label{eq:DK-entry}
\end{equation}
Consequently, the degree-normalized adjacency satisfies
\begin{equation}
\bigl(D^{(K)}\bigr)^{-1} A^{(K)}
=
\sum_{u=1}^n
\frac{1}{D^{(K)}_{uu}}
\; e_u e_u^\top A^{(K)} .
\label{eq:DinvA}
\end{equation}

Assume that the coefficients are independently initialized as
$
\beta_i^{(k)} \sim \mathcal{N}(0, \sigma_k^2),
\text{for all } i = 1,\dots,\mathcal{T},\; k = 1,\dots,K .
\label{eq:beta-gaussian}
$
Under this initialization, both $A^{(K)}_{uv}$ and $D^{(K)}_{uu}$ are
nonlinear functions of the same Gaussian random variables.
In particular, $D^{(K)}_{uu}$ is a centered random variable with a
continuous distribution supported on $\mathbb{R}$, and hence satisfies
\begin{equation}
\mathbb{P}\!\left( |D^{(K)}_{uu}| < \varepsilon \right) > 0
\quad
\text{for all } \varepsilon > 0 .
\label{eq:near-zero-DK}
\end{equation}
As a result, the random matrix
$\bigl(D^{(K)}\bigr)^{-1} A^{(K)}$ involves ratios of strongly coupled
random variables, where the numerator and denominator depend on the
same Gaussian coefficients.
This nonlinear coupling prevents any decoupling under expectation, and
the matrix-valued expectation
$
\mathbb{E}\!\left[ \bigl(D^{(K)}\bigr)^{-1} A^{(K)} \right]
$
does not admit a finite or closed-form expression.
In particular, the presence of $D^{(K)}_{uu}$ in the denominator
precludes the existence of well-defined second-order moments required
for a Gaussian process or kernel interpretation.

To avoid this issue, we modify the GTN construction by normalizing each
base adjacency matrix independently.

Specifically,  define
\[
\tilde A_i := D_i^{-1} A_i,
\qquad i \in \mathcal{T},
\]
where $D_i$ denotes the diagonal degree matrix associated with $A_i$,
i.e., $(D_i)_{uu} = \sum_v (A_i)_{uv}$.
For each head of graph convolution layer heads $d^{\ell,H}$, recursion is initialized as
$
A^{(1)}
=
\sum_{i=1}^{\mathcal{T}} \beta_i^{(1)} \tilde A_i ,
$
and, for a fixed meta-path length $K \ge 2$, proceeds according to
\[
S_{\mathrm{GTN}}^{(K)}
=
A^{(K)}
=
A^{(K-1)}
\left(
\sum_{i=1}^{\mathcal{T}} \beta_i^{(K)} \tilde A_i
\right).
\]
By performing degree normalization at the level of individual base
adjacency matrices, rather than on the recursively constructed
adjacency $A^{(K)}$, the resulting graph structure avoids nonlinear
coupling between normalization terms and random combination
coefficients.

\paragraph{GTN-GP}
\begin{theorem}[Infinite-width and infinite-head  limit of a GTN layer]
\label{thm:GTN_GP}
If  for the
$\ell$-st GTN layer with meta-path length $K$, the parameters satisfy
$
\beta_i^{(k)}\sim \mathcal N(0,\sigma_{k}^2),
W_{ij}^{\ell} \sim \mathcal N\!\left(0,\frac{\sigma_w^2}{d_{\ell-1}}\right),W^{\ell,H}_{ij}
\sim \mathcal{N}\!\left(0,\frac{\sigma_H^2}{d^{\ell,H}d_{\ell-1}}\right),
$
independently for all $i,j$ and $k\in[K]$, and define
$\tilde A_i := D_i^{-1}A_i$ for $i\in\mathcal T$, then, as $d_\ell\to\infty$,
$g^{\ell+1 }(X)$
converges in distribution to $g^{\ell }(X) \sim \mathcal{GP}(0, \Sigma^{(\ell)})$ with
\begin{equation}
\small
\begin{aligned}
&\Sigma^{(\ell)}_{ab}=\mathbb E\!\left[
g^{\ell,1}_{a}(X)\,g^{\ell,1}_{b}(X)
\right]
\\
&=
\Bigg(
\sigma_H^2\sigma_w^2
\prod_{k=1}^{K}\sigma_{k}^2
\sum_{i_1=1}^{\mathcal T}\cdots\sum_{i_K=1}^{\mathcal T}
\big(\tilde A_{i_1}\cdots\tilde A_{i_K}\big)\,
K^{(\ell-1)}\,
\big(\tilde A_{i_K}^\top\cdots\tilde A_{i_1}^\top\big)
\Bigg)_{ab},
\end{aligned}
\label{eq:gtn-one-layer-node-pre}
\end{equation}
where $a,b\in V$.
\end{theorem}
Equation~\eqref{eq:gtn-one-layer-node-pre} shows that, in the infinite-width and infinite-head
limit, the dependence between node pre-activations at layer $\ell$ is obtained
by aggregating over all possible length-$K$ meta-path compositions. The GTN
layer enumerates all sequences of relation types
$(i_1,\ldots,i_K)\in\mathcal T^K$ and combines the corresponding normalized
adjacency products $\tilde A_{i_1}\cdots\tilde A_{i_K}$ to construct the
resulting graph structure, so the update can be viewed as forming a composite
graph that accounts for every possible meta-path of length $K$ and accumulating
their contributions.

\begin{proof}
   
In Lemma~\ref{lem:Node-gp-limit}, we showed that, under the GP induction
hypothesis on the node features, it is sufficient to establish that
$S_{\mathrm{GNN}}^{(\ell,h)}$ converges in distribution in order for the
layer-$\ell$ pre-activation outputs $g^{\ell}(X)$ to converge in
distribution to a Gaussian process. In GTN, the message-passing operator in
head $h$ is $S_{\mathrm{GNN}}^{(\ell,h)}\equiv S_{\mathrm{GTN}}^{(K,h)}$.
Consequently, to establish Gaussian process convergence of the GTN outputs,
it suffices to prove that $S_{\mathrm{GTN}}^{(K,h)}$ converges in
distribution.

\paragraph{Convergence of $S_{\mathrm{GTN}}^{(K,h)}$}
\label{app:gtn_proof}
\label{Ap:theo_GTN}

\begin{lemma}
\label{lem:S-GTN}
Let the assumptions of Theorem~\ref{thm:GTN_GP} hold.
Then, for the fixed layer $\ell$, the process
\[
S^K_{\mathrm{GTN}} := \bigl\{S_{\mathrm{GTN}}^{(K,h)} :  h\in\mathbb N\bigr\}
\]
converges in distribution.
\end{lemma}
\begin{proof}
Recall from Equation~\eqref{eq:gtn-graph} that the GTN propagation operator
with meta-path length $K$ in layer $\ell$ and head $h$ is generated by the
random coefficients $\{\beta^{(k,h)}\}_{k=1}^K$ and can be written as
\[
S_{\mathrm{GTN}}^{(K,h)}
=
\Big(\sum_{i_K=1}^{\mathcal T}\beta_{i_K}^{(K,h)}\tilde A_{i_K}\Big)
\cdots
\Big(\sum_{i_1=1}^{\mathcal T}\beta_{i_1}^{(1,h)}\tilde A_{i_1}\Big),
\qquad
\tilde A_i := D_i^{-1}A_i .
\]
Assume that for each $k\in\{1,\ldots,K\}$ the vectors
$\beta^{(k,h)}=(\beta^{(k,h)}_1,\ldots,\beta^{(k,h)}_{\mathcal T})$ are i.i.d.\
across $h$, independent across $k$, and satisfy
$\mathbb E[\beta_{i}^{(k,h)}]=0$ and
$\mathbb E[\beta_{i}^{(k,h)}\beta_{j}^{(k,h)}]=\sigma_k^2\delta_{ij}$.
Since the matrices $\{\tilde A_i\}_{i=1}^{\mathcal T}$ are fixed, the map
$\{\beta^{(k,h)}\}_{k=1}^K\mapsto S_{\mathrm{GTN}}^{(K,h)}$ is a polynomial
function, hence measurable and continuous. Therefore $S_{\mathrm{GTN}}^{(K,h)}$
has a well-defined distribution which does not depend on the width $d_\ell$,
and in particular it is tight and thus converges in distribution along any
sequence $d_\ell\to\infty$.

Finally, writing $K^{(\ell-1)}(X,X'):=\mathbb E[f^{\ell-1}(X)f^{\ell-1}(X')^\top]$
and letting $S_{\mathrm{GTN}}^{(K,h)}$ be an independent draw of the GTN operator,
we obtain the kernel recursion
\begin{align*}
K^{(\ell)}(X,X')
&:=\mathbb E\!\left[g^{\ell}(X)\,g^{\ell}(X')^\top\right]
=\sigma_H^2\sigma_w^2\,\mathbb E\!\left[S_{\mathrm{GTN}}^{(K,1)}\,K^{(\ell-1)}(X,X')\,S_{\mathrm{GTN}}^{(K,1)\top}\right] \\
&=\sigma_H^2\sigma_w^2\Bigg(
\prod_{k=1}^{K}\sigma_{k}^2
\sum_{i_1=1}^{\mathcal T}\cdots\sum_{i_K=1}^{\mathcal T}
\tilde A_{i_1}\cdots\tilde A_{i_K}\,
K^{(\ell-1)}(X,X')\,
\tilde A_{i_K}^\top\cdots\tilde A_{i_1}^\top
\Bigg),
\end{align*}
where we used $\mathbb E[\beta_{i}^{(k,h)}\beta_{j}^{(k,h)}]=\sigma_k^2\delta_{ij}$ and
independence across $k$. Therefore the GTN outputs $g^{\ell}(X)$ converge in
distribution to a centred Gaussian process with covariance kernel $K^{(\ell)}$.
\end{proof}

\end{proof}

\subsection{GTN-GP experiments on heterogeneous graphs}
To evaluate the practical utility of the derived \textbf{GTN-GP}, we compare its performance against several state-of-the-art baselines, including GCN, GAT, HAN, and the standard finite-width GTN. We conduct experiments on three standard heterogeneous graph benchmarks \citep{wang2019han}: \textbf{DBLP}, \textbf{ACM}, and \textbf{IMDB}.

The results, summarized in Table~\ref{tab:hetero-node-cls}, demonstrate that the infinite-width limit approach is highly competitive.
\FloatBarrier
\begin{table}[!]
\centering
\caption{Node classification accuracy on heterogeneous graphs.}
\label{tab:hetero-node-cls}

\begin{tabular}{lcccccc}
\toprule
Dataset & GCN & GAT & HAN &  GTN  & GTN-GP \\
\midrule
DBLP  & 87.30 & 93.71 & 92.83 &  94.18 & \textbf{94.33} \\
ACM   & 91.60 & 92.33 & 90.96 &  \textbf{92.68} & 92.19 \\
IMDB  & 56.89 & 58.14 & 56.77 &  \textbf{60.92} & 60.50 \\
\bottomrule
\end{tabular}

\end{table}
\FloatBarrier

\section{Supplementary Material of Section \ref{Ch:Oversmoothing}}
\label{App:SBM_Theorems}


\subsection{GCN-GP Kernel under SBM}
\label{Ap:Cor_SBM_GCN}

\begin{corollary}[GCN kernel under SBM]
\label{cor:GCN_kernel_SBM}
Consider the GCN kernel in \cite{DBLP:conf/iclr/NiuA023}, given by $K^{(\ell)} = A^\top K^{(\ell-1)} A$ then, for all $\ell \geq 1$, $K^{(\ell)}$ under SBM admits the following closed-form expression:
\begin{align}
x_\ell &= \frac{1}{2}\left(\frac{n}{2}\right)^{2\ell} \left[ (x_0 + y_0)((p+q))^{2\ell} + (x_0 - y_0)(p-q)^{2\ell} \right], \\
y_\ell &= \frac{1}{2}\left(\frac{n}{2}\right)^{2\ell} \left[ (x_0 + y_0)(p+q)^{2\ell} - (x_0 - y_0)(p-q)^{2\ell} \right].
\end{align}
\end{corollary}

Define $J = \mathbf{1}\mathbf{1}^\top \in \mathbb{R}^{\frac{n}{2} \times \frac{n}{2}}$. The matrices $A$ and $K^{(0)}$ share the same eigenbasis $\{v_1, v_2\}$ where $v_1 = \frac{1}{\sqrt{n}}\mathbf{1}_n$ and $v_2 = \frac{1}{\sqrt{n}}[\mathbf{1}_{n/2}^\top, -\mathbf{1}_{n/2}^\top]^\top$. The spectral decompositions are:
\begin{align*}
    A &= \lambda_{A,1} v_1 v_1^\top + \lambda_{A,2} v_2 v_2^\top, \quad \lambda_{A,1} = (p+q)\frac{n}{2}, \lambda_{A,2} = (p-q)\frac{n}{2} \\
    K^{(0)} &= \lambda_{K,1} v_1 v_1^\top + \lambda_{K,2} v_2 v_2^\top, \quad \lambda_{K,1} = (x_0+y_0)\frac{n}{2}, \lambda_{K,2} = (x_0-y_0)\frac{n}{2}
\end{align*}

\paragraph{Proof of Corollary \ref{cor:GCN_kernel_SBM}}
Under the GCN update rule $K^{(\ell)} = A K^{(\ell-1)} A$, we have $K^{(\ell)} = A^\ell K^{(0)} A^\ell$. Utilizing the orthogonality $v_i^\top v_j = \delta_{ij}$:
\begin{align*}
    K^{(\ell)} &= \lambda_{A,1}^{2\ell} \lambda_{K,1} v_1 v_1^\top + \lambda_{A,2}^{2\ell} \lambda_{K,2} v_2 v_2^\top \\
    &= \frac{\lambda_{A,1}^{2\ell} \lambda_{K,1}}{n} \begin{pmatrix} J & J \\ J & J \end{pmatrix} + \frac{\lambda_{A,2}^{2\ell} \lambda_{K,2}}{n} \begin{pmatrix} J & -J \\ -J & J \end{pmatrix}
\end{align*}
Substituting $\lambda_{A,i}$ and $\lambda_{K,i}$ and simplifying the leading coefficient $\frac{1}{n} \cdot \frac{n}{2} = \frac{1}{2}$ yields the scalar entries:
\begin{align*}
    x_\ell &= \frac{1}{2} \left[ (x_0+y_0) \left((p+q)\frac{n}{2}\right)^{2\ell} + (x_0-y_0) \left((p-q)\frac{n}{2}\right)^{2\ell} \right] \\
    y_\ell &= \frac{1}{2} \left[ (x_0+y_0) \left((p+q)\frac{n}{2}\right)^{2\ell} - (x_0-y_0) \left((p-q)\frac{n}{2}\right)^{2\ell} \right]
\end{align*}

\paragraph{Proof of Corollary \ref{cor:GCN_oversmoothing}}
From the expressions for $x_\ell$ and $y_\ell$ in Table \ref{tab:sbm_summary}, we observe that the diagonal entries of the kernel matrix are all equal to $x_\ell$. Thus, the normalization term is $\frac{1}{n}\text{tr}(K^{(\ell)}) = \frac{1}{n}(n x_\ell) = x_\ell$. The normalized kernel is then:
\[
\frac{K^{(\ell)}}{\frac{1}{n}\text{tr}(K^{(\ell)})} = \frac{1}{x_\ell} \begin{pmatrix} 
x_\ell \mathbf{1}\mathbf{1}^\top & y_\ell \mathbf{1}\mathbf{1}^\top \\ 
y_\ell \mathbf{1}\mathbf{1}^\top & x_\ell \mathbf{1}\mathbf{1}^\top 
\end{pmatrix} = \begin{pmatrix} 
\mathbf{1}\mathbf{1}^\top & \frac{y_\ell}{x_\ell} \mathbf{1}\mathbf{1}^\top \\ 
\frac{y_\ell}{x_\ell} \mathbf{1}\mathbf{1}^\top & \mathbf{1}\mathbf{1}^\top 
\end{pmatrix}.
\]
Using the closed-form expressions for GCN-GP, the ratio of inter- to intra-community similarity is:
\[
\frac{y_\ell}{x_\ell} = \frac{(x_0+y_0)(p+q)^{2\ell} - (x_0-y_0)(p-q)^{2\ell}}{(x_0+y_0)(p+q)^{2\ell} + (x_0-y_0)(p-q)^{2\ell}} = \frac{1 - \left(\frac{x_0-y_0}{x_0+y_0}\right) \left(\frac{p-q}{p+q}\right)^{2\ell}}{1 + \left(\frac{x_0-y_0}{x_0+y_0}\right) \left(\frac{p-q}{p+q}\right)^{2\ell}}.
\]
Since $|p-q| < |p+q|$, the term $\left(\frac{p-q}{p+q}\right)^{2\ell} \to 0$ as $\ell \to \infty$. Therefore, $\lim_{\ell \to \infty} \frac{y_\ell}{x_\ell} = 1$, and the normalized kernel converges to the rank-one all-ones matrix.

\title{GAT-GP Kernel under the Stochastic Block Model}

\subsection{GAT-GP Kernel under SBM}
\label{Ap:GAT_SBM}

\begin{corollary}[GAT kernel under SBM]
\label{cor:GAT_kernel_SBM}
Consider the GAT kernel in Corollary~\ref{cor:GAT_kernel_linearized}, then, for all $\ell \geq 1$, $K^{(\ell)}$ under SBM admits the following closed-form expression:
\begin{align*}
x_\ell &= \frac{1}{2} \left[ \frac{G^{2^\ell}}{\left(\frac{n}{2}\right)^2 (p + q)^2} \right] \left[ 1 + \left( \frac{x_0 - y_0}{x_0 + y_0} \right) F^\ell \right], \\
y_\ell &= \frac{1}{2} \left[ \frac{G^{2^\ell}}{\left(\frac{n}{2}\right)^2 (p + q)^2} \right] \left[ 1 - \left( \frac{x_0 - y_0}{x_0 + y_0} \right) F^\ell \right],
\end{align*}
where the global growth factor $G$ and the structural preservation factor $F$ are defined as:
\begin{equation*}
    G = (x_0 + y_0)(p + q)^2 \left(\frac{n}{2}\right)^2 \quad \text{and} \quad F = 2\frac{p^2 - pq + q^2}{(p + q)^2}.
\end{equation*}
\end{corollary}

\paragraph{Proof of Corollary \ref{cor:GAT_kernel_SBM}}
The evolution of the Gaussian Process (GP) kernel for a Graph Attention Network (GAT) at layer $l+1$ is given by:

\begin{equation}
    K^{(l+1)} = K^{(l)} \odot (A K^{(l)} A) + A (K^{(l)} \odot K^{(l)}) A,
\end{equation}

where $A$ is the adjacency matrix and $K^{(l)}$ the kernel at layer $l$.

We consider a Stochastic Block Model with two communities of size $n/2$. The expected adjacency matrix $A$ and the kernel $K^{(l)}$ exhibit the following block structures:

\begin{equation}
    A = \begin{bmatrix} 
    p \mathbf{1}\mathbf{1}^\top & q \mathbf{1}\mathbf{1}^\top \\ 
    q \mathbf{1}\mathbf{1}^\top & p \mathbf{1}\mathbf{1}^\top 
    \end{bmatrix}, \quad 
    K^{(l)} = \begin{bmatrix} 
    x_l \mathbf{1}\mathbf{1}^\top & y_l \mathbf{1}\mathbf{1}^\top \\ 
    y_l \mathbf{1}\mathbf{1}^\top & x_l \mathbf{1}\mathbf{1}^\top 
    \end{bmatrix}.
\label{Eq:sbm_blocks}
\end{equation}
where $\mathbf{1}\mathbf{1}^\top$ is the all-ones matrix of size $\frac{n}{2} \times \frac{n}{2}$.

\subsubsection{Recurrence Relation}
By substituting the block structures of $A$ and $K^{(l)}$ into the evolution equation and noting that $(\mathbf{1}\mathbf{1}^\top)^2 = \frac{n}{2}\mathbf{1}\mathbf{1}^\top$ (yielding a factor of $(\frac{n}{2})^2$ for the triple products), we derive the following recurrence relations for the scalar components $x_{l+1}$ and $y_{l+1}$:

\begin{align}
    x_{l+1} &= \left(\frac{n}{2}\right)^2 \left[ 2(p^2 + q^2)x_l^2 + 2pq(y_l x_l + y_l^2) \right] \\
    y_{l+1} &= \left(\frac{n}{2}\right)^2 \left[ 2(p^2 + q^2)y_l^2 + 2pq(y_l x_l + x_l^2) \right]
\end{align}

We introduce a change of variables:
\begin{align}
    S_{l+1} &= x_{l+1} + y_{l+1} \\
    \Delta_{l+1} &= x_{l+1} - y_{l+1}
\end{align}
Substituting the recurrence relations for $x_{l+1}$ and $y_{l+1}$, the evolution of the discrepancy $\Delta_{l+1}$ is given by:

\begin{align}
    \Delta_{l+1} &= \left(\frac{n}{2}\right)^2 \left[ [2(p^2+q^2)x_l^2+2pq(y_lx_l+y_l^2)] - [2(p^2+q^2)y_l^2+2pq(y_lx_l+x_l^2)] \right] \nonumber \\
    &= \left(\frac{n}{2}\right)^2 \left[ 2(p^2+q^2)(x_l^2 - y_l^2) - 2pq(x_l^2 - y_l^2) \right] \nonumber \\
    &= 2\left(\frac{n}{2}\right)^2 (p^2 - pq + q^2)(x_l - y_l)(x_l + y_l) \nonumber \\
    &= 2\left(\frac{n}{2}\right)^2 (p^2 - pq + q^2) \Delta_l S_l
\end{align}
Similarly, we derive the evolution of the total kernel sum $S_{l+1}$:

\begin{align}
    S_{l+1} &= \left(\frac{n}{2}\right)^2 \left[ [2(p^2+q^2)x_l^2+2pq(y_lx_l+y_l^2)] + [2(p^2+q^2)y_l^2+2pq(y_lx_l+x_l^2)] \right] \nonumber \\
    &= \left(\frac{n}{2}\right)^2 \left[ 2(p^2+q^2)(x_l^2 + y_l^2) + 2pq(x_l + y_l)^2 \right] \nonumber \\
    &= \left(\frac{n}{2}\right)^2 \left[ (p^2+q^2)((x_l-y_l)^2 + (x_l+y_l)^2) + 2pq(x_l + y_l)^2 \right] \nonumber \\
    &= \left(\frac{n}{2}\right)^2 (p^2+q^2) \Delta_l^2 + \left(\frac{n}{2}\right)^2 (p+q)^2 S_l^2
\end{align}

In order to simplify the problem to a single variable recurrence relation, we express $\Delta_l$ as a product of previous terms. By iteratively applying the recurrence $\Delta_{j+1} = 2(\frac{n}{2})^2 (p^2 - pq + q^2) \Delta_j S_j$ starting from $j=0$, we obtain:

\begin{equation}
    \Delta_l = \Delta_0 \prod_{j=0}^{l-1} \left[ 2\left(\frac{n}{2}\right)^2 (p^2 - pq + q^2) S_j \right] = \Delta_0 [2\left(\frac{n}{2}\right)^2 (p^2 - pq + q^2)]^l \prod_{j=0}^{l-1} S_j
\end{equation}

Substituting the square of this expression into the equation for $S_{l+1}$, we have:

\begin{align}
    S_{l+1} &= \left(\frac{n}{2}\right)^2 (p^2+q^2) \left( \Delta_0 [2\left(\frac{n}{2}\right)^2 (p^2 - pq + q^2)]^l \prod_{j=0}^{l-1} S_j \right)^2 + \left(\frac{n}{2}\right)^2 (p+q)^2 S_l^2 \nonumber \\
    &= \left(\frac{n}{2}\right)^2 (p^2+q^2) [2\left(\frac{n}{2}\right)^2 (p^2 - pq + q^2)]^{2l} \Delta_0^2 \prod_{j=0}^{l-1} S_j^2 + \left(\frac{n}{2}\right)^2 (p+q)^2 S_l^2
\end{align}

Using the initial condition $\Delta_0 = x_0 - y_0$, we arrive at the final one-variable recurrence relation:

\begin{equation}
    S_{l+1} = \left(\frac{n}{2}\right)^2 (p^2+q^2) [2\left(\frac{n}{2}\right)^2 (p^2 - pq + q^2)]^{2l} (x_0 - y_0)^2 \prod_{j=0}^{l-1} S_j^2 + \left(\frac{n}{2}\right)^2 (p+q)^2 S_l^2
    \label{Eq:alpha}
\end{equation}

This equation demonstrates that the total kernel sum at layer $l+1$ depends on the entire history of the sums from previous layers $\{S_0, \dots, S_l\}$, weighted by the SBM parameters and the initial discrepancy.

\subsubsection{Analytical Solution of the Recurrence Relation}

Consider the sequence $S_l$:
\begin{equation}
    S_l = \alpha \prod_{j=0}^{l-2} S_j^2 + \beta S_{l-1}^2,
    \label{eq:original}
\end{equation}
where $\alpha = \left(\frac{n}{2}\right)^2(p^2+q^2) \left[2\left(\frac{n}{2}\right)^2 (p^2 - pq + q^2)\right]^{2(l-1)} (x_0 - y_0)^2$ and $\beta = \left(\frac{n}{2}\right)^2 (p+q)^2$.

To eliminate the product term, we examine $S_{l-1}$. By shifting the index $l \to l-1$:
\begin{equation}
    S_{l-1} = \alpha \prod_{j=0}^{l-3} S_j^2 + \beta S_{l-2}^2
\end{equation}
Rearranging to isolate the product:
\begin{equation}
    \alpha \prod_{j=0}^{l-3} S_j^2 = S_{l-1} - \beta S_{l-2}^2
    \label{eq:isolated_product}
\end{equation}
Returning to $S_l$ and peeling off the last term:
\begin{equation}
    S_l = \left( \alpha \prod_{j=0}^{l-3} S_j^2 \right) S_{l-2}^2 + \beta S_{l-1}^2
\end{equation}
Substituting Equation \ref{eq:isolated_product}:
\begin{equation}
    S_l = (S_{l-1} - \beta S_{l-2}^2) S_{l-2}^2 + \beta S_{l-1}^2
\end{equation}
Distributing the $S_{l-2}^2$ term:
\begin{equation}
    S_l = S_{l-1}S_{l-2}^2 - \beta S_{l-2}^4 + \beta S_{l-1}^2
\end{equation}

Using the ansatz $S_l = k \gamma^{2^l}$, we derive the characteristic polynomial for $k$:
\begin{equation}
    \beta k^4 - k^3 - \beta k^2 + k = 0
\end{equation}
The non-trivial solutions are $1$ and $1/\beta$. Given $S_0$ and $S_1$, the specific choice of $k$ ensures consistency with the initial data. For the remainder of the paper we will consider $k=1/\beta = [(\frac{n}{2})^2 (p+q)^2]^{-1}$. Using $S_0 = x_0 + y_0$:
\begin{equation}
    S_l = k \left( \frac{x_0 + y_0}{k} \right)^{2^l} = \frac{1}{\left(\frac{n}{2}\right)^2(p+q)^2} \left[ \left(\frac{n}{2}\right)^2 (p+q)^2 (x_0 + y_0) \right]^{2^l}
    \label{eq:Sl_final}
\end{equation}

The difference is:
\begin{equation}
    \Delta_l = [2\left(\frac{n}{2}\right)^2 (p^2 - pq + q^2)]^l (x_0 - y_0) \prod_{j=0}^{l-1} S_j
\end{equation}
The product simplifies to $\prod_{j=0}^{l-1} S_j = k^{-2^l + l + 1} (x_0 + y_0)^{2^l - 1}$. Thus:
\begin{equation}
    \Delta_l = [2\left(\frac{n}{2}\right)^2 (p^2 - pq + q^2)]^l (x_0 - y_0) k^{-2^l + l + 1} (x_0 + y_0)^{2^l - 1}
    \label{eq:Deltal_final}
\end{equation}

\subsection{Final Expressions for $x_l$ and $y_l$}
Solving $x_l = \frac{S_l + \Delta_l}{2}$ and $y_l = \frac{S_l - \Delta_l}{2}$ by defining:
\begin{equation}
    G = (x_0 + y_0)(p + q)^2 \left(\frac{n}{2}\right)^2 \quad \text{and} \quad F = 2\frac{p^2 - pq + q^2}{(p + q)^2}
\end{equation}
we obtain:
\begin{equation}
    x_l = \frac{1}{2} \left[ \frac{G^{2^l}}{\left(\frac{n}{2}\right)^2 (p + q)^2} \right] \left[ 1 + \left( \frac{x_0 - y_0}{x_0 + y_0} \right) F^l \right]
\end{equation}
\begin{equation}
    y_l = \frac{1}{2} \left[ \frac{G^{2^l}}{\left(\frac{n}{2}\right)^2 (p + q)^2} \right] \left[ 1 - \left( \frac{x_0 - y_0}{x_0 + y_0} \right) F^l \right]
\end{equation}

\subsubsection{Proof of Corollary \ref{cor:GAT_no_oversmoothing}}
\begin{proof}

Using the expressions for $x_\ell$ and $y_\ell$ from Corollary \ref{cor:GAT_kernel_SBM}, the ratio of inter- to intra-community similarity is:
\[
\frac{y_\ell}{x_\ell} = \frac{1 - \left( \frac{x_0 - y_0}{x_0 + y_0} \right) F^\ell}{1 + \left( \frac{x_0 - y_0}{x_0 + y_0} \right) F^\ell}.
\]

The normalized kernel, defined by dividing by the average trace $\frac{1}{n}\text{tr}(K^{(\ell)}) = x_\ell$, is given by:
\[
\frac{K^{(\ell)}}{\frac{1}{n}\text{tr}(K^{(\ell)})} = \begin{pmatrix} 
\mathbf{1}\mathbf{1}^\top & \left( \frac{1 - \left( \frac{x_0 - y_0}{x_0 + y_0} \right) F^\ell}{1 + \left( \frac{x_0 - y_0}{x_0 + y_0} \right) F^\ell} \right) \mathbf{1}\mathbf{1}^\top \\ 
\left( \frac{1 - \left( \frac{x_0 - y_0}{x_0 + y_0} \right) F^\ell}{1 + \left( \frac{x_0 - y_0}{x_0 + y_0} \right) F^\ell} \right) \mathbf{1}\mathbf{1}^\top & \mathbf{1}\mathbf{1}^\top 
\end{pmatrix}.
\]
As $\ell \to \infty$, if $F < 1$, the second term of $y_\ell$  vanishes, leading to rank collapse (oversmoothing). However, if $F \geq 1$, the second term does not vanish. Setting $F = 2\frac{p^2 - pq + q^2}{(p + q)^2} \geq 1$ yields:
\begin{align*}
2(p^2 - pq + q^2) &\geq p^2 + 2pq + q^2 \\
p^2 - 4pq + q^2 &\geq 0
\end{align*}
This condition is satisfied when the ratio $p/q$ is sufficiently large (or small), specifically $p/q \geq 2 + \sqrt{3}$ or $p/q \leq 2 - \sqrt{3}$. Under this condition, the limit matrix retains rank 2, preserving the community structure.
\end{proof}

\subsection{Graphormer-GP Kernel under SBM}

\begin{corollary}[Graphormer Kernel under SBM]
\label{cor:Graphormer_kernel_SBM}
Consider the Graphormer kernel from Corollary \ref{cor:graphormer-linear} without the bias term, and assume that the positional encodings capture all structural information of the graph. In the SBM setting, this implies that the spatial relation matrix coincides with the adjacency matrix, i.e., $R = A$. Under SBM, the block entries $\tilde{x}_\ell$ and $ \tilde{y}_\ell$ of $\tilde{K}^{(\ell)}$ are determined by the recurrence:
\begin{align*}
    \tilde{x}_\ell &= \alpha^\ell x_0 + (1-\alpha^\ell)p, \\
    \tilde{y}_\ell &= \alpha^\ell y_0 + (1-\alpha^\ell)q,
\end{align*}
where the normalized kernel is $K^{(\ell)} = \tilde{K}^{(\ell)} \cdot Z_{\ell-1}$, with $Z_{\ell-1} = \text{Tr}(A^\top (\tilde{K}^{(\ell-1)} \odot \tilde{K}^{(\ell-1)}))$.
\end{corollary}
Since our analysis focuses on whether the kernel preserves the block structure induced by the SBM, the trace term does not play a role. Consequently, it suffices to study the evolution of $\tilde{K}^{(\ell)}$.

\begin{proof}
By the recursive definition of the Graphormer kernel, the unnormalized kernel $\tilde{K}^{(\ell)}$ is a weighted sum of the previous layer’s normalized kernel $K^{(\ell-1)}$ and the spatial relation matrix $R=A$. Under the SBM assumption, both $K^{(\ell-1)}$ and $A$ are block-constant. 
Solving these first-order recurrences with initial conditions $x_0, y_0$ yields the closed-form expressions:
\begin{equation*}
\tilde{K}^{(\ell)} =
\begin{pmatrix}
\alpha^\ell x_0 + (1-\alpha^\ell)p
& \alpha^\ell y_0 + (1-\alpha^\ell)q \\
\alpha^\ell y_0 + (1-\alpha^\ell)q
& \alpha^\ell x_0 + (1-\alpha^\ell)p
\end{pmatrix}
\end{equation*}
The normalized kernel $K_{ab}^{(\ell)}$ is obtained by multiplying the previous unnormalized entries by the sum over the entries of the spatial relation matrix (here $A$):
\begin{align*}
K_{ab}^{(\ell)} &= \tilde{K}_{ab}^{(\ell-1)} \sum_{i,j \in V} A_{ij} \tilde{K}{ij}^{(\ell-1)} \tilde{K}{ij}^{(\ell-1)} \
= \tilde{K}_{ab}^{(\ell-1)} \cdot \text{Tr}\left(A^\top (\tilde{K}^{(\ell-1)} \odot \tilde{K}^{(\ell-1)})\right)
\end{align*}
Hence the normalized kernel is defined as $K^{(\ell)} = \tilde{K}^{(\ell)} \cdot Z_{\ell-1}$ with $Z_{\ell-1} = \text{Tr}(A^\top (\tilde{K}^{(\ell-1)} \odot \tilde{K}^{(\ell-1)}))$. 
\end{proof}

\subsection{SPECFORMER-GP KERNEL UNDER SBM}

\begin{corollary}[SpecFormer Kernel under SBM]
\label{cor:SpecFormer_kernel_SBM}
Consider the Specformer kernel from Theorem \ref{thm:Specformer_GP}, defined as  For $\ell \geq 1$, the block entries of $K^{(\ell)}$ admit the following closed-form expressions:
\begin{align*}
x_\ell &= \frac{1}{2} \left[ (x_0 + y_0) \bar{\lambda}_1^{2\ell} + (x_0 - y_0) \bar{\lambda}_2^{2\ell} \right], \\
y_\ell &= \frac{1}{2} \left[ (x_0 + y_0) \bar{\lambda}_1^{2\ell} - (x_0 - y_0) \bar{\lambda}_2^{2\ell} \right],
\end{align*}
where $\bar{\lambda}_1$ and $\bar{\lambda}_2$ $\bar{\lambda}_2$ denote the learned eigenvalues of the Specformer's spectral filter; these parameters determine the cross-covariance $K_{\lambda}$.
\end{corollary}

\begin{proof}
As detailed in Theorem \ref{thm:Specformer_GP}, the Specformer kernel at layer $\ell$ is obtained via  $K^{(\ell)} = U (K_{\lambda} \odot (U^\top K^{(\ell-1)} U)) U^\top$. We define $M^{(\ell)} := U^\top K^{(\ell)} U$. Given the recurrence relation, we have:
\begin{equation*}
    M^{(\ell)} = K_{\lambda} \odot M^{(\ell-1)} = (K_{\lambda} \odot K_{\lambda} \odot \dots \odot K_{\lambda}) \odot M^{(0)} = (K_{\lambda})^{\odot \ell} \odot M^{(0)},
\end{equation*}
where $(K_{\lambda})_{ab} = \bar{\lambda}_a \bar{\lambda}_b$. Thus, the $\ell$-th Hadamard power is $((K_{\lambda})^{\odot \ell})_{ab} = \bar{\lambda}_a^\ell \bar{\lambda}_b^\ell$. For an SBM with two equal communities, $U=[v_1 \hspace{1mm} v_2]$ with $v_1 = \frac{1}{\sqrt{n}}\mathbf{1}_n$ and  $v_2 = \frac{1}{\sqrt{n}}[\mathbf{1}_{n/2}^\top, -\mathbf{1}_{n/2}^\top]^\top$.

The initial spectral kernel $M^{(0)} = U^\top K^{(0)} U$ is a diagonal matrix containing the eigenvalues of the block-constant kernel $K^{(0)}$:
\begin{equation*}
    M^{(0)} = \begin{pmatrix} \frac{n}{2}(x_0 + y_0) & 0 \\ 0 & \frac{n}{2}(x_0 - y_0) \end{pmatrix}.
\end{equation*}
The non-zero diagonal entries of $M^{(\ell)}$ evolve as:
\begin{align*}
    M^{(\ell)}_{11} &= (\bar{\lambda}_1 \bar{\lambda}_1)^\ell \frac{n}{2}(x_0 + y_0) = \bar{\lambda}_1^{2\ell} \frac{n}{2}(x_0 + y_0), \\
    M^{(\ell)}_{22} &= (\bar{\lambda}_2 \bar{\lambda}_2)^\ell \frac{n}{2}(x_0 - y_0) = \bar{\lambda}_2^{2\ell} \frac{n}{2}(x_0 - y_0).
\end{align*}
Reconstructing the kernel in the spatial domain via $K^{(\ell)} = M^{(\ell)}_{11} v_1 v_1^\top + M^{(\ell)}_{22} v_2 v_2^\top$, we use the outer product definitions:
\begin{equation*}
    v_1 v_1^\top = \frac{1}{n} \begin{pmatrix} \mathbf{J} & \mathbf{J} \\ \mathbf{J} & \mathbf{J} \end{pmatrix}, \quad v_2 v_2^\top = \frac{1}{n} \begin{pmatrix} \mathbf{J} & -\mathbf{J} \\ -\mathbf{J} & \mathbf{J} \end{pmatrix},
\end{equation*}
where $\mathbf{J} = \mathbf{1}\mathbf{1}^\top$ of size $\frac{n}{2} \times \frac{n}{2}$. Substituting these yields:
\begin{align*}
    K^{(\ell)} &= \frac{\bar{\lambda}_1^{2\ell}(x_0 + y_0)}{2} \begin{pmatrix} \mathbf{J} & \mathbf{J} \\ \mathbf{J} & \mathbf{J} \end{pmatrix} + \frac{\bar{\lambda}_2^{2\ell}(x_0 - y_0)}{2} \begin{pmatrix} \mathbf{J} & -\mathbf{J} \\ -\mathbf{J} & \mathbf{J} \end{pmatrix}.
\end{align*}
Identifying the diagonal block scalar $x_\ell$ and the off-diagonal block scalar $y_\ell$ gives the desired result.
\end{proof}

\paragraph{Proof of Remark \ref{cor:Specformer_advantage}}
Using the expressions for $x_\ell$ and $y_\ell$ from Corollary \ref{cor:SpecFormer_kernel_SBM}, the ratio of inter- to intra-community similarity is:
\[
\frac{y_\ell}{x_\ell} = \frac{(x_0+y_0)\bar{\lambda}_1^{2\ell} - (x_0-y_0)\bar{\lambda}_2^{2\ell}}{(x_0+y_0)\bar{\lambda}_1^{2\ell} + (x_0-y_0)\bar{\lambda}_2^{2\ell}} = \frac{\left(\frac{x_0+y_0}{x_0-y_0}\right)\left(\frac{\bar{\lambda}_1}{\bar{\lambda}_2}\right)^{2\ell} - 1}{\left(\frac{x_0+y_0}{x_0-y_0}\right)\left(\frac{\bar{\lambda}_1}{\bar{\lambda}_2}\right)^{2\ell} + 1}.
\]
If $|\bar{\lambda}_2| \geq |\bar{\lambda}_1|$, the term $(\bar{\lambda}_1/\bar{\lambda}_2)^{2\ell}$ either vanishes or stays constant as $\ell \to \infty$. Consequently, the ratio $y_\ell/x_\ell$ does not converge to $1$, and the kernel maintains its rank-$2$ structure. This allows Specformer to avoid oversmoothing by amplifying or preserving the second spectral direction $v_2$, which corresponds to the community partition.

\section{Supplementary Material for Experiments}
\label{app:experimental details}
The code used to reproduce all experiments is available at
\url{https://figshare.com/s/4ad5245d4d6c405f46f0}. In this section we provide additional information regarding datasets and experiment detauls. In Section~\ref{app:additional_experiments} we present additional experiments. 
\subsection{Datasets}
We conduct experiments on both real-world benchmarks and synthetic datasets to evaluate model performance under varying conditions. 

\paragraph{Benchmark Datasets} 
We utilize two widely recognized benchmarks: \textbf{PubMed} \citep{sen2008collective}, a homopholic citation network, and \textbf{Chameleon} \citep{rozemberczki2019gemsec}, a heterophilic Wikipedia page-page network.

\paragraph{Synthetic Data} 
To analyze the interplay between graph topology and feature information, we employ two synthetic generative models, considering two-class versions for both:
\begin{itemize}
    \item \textbf{SBM (Random Features):} The graph is generated according to the Stochastic Block Model (SBM) rule, where the probability of an edge existing between nodes in the same community is $p$ and between different communities is $q$. Node features are drawn from an independent Gaussian distribution, making them uninformative of the community structure.
    \item \textbf{CSBM (Aligned Features):} In the Contextual Stochastic Block Model (CSBM), node features are generated as a mixture of Gaussians such that the feature distributions are aligned with the graph's community labels. This represents a more realistic scenario common in literature where the graph structure and node attributes provide complementary information regarding the underlying classes.
\end{itemize}

All datasets are accessed via the PyTorch Geometric (PyG) data loaders. We follow the standard PyG splits for each dataset.
In particular, for datasets that come with multiple predefined splits (notably the heterophilous benchmarks), we run the full training--validation--test procedure on every split and report the mean of the test accuracy across splits.

\begin{table}[H]
\centering
\caption{Dataset statistics for node-level classification tasks.}
\label{tab:dataset_statistics}
\begin{tabular}{lccccccc}
\toprule
Dataset & Level & Splits &  Nodes &  Edges & Classes & Features & Homophily  \\
\midrule
Pubmed    & Node & 1 & 19,717 & 44,338  & 3  & 500   & 0.80 \\
Chameleon & Node & 10 & 2,277  & 31,421  & 10 & 2,325 & 0.23  \\

\bottomrule
\end{tabular}
\end{table}
\begin{table}[H]
\centering
\caption{Datasets for node classification on heterogeneous graphs.}
\label{tab:hetero-datasets}
\begin{tabular}{lrrrrrrr}
\toprule
Dataset & \ Nodes & \ Edges & \ Edge type & \ Features & \ Training & \ Validation & \ Test \\
\midrule
DBLP  & 18405 & 67946 & 4 & 334  & 800 & 400 & 2857 \\
ACM   &  8994 & 25922 & 4 & 1902 & 600 & 300 & 2125 \\
IMDB  & 12772 & 37288 & 4 & 1256 & 300 & 300 & 2339 \\
\bottomrule
\end{tabular}
\end{table}

\subsection{Experimental Environment, Training Details, and Hyperparameters}
We compare several graph GP kernels derived from the infinite-width limits of representative architectures, including GCN-GP, GAT-GP, Graphormer-GP, and Specformer-GP.
For each kernel, we perform node classification using a GP-style multi-output regression setup: class labels are encoded as one-hot targets, and predictions are obtained by selecting the class with the largest output score.

Replacing the softmax with the identity map requires an explicit normalisation step; consistent with \citet{InfiniteAttention}, we incorporate a per-node LayerNorm. Concretely, at the end of each layer we apply the following LayerNorm kernel:
\(
K_{ab}\Big[K_{aa}K_{bb}\Big]^{-1/2}.
\)
All GP models in our experiments include a LayerNorm module at the end of every layer.

For each split, model selection is conducted on the validation set (e.g., via a grid search over the ridge regularization coefficient) and the resulting configuration is evaluated on the test set.
For datasets with multiple splits, we aggregate results by reporting the mean across splits.

\subsection{Additional Experiments}
\label{app:additional_experiments}
\paragraph{Distribution Experiments}
Figure~\ref{fig:app_spec_dist} complements Figure~\ref{fig:distribution_grid} in the introduction by further illustrating the distribution of eigenvalue encoding of Specformer.

\begin{figure}[H]
    \centering
    \begin{minipage}{0.6\columnwidth} 
        \centering
        \includegraphics[width=0.48\linewidth]{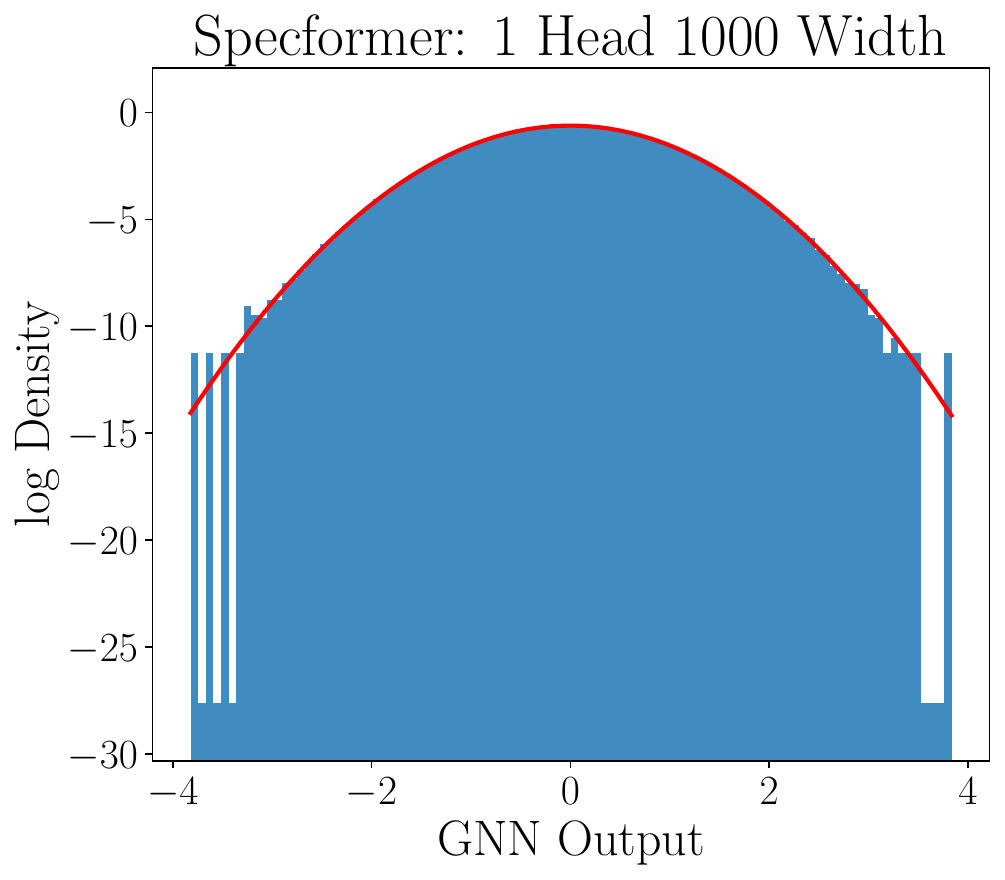}
        \hfill
        \includegraphics[width=0.48\linewidth]{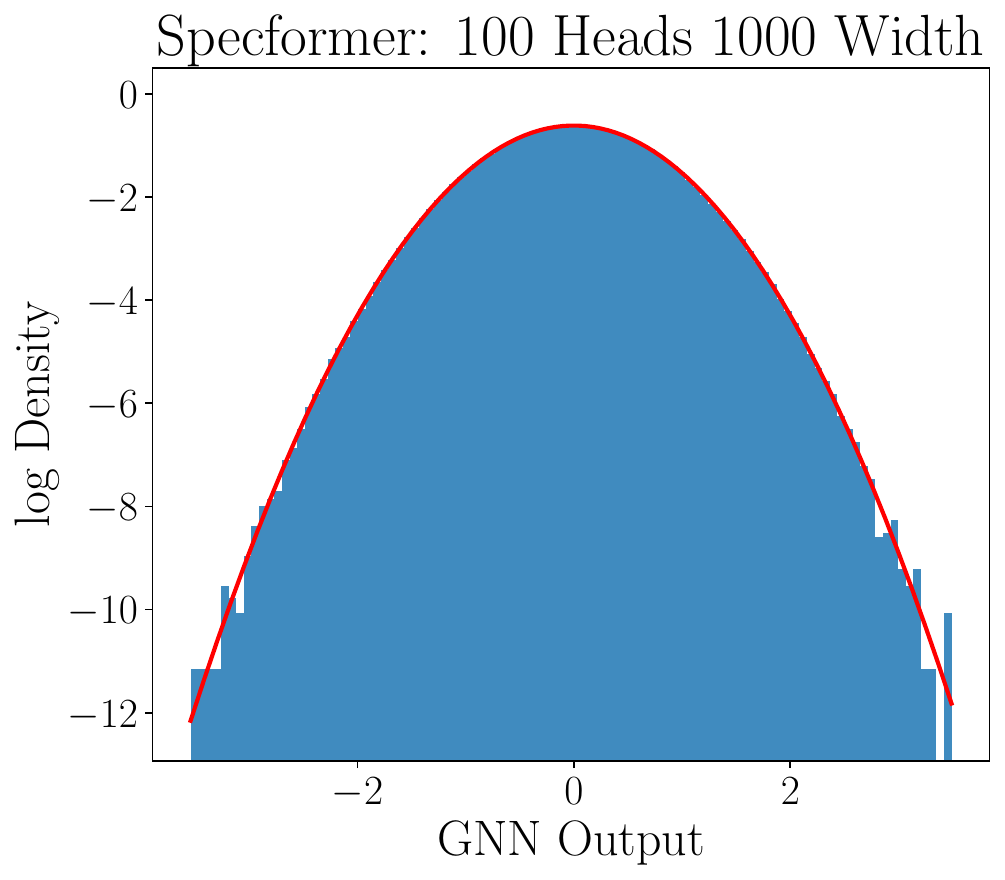}
    \end{minipage}
    \caption{Histogram of the eigenvalue encoding of Specformer for different number of heads. The output distribution converges to a Gaussian (red line) fitted with mean and variance of the empirical distribution when both width and number of heads are large.}
    \label{fig:app_spec_dist}
\end{figure}

\paragraph{Oversmoothing Experiments}
Figure~\ref{fig:pe_oversmoothing_graphormer} further illustrates the effect of positional encodings on oversmoothing. In both datasets, Laplacian-based positional encodings consistently achieve high test accuracy across depth. In the CSBM setting, where node features are already informative, these encodings have perfect accuracy.
\begin{figure}[H]
    \centering
    \includegraphics[width=0.5\columnwidth]{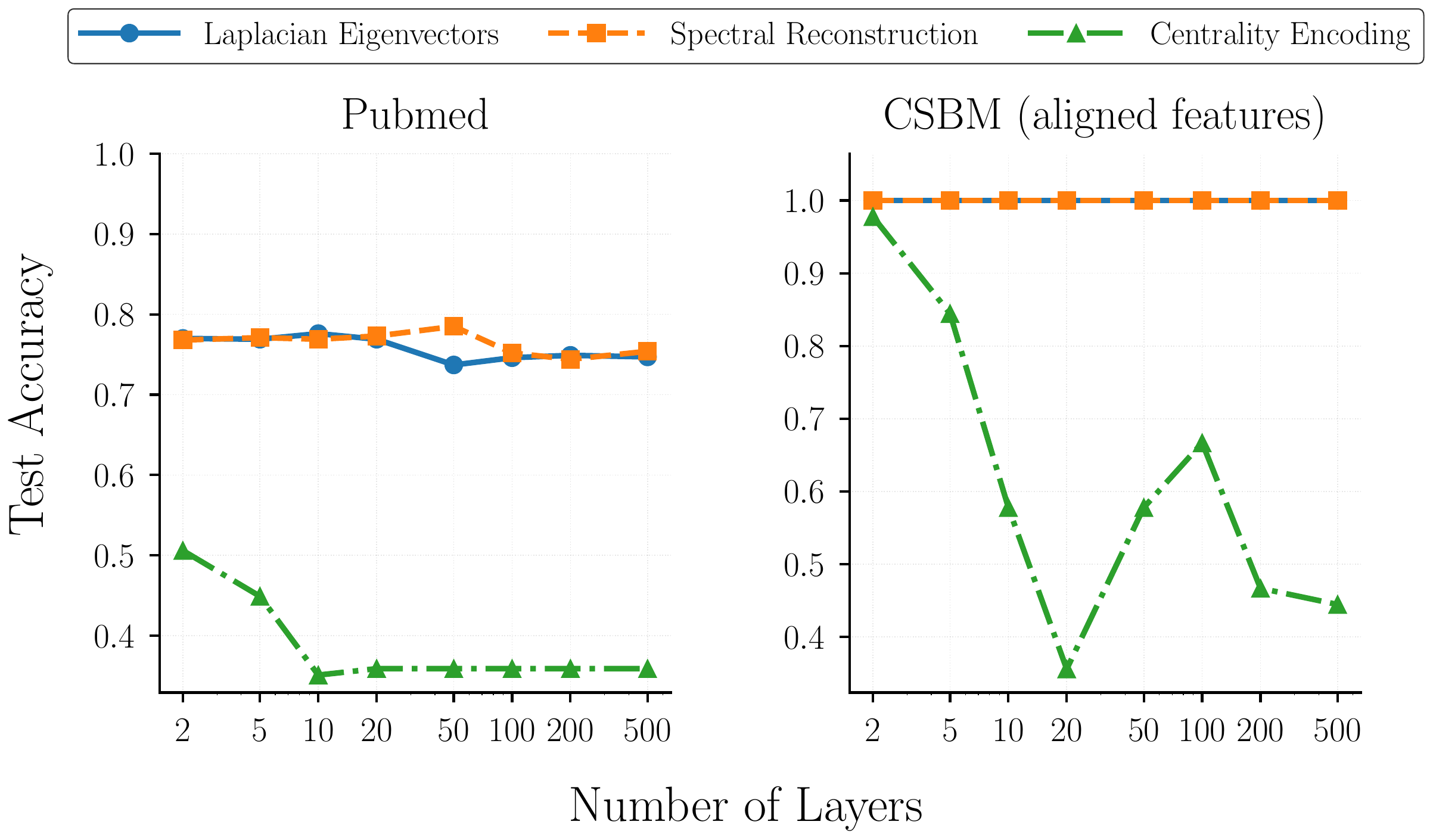}
    \caption{\textbf{Oversmoothing behaviour and the impact of positional encodings.} Test accuracy of \textbf{Graphormer-GP} on Pubmed (left) and CSBM (right) as a function of the number of layers.}
    \label{fig:pe_oversmoothing_graphormer}
\end{figure}
\end{document}